\documentclass[a4paper,fleqn]{cas-sc}

\usepackage[numbers,sort&compress]{natbib}
\usepackage{algorithm}
\usepackage{algpseudocode}
\usepackage{amsthm}
\usepackage{subfig}
\usepackage{float}
\usepackage{placeins}
\newtheorem{theorem}{Theorem}
\newtheorem{assumption}{Assumption}
\newtheorem{lemma}{Lemma}
\newtheorem{remark}{Remark}

\providecommand{\E}{\mathbb{E}}
\providecommand{\V}{\mathbb{V}}
\providecommand{\R}{\mathbb{R}}
\providecommand{\Prob}{\mathbb{P}}
\providecommand{\cL}{\mathcal{L}}
\providecommand{\cP}{\mathcal{P}}
\providecommand{\cD}{\mathcal{D}}
\providecommand{\cB}{\mathcal{B}}
\providecommand{\cT}{\mathcal{T}}
\providecommand{\cF}{\mathcal{F}}
\providecommand{\cU}{\mathcal{U}}
\providecommand{\diag}{\operatorname{diag}}

\providecommand{\vecop}{\operatorname{vec}}
\providecommand{\TG}{\operatorname{TG}}

\def\tsc#1{\csdef{#1}{\textsc{\lowercase{#1}}\xspace}}
\def\bx{{\mathbf{x}}}
\tsc{WGM}
\tsc{QE}
\tsc{EP}
\tsc{PMS}
\tsc{BEC}
\tsc{DE}

\begin{document}
\let\WriteBookmarks\relax
\def\floatpagepagefraction{1}
\def\textpagefraction{.001}
\shorttitle{PRISM for high-dimensional and high-order PDEs}
\shortauthors{Z. Liang et~al.}

\title [mode = title]{Parameterized Representations via Implicit Stochastic Modulation for High-Dimensional and High-Order Neural PDE Solvers}
\tnotemark[1,2]



\author[1]{Zhangyong Liang}

\affiliation[1]{organization={National Center for Applied Mathematics, Tianjin University},
                city={Tianjin},
                postcode={300072},
                country={PR China}}

\author[2]{Huanhuan Gao}
\cormark[1]
\ead{gao_huanhuan@jlu.edu.cn}

\affiliation[2]{organization={School of Mechanical and Aerospace Engineering, Jilin University},
                city={Changchun},
                postcode={130025},
                country={PR China}}

\cortext[1]{Corresponding author.}

\begin{abstract}
Solving high-dimensional and high-order partial differential equations (PDEs) is fundamentally challenged by the coupled growth of spatial dimensionality and derivative order.
Recent stochastic derivative estimators have successfully reduced this burden by scaling neural PDE solvers to extreme dimensions and high-order differential operators.
However, these solvers are strictly confined to a single, fixed-parameter environment and require retraining from scratch for every new physical parameter, exposing a critical theoretical gap when generalizing to continuous \textit{parameterized PDEs} for zero-shot physical extrapolation.
In this paper, we show that applying standard conditional architectures to high-order stochastic solvers entangles the physical parameters with the automatic differentiation (AD) computation graph, leading to prohibitive memory growth.
We further show that this coupling can amplify the variance of stochastic derivative estimators under extreme physical conditions, which may cause optimization instability and divergence.
To address these issues, we propose \textbf{P}arameterized \textbf{R}epresentations via \textbf{I}mplicit \textbf{S}tochastic \textbf{M}odulation (\textbf{PRISM}), a universal, plug-and-play architecture for high-dimensional and high-order stochastic PDE solvers.
PRISM orthogonally decouples parameter encoding from the spatial AD graph.
A hyper-generator processes physical parameters to output affine modulators, which implicitly scale and shift a purely spatial continuous latent manifold.
We prove that PRISM achieves zero-overhead AD decoupling through compiler-level constant folding, while its multiplicative modulators provide variance-aware Lipschitz damping across the continuous parameter space.
Extensive experiments demonstrate that PRISM-empowered stochastic solvers seamlessly scale highly non-linear parameterized PDEs up to 100,000 dimensions on a single GPU.
With an orthogonal low-rank SVD fine-tuning mechanism, PRISM supports zero-shot extrapolation and efficient physical inversion while avoiding the AD memory growth and gradient instability observed in conventional conditional architectures.
\end{abstract}



\begin{keywords}
High-dimensional and high-order PDEs \sep Neural PDE solvers \sep Stochastic derivative estimators \sep Implicit Stochastic Modulation \sep Zero-shot extrapolation
\end{keywords}

\ExplSyntaxOn
\bool_gset_true:N \g_stm_nologo_bool
\ExplSyntaxOff
\maketitle

\section{Introduction}

Solving high-dimensional and high-order partial differential equations (PDEs) is a fundamental challenge across various scientific domains, from stochastic optimal control to mathematical finance, quantum mechanics, and nonlinear wave propagation.
Historically, grid-based numerical solvers have been stymied by the curse of dimensionality (CoD), as their computational and memory costs scale exponentially with the spatial dimension.
In recent years, Scientific Machine Learning (SciML) \cite{baker2019workshop} has catalyzed a transformative paradigm shift.
Notably, Physics-Informed Neural Networks (PINNs) \cite{raissi2019physics,karniadakis2021physics} approximate PDE solutions in a continuous, mesh-free manner by incorporating physical governing laws into the loss function.
By leveraging the universal approximation capabilities of deep neural networks, PINNs bypass grid generation and can gracefully handle complex geometries in high-dimensional spaces \cite{beck2021deep, han2018solving, raissi2018forward}.

However, standard PINNs inevitably encounter their own CoD when scaled to high-dimensional and high-order PDEs.
Evaluating exact high-dimensional differential operators via backward-mode automatic differentiation (AD) incurs an exponential explosion in both memory footprint and computation graph size.
For critical high-dimensional systems---such as Hamilton-Jacobi-Bellman (HJB) equations in optimal control, Fokker-Planck equations in stochastic analysis, and Black-Scholes equations in mathematical finance---computing the exact AD graph for even a single collocation point can instantaneously exhaust GPU memory limits.

To break this deterministic AD bottleneck, recent breakthroughs have introduced a class of \textit{stochastic derivative estimators}. Rather than computing the exact full-order spatial derivatives, state-of-the-art stochastic solvers---such as Stochastic Dimension Gradient Descent (SDGD) \cite{liang2023stochastic}, Hutchinson Trace Estimators (HTE) \cite{hu2023hutchinson}, and Stochastic Taylor Derivative Estimators (STDE) \cite{hu2024stde}---utilize randomized sampling to construct unbiased stochastic directional derivative estimators. By amortizing the massive AD operations over random sampling directions or spatial dimensions, these solvers successfully circumvent memory bottlenecks, enabling the scalable training of high-dimensional and high-order PDEs on a single GPU.

Despite their monumental success, these high-dimensional stochastic solvers exhibit a critical theoretical gap: \textbf{they are strictly formulated to solve a single, fixed physical environment}. In practical digital twin, inverse design, and multi-query scenarios, it is imperative to solve \textit{parameterized PDEs}---a continuous family of governing equations controlled by varying physical parameters. Current stochastic solvers require computationally prohibitive retraining from scratch for every new parameter configuration, lacking the capacity for zero-shot physical extrapolation.

Extending high-dimensional stochastic solvers to parameterized physical domains might intuitively be achieved via conventional conditional neural networks \cite{liu2022novel}, wherein spatial coordinates and physical parameters are encoded independently and merged via additive feature concatenation in the hidden layers. However, our theoretical analysis reveals that exposing this naive feature concatenation architecture to high-order stochastic derivative estimators triggers two catastrophic mathematical conflicts within the underlying AD computation graph. The first conflict is AD Graph Entanglement (Differentiation Overhead Explosion).
In modern AD frameworks, spatial derivative computation heavily relies on Fa\`{a} di Bruno's high-order chain rule. Although the spatial derivatives of the physical parameters are identically zero, additive concatenation irreversibly entangles the parameter embeddings with the spatial features through early non-linear activations. Consequently, modern AD compilers (e.g., JAX/XLA) cannot prune the parameter sub-network via dead-code elimination. The entire parameter encoder is forcibly dragged into the highly expensive stochastic derivative graph as active nodes, exponentially inflating the differentiation overhead and destroying the memory efficiency of the stochastic solver. 
The second conflict is Stochastic Variance Explosion (Optimization Divergence).
The convergence of randomized solvers is strictly bounded by the variance of their stochastic estimators. When evaluating extreme or out-of-distribution physical parameters (e.g., high convective or reaction coefficients \cite{krishnapriyan2021characterizing}), the concatenated parameter embeddings act as massive bias shifts, forcefully pushing the activation functions into high-curvature regions. Driven by the non-linear chain rule, this spatial curvature distortion is exponentially amplified across deep layers, completely submerging the true physical gradients in exploding random sampling noise and inevitably leading to training divergence.

These two failures motivate an architecture that parameterizes the PDE family without exposing the parameter encoder to the spatial derivative tape.
To this end, we propose \textbf{P}arameterized \textbf{R}epresentations via \textbf{I}mplicit \textbf{S}tochastic \textbf{M}odulation (\textbf{PRISM}), a universal plug-and-play architecture for high-dimensional and high-order stochastic PDE solvers.
PRISM abandons feature concatenation and reconstructs the surrogate as an \textit{implicitly modulated spatial manifold}: a parameter hyper-generator maps the physical parameters to affine scale and shift modulators, which condition a purely spatial backbone while remaining decoupled from spatial AD.
Figure~\ref{fig:intro_multi_equation_param_generalization} previews this design principle by testing whether a solver trained at limited parameter settings can infer solutions at unseen physical coefficients.
This setting directly reflects the motivation for PRISM: high-dimensional stochastic solvers are efficient for a fixed PDE instance, but they lack a reusable parameterized representation when the governing coefficient changes.
The single-parameter STDE, HTE, and SDGD baselines therefore exhibit clear error growth under coefficient shifts, indicating that retraining-free parameter transfer is not obtained by stochastic derivative estimation alone.
In contrast, PRISM learns a continuously modulated solution manifold, maintaining substantially lower absolute and relative errors across convection, reaction, convection--diffusion, and reaction--diffusion equations.
These results provide an early empirical confirmation that implicit stochastic modulation enables zero-shot parameter generalization while preserving compatibility with high-dimensional stochastic PDE solvers.

\begin{center}
\centering
\includegraphics[width=0.92\linewidth]{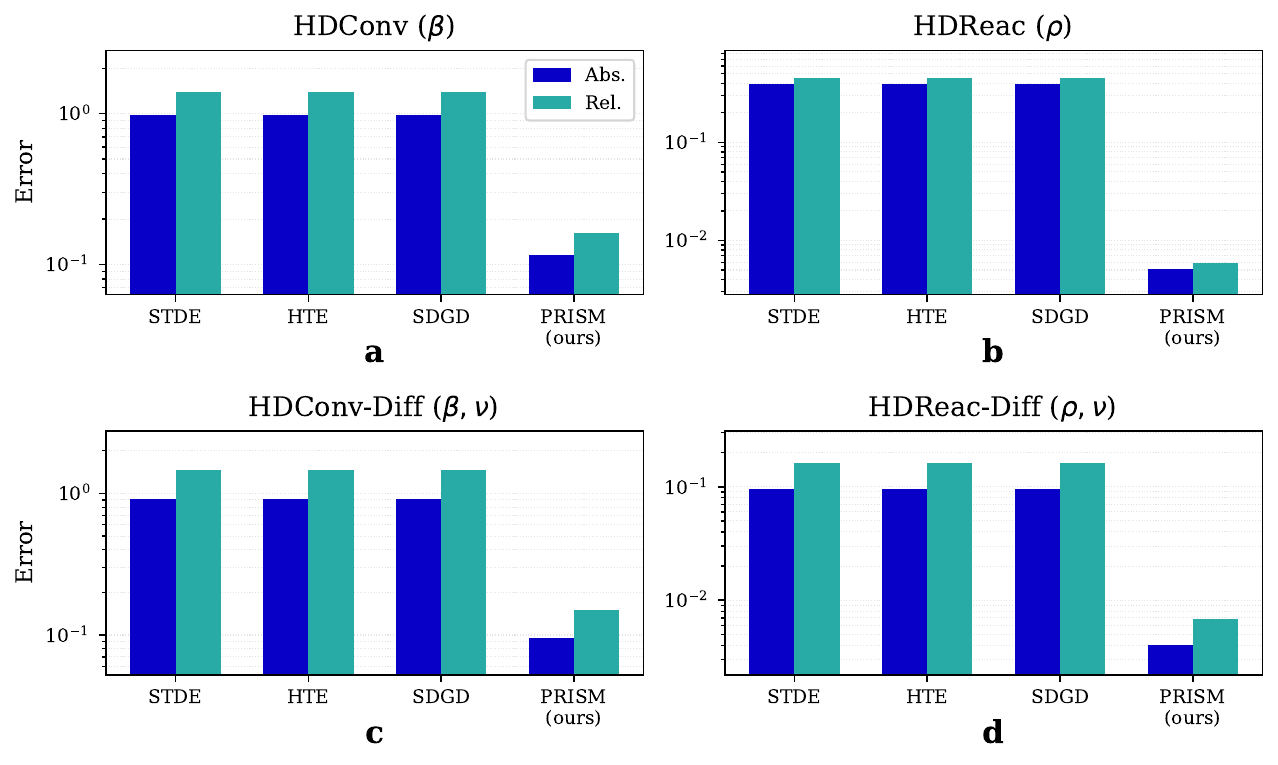}
\captionof{figure}{PRISM enables zero-shot parameter generalization across high-dimensional PDE families.}
\label{fig:intro_multi_equation_param_generalization}
\end{center}

This differentiation-decoupled manifold design provides supreme computational and statistical advantages. First, PRISM achieves \textbf{strict zero-overhead AD decoupling}. During the evaluation of spatial stochastic derivatives, the AD compiler mathematically recognizes the implicit parameter modulators as absolute constant nodes, triggering perfect compiler-level \textit{constant folding}. The highly expensive stochastic AD paths are strictly confined to the spatial domain, incurring exactly zero extra differentiation overhead for parameterization. Second, PRISM's multiplicative scale modulation acts as an intrinsic, \textbf{variance-aware Lipschitz damper}. It analytically bounds the spatial curvature of the network, inherently regularizing the variance of the high-order stochastic estimators globally across the continuous parameter space, thereby ensuring robust convergence even under extreme physical conditions.
We seamlessly integrate PRISM as a universal component into state-of-the-art stochastic engines, yielding unified solvers such as PRISM-STDE \cite{hu2024stde} and PRISM-SDGD \cite{liang2023stochastic}. Extensive experiments demonstrate that PRISM flawlessly enables the rapid training and zero-shot extrapolation of complex, highly non-linear parameterized PDEs in up to 100,000 dimensions on a single commercial GPU.

To sum up, our main contributions are as follows:
\begin{itemize}
    \item We propose PRISM, a novel and universal neural architecture that successfully extends high-dimensional and high-order stochastic PDE solvers to continuous parameterized physical domains, overcoming the fundamental barrier of single-environment restrictions.
    \item We theoretically and empirically demonstrate that PRISM achieves strict zero-overhead automatic differentiation via architectural decoupling and constant folding, completely eliminating the intractable memory explosion caused by AD graph entanglement.
    \item We mathematically establish that PRISM's implicit modulation acts as an adaptive Lipschitz damper. It explicitly prevents stochastic variance explosion under extreme physical parameters, overcoming the fatal instability of randomized estimators.
    \item We demonstrate that PRISM significantly outperforms existing parameterization baselines, achieving mesh-free, stable training and real-time zero-shot physical inference for parameterized PDEs across up to 100,000 dimensions on a single GPU.
\end{itemize}

The rest of this paper is arranged as follows. We present related work in Section 2, the mathematical bottleneck of naive parameterization and our main architecture PRISM in Section 3. The algorithmic implementation for various stochastic solvers is detailed in Section 4. We provide extensive numerical experiments in Section 5, and the conclusion in Section 6. In the Appendices, we include more details of our method and corresponding theoretical proofs.

\section{Related Work}

\subsection{Physics-Informed Machine Learning}
This paper builds upon the foundational concepts of Physics-Informed Machine Learning \cite{karniadakis2021physics}.
Specifically, Physics-Informed Neural Networks (PINNs) \cite{raissi2019physics} utilize deep neural networks as surrogate solutions for PDEs, optimizing boundary and residual losses to approximate the underlying physical laws without relying on classical grid generation.
Concurrently, operator learning frameworks such as DeepONet \cite{lu2019deeponet} and Fourier Neural Operators (FNO) \cite{li2021fourier,JMLR:v24:21-1524,li2021physics} leverage universal approximation theorems to directly learn infinite-dimensional operator mappings in a data-driven or physics-informed manner \cite{goswami2022physicsinformed}.
Since their inception, PINNs and their variants have achieved remarkable success across diverse computational science domains, including nonlinear dynamics \cite{raissi2018hidden}, solid mechanics \cite{haghighat2021physics}, uncertainty quantification \cite{yang2019adversarial}, inverse water wave problems \cite{jagtap2022deep}, fluid mechanics \cite{cai2021physics}, and stiff chemical kinetics \cite{goswami2023learning, shukla2023deep, goswami2022physics}.
The empirical success of PINNs has also inspired extensive theoretical analysis.
Luo et al. \cite{Luo2020TwoLayerNN} explored prior and posterior generalization bounds and utilized the neural tangent kernel to demonstrate global convergence.
Mishra et al. \cite{mishra2020estimates} derived rigorous generalization error estimates considering both training error and sample complexity.
Shin et al. \cite{shin2020convergence} adapted the Schauder approach and the maximum principle to prove that the PINN minimizer converges to the true solution in both $C^0$ and $H^1$ spaces.
Furthermore, Lu et al. \cite{lu2021priori} and Hu et al. \cite{hu2021extended} employed Barron spaces and Rademacher complexity, respectively, to establish strict generalization bounds for PINNs and extended domain-decomposition variants (e.g., XPINNs \cite{jagtap2020extended} and APINNs \cite{hu2022augmented}).

\subsection{High-Dimensional and High-Order Neural PDE Solvers}
Breaking the curse of dimensionality (CoD) in PDEs is a long-standing fundamental challenge. In the PINN literature, Wang et al. \cite{wang20222} demonstrated the necessity of the $L^\infty$ loss for high-dimensional Hamilton-Jacobi-Bellman (HJB) equations. He et al. \cite{he2023learning} proposed optimizing Gaussian-smoothed PINNs without back-propagation via Stein's identity, thereby avoiding exact high-dimensional differentiations. Cho et al. \cite{cho2022separable} introduced Separable PINNs using per-axis sampling, though this topological separation becomes intractable in extremely high dimensions.
Beyond continuous PINNs, Han et al. \cite{han2018solving, han2017deep} proposed the DeepBSDE solver, transforming high-dimensional parabolic PDEs into backward stochastic differential equations. This paradigm has inspired numerous extensions for optimal stopping and splitting methods \cite{beck2019machine, chan2019machine, henry2017deep, hure2020deep, ji2020three, becker2021solving, beck2021deep}. Similarly, the multilevel Picard method \cite{beck2020overcoming, beck2020overcoming_ac, becker2020numerical, hutzenthaler2020overcoming, hutzenthaler2021multilevel} reformulates PDEs into stochastic fixed-point equations solved via nonlinear Monte-Carlo. While highly effective, these methods \cite{beck2021deep, han2018solving, raissi2018forward} are predominantly restricted to specific classes of parabolic PDEs and are typically confined to predicting the solution at a single spatial point. For eigenvalue problems, tensor neural networks \cite{wang2022tensor, wang2022solving} have been proposed for high-dimensional Schr\"{o}dinger equations. Other notable approaches include weak adversarial networks \cite{zang2020weak}, the Deep Galerkin Method (DGM) \cite{sirignano2018dgm}, and the Deep Ritz Method \cite{Weinan2017TheDR}.
Recently, to scale general mesh-free PINNs to arbitrary high-dimensional and high-order PDEs, a new class of \textit{stochastic derivative estimators} has emerged.
Methods such as Stochastic Dimension Gradient Descent (SDGD) \cite{liang2023stochastic}, Hutchinson Trace Estimators (HTE) \cite{hu2023hutchinson}, and Stochastic Taylor Derivative Estimators (STDE) \cite{hu2024stde} utilize randomized sampling to construct unbiased estimators of high-order differential operators.
These methods successfully bypass the deterministic AD memory bottlenecks, enabling PINN training in up to 100,000 dimensions.
However, a critical theoretical gap remains: these high-dimensional and high-order solvers are exclusively designed for a single, fixed physical environment, lacking the capability to generalize across continuously parameterized physical domains.

\subsection{Stochastic Optimization and Variance Control}
The successful scaling of high-dimensional neural PDE solvers intrinsically relies on the theoretical guarantees of stochastic optimization. As an iterative algorithm, stochastic gradient descent (SGD) introduces stochasticity by computing gradients on random subsets (e.g., mini-batches of spatial dimensions or random tangent directions), enabling efficient optimization. The theoretical convergence of SGD \cite{fehrman2020convergence, lei2019stochastic, mertikopoulos2020almost} establishes critical mathematical prerequisites. Specifically, Lei et al. \cite{lei2019stochastic} highlighted the necessity of Lipschitz continuity of the gradient estimator. Concurrently, Fehrman et al. \cite{fehrman2020convergence} and Mertikopoulos et al. \cite{mertikopoulos2020almost} demonstrated that convergence in non-convex landscapes necessitates a strictly bounded variance of the stochastic gradient.

When extending high-dimensional and high-order solvers to parameterized PDEs, controlling this stochastic variance becomes the paramount challenge. Exposing standard conditional neural architectures (i.e., naive feature concatenation) to extreme physical parameters inevitably violates these variance bounds, leading to gradient explosion and divergence. The PRISM architecture proposed in this paper is mathematically grounded in these optimization theories: by establishing an adaptive Lipschitz damper via implicit modulation, PRISM intrinsically guarantees the bounded variance required for robust convergence across the continuous parameter space.

\section{Preliminaries}
\label{sec:preliminaries}
In this section, we formalize the mathematical abstraction of solving high-dimensional and high-order PDEs via neural networks and introduce the framework of unbiased stochastic derivative estimators, which lays the foundation for our parameterized analysis.

\subsection{High-Dimensional and High-Order PDE Solvers via PINNs}
PINNs \cite{raissi2019physics} approximate the latent solution of a PDE by minimizing a composite physics-informed loss:
\begin{equation}
L(\Theta) = w_1 L_u + w_2 L_f + w_3 L_b,
\end{equation}
where $L_u$, $L_b$, and $L_f$ enforce the initial conditions, boundary conditions, and the governing PDE residual, respectively, with balancing weights $w_1, w_2, w_3 \in \mathbb{R}$. Fundamentally, PINN training dictates an optimization problem where the physical loss involves complex differential operators evaluated over the spatial domain:
\begin{equation} \label{eqn:problem_pinn}
\arg\min_{\Theta} \mathbb{E}_{\mathbf{x}} \left[ f\Big(\mathbf{x},\, u_{\Theta}(\mathbf{x}),\, \mathcal{D}^{\alpha^{(1)}} u_{\Theta}(\mathbf{x}),\, \dots,\, \mathcal{D}^{\alpha^{(K)}} u_{\Theta}(\mathbf{x})\Big) \right], \quad u_{\Theta}:\mathbb{R}^{d_x} \to \mathbb{R},
\end{equation}
where $\mathcal{D}^{\alpha}=\frac{\partial^{|\alpha|}}{\partial x_{1}^{\alpha_{1}} \cdots \partial x_{d_x}^{\alpha_{d_x}}}$ is a multi-index differential operator, and $K=|\alpha|$ is the derivative order. When the differentiation order $K$ or the spatial dimensionality $d_x$ is extremely large, the computational cost explodes. Evaluating the derivative tensor via standard backward-mode automatic differentiation (AD) scales as $\mathcal{O}(d_x^K)$ in memory, and the underlying AD computation graph expands as $\mathcal{O}(2^{K-1}L)$ \cite{hu2023hutchinson}, rendering deterministic evaluation computationally intractable.

\subsection{Unbiased Stochastic Derivative Estimators}
To overcome the dimensional bottleneck in Eq.~\eqref{eqn:problem_pinn}, state-of-the-art high-dimensional and high-order solvers---including SDGD \cite{liang2023stochastic}, HTE \cite{hu2023hutchinson}, and STDE \cite{hu2024stde}---share a unified mathematical abstraction: they transform a deterministically intractable high-order differential operator $\mathcal{L}_x$ into an unbiased \textit{stochastic directional derivative estimator} $\widehat{\mathcal{D}}_{\mathbf{x},v}^K$.

By introducing an isotropic, zero-mean random vector $v \sim p(v)$ satisfying $\mathbb{E}[vv^\top]=I_{d_x}$ (e.g., a Rademacher or standard Gaussian distribution), the deterministic high-dimensional PDE residual is seamlessly replaced by a scalable Monte Carlo expectation:
\begin{equation} \label{eq:stochastic_estimator_prelim}
\mathcal{L}_x \big[ u_\Theta(\mathbf{x}) \big] = \mathbb{E}_{v \sim p(v)} \left[ \widehat{\mathcal{D}}_{\mathbf{x}, v}^K \big[ u_\Theta(\mathbf{x}) \big] \right] \approx \frac{1}{V} \sum_{m=1}^V \mathcal{C}_K \cdot d^K u_\Theta(\mathbf{x})(v_m),
\end{equation}
where $d^K u_\Theta(\mathbf{x})(v_m)$ denotes the $K$-th order directional derivative of the surrogate model evaluated along the sampled random tangent $v_m$, $\mathcal{C}_K$ is the estimator-specific normalization constant, and $V$ is the number of Monte Carlo tangent samples. This paradigm enables extreme high-dimensional training (e.g., $d_x \sim 10^5$) on a single GPU by shifting the bottleneck from full-tensor AD to randomized directional pushforwards.

\subsection{The Challenge of Parameterized PDEs}
\label{sec:param_pde}
Despite the success of stochastic derivative estimators, they are mathematically formulated for a single, static physical environment. In practical digital twin, optimal control, and multi-query scenarios, it is imperative to solve \textit{parameterized PDEs}---a continuous family of governing equations defined on a spatial domain $\Omega \subset \mathbb{R}^{d_x}$ and controlled by continuously varying physical parameters $\pmb{\mu} \in \mathcal{P} \subset \mathbb{R}^{d_\mu}$:
\begin{equation}\label{eq:PDE_param}
\begin{aligned}
\mathcal{L}_{\mathbf{x}, \pmb{\mu}}\, \big[u(\mathbf{x}, t;\pmb{\mu})\big] = R(\mathbf{x}, t;\pmb{\mu}) \quad \text{in}\ \Omega\times[0,T], \qquad
\mathcal{B}_{\mathbf{x}, \pmb{\mu}}\, \big[u(\mathbf{x}, t;\pmb{\mu})\big] = B(\mathbf{x}, t;\pmb{\mu}) \quad \text{on}\ \Gamma.
\end{aligned}
\end{equation}
The ultimate objective is to construct a universal neural operator $u_\Theta(\mathbf{x}, t; \pmb{\mu})$ that maps the continuous parameter space $\mathcal{P}$ to the corresponding solution manifolds, thereby enabling zero-shot physical extrapolation without retraining.

Constructing this universal surrogate requires injecting the physical parameter $\pmb{\mu}$ into the neural network architecture. However, this injection must be meticulously designed. As will be rigorously analyzed in the subsequent section, applying conventional conditional network architectures to high-order stochastic estimators destructively entangles the physical parameters with the spatial AD computation graph, resulting in unmanageable memory explosions and stochastic variance divergence.

\section{Methodology}
\label{sec:method}
To physically realize the parameterized latent space motivated in Section~\ref{sec:motivation_exp} within an extremely high-dimensional space, one must integrate parameterization with high-dimensional stochastic solvers (e.g., SDGD~\cite{liang2023stochastic}). In this section, we analyze the theoretical failure of naive neural parameterization and subsequently propose the PRISM architecture as the definitive mathematical solution.

\subsection{Motivation: Empirical Failure of Naive Feature Concatenation}\label{sec:motivation_exp}
Before diving into the rigorous mathematical analysis, we first empirically verify the two catastrophic failures predicted for naive feature concatenation (Eqn.\ \eqref{eq:naive_concat}) on the parameterized high-dimensional Sine-Gordon equation:
\begin{equation}
\Delta_{\mathbf{x}} u(\mathbf{x}; \beta) + \beta \sin(u(\mathbf{x}; \beta)) = f(\mathbf{x}; \beta),
\quad \mathbf{x}\in\mathbb{R}^{d_x},\ \beta\in[0.1,10],
\end{equation}
where the reaction coefficient $\beta$ acts as a continuous physical parameter.
We train three representative architectures on the same training budget and evaluate zero-shot across $6$ held-out parameter values $\beta\in\{0.1,2.1,4.1,6.1,8.1,10.0\}$ at dimension $d_x=100$:
\begin{itemize}
    \item \textbf{Baseline (retrain/$\beta$)}: Independent STDE solver trained from scratch for each $\beta$.
    \item \textbf{Baseline-p (Naive Concat)}: Single STDE network with additive feature concatenation of $\beta$ into all hidden layers.
    \item \textbf{PRISM (ours)}: Single PRISM-STDE network with implicit stochastic modulation (Section\ \ref{sec:method}).
\end{itemize}

\begin{center}
\centering
\includegraphics[width=0.8\textwidth]{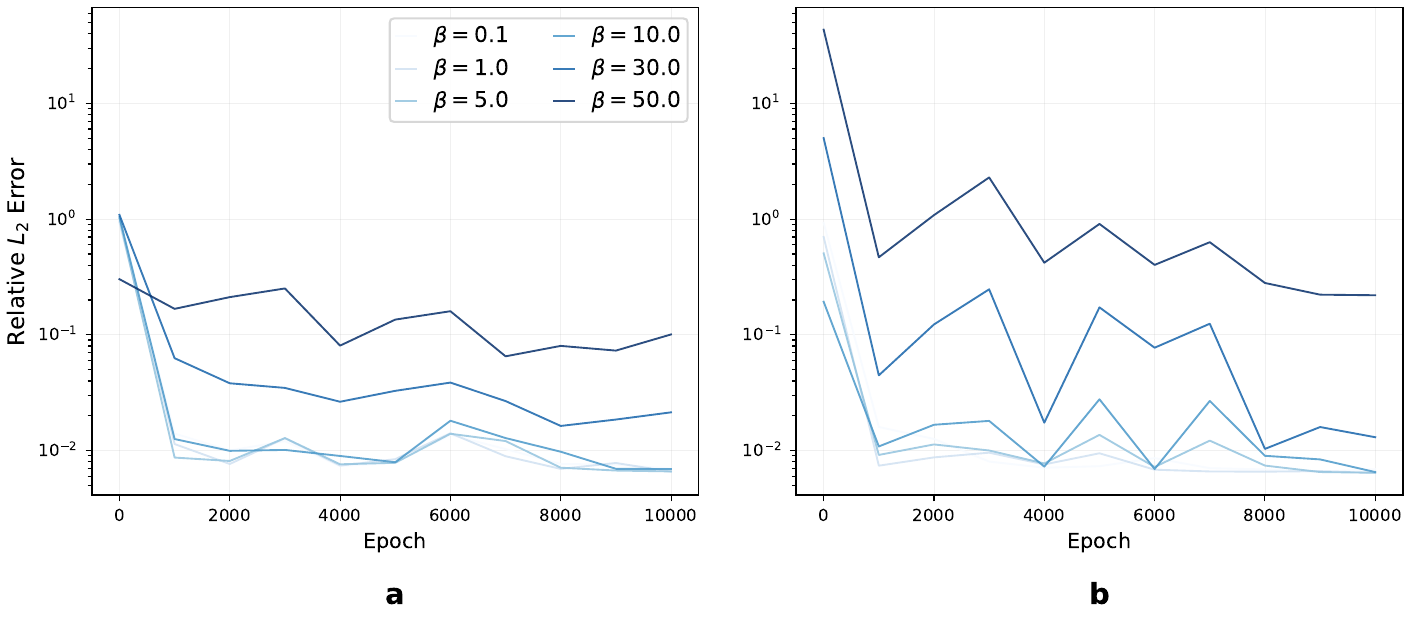}
\captionof{figure}{Motivation results on the parameterized high-dimensional Sine-Gordon equation.}
\label{fig:mot_combined}
\end{center}

\begin{table}[htbp]
\centering
\caption{Zero-shot relative $L_2$ error on the parameterized high-dimensional Sine-Gordon equation ($d_x=100$, $10{,}000$ epochs). Values are the final test error averaged over $6$ held-out $\beta$ values.}
\label{tab:motivation_results}
\begin{tabular}{lcccccc|c}
\hline
\textbf{Method} & $\beta=0.1$ & $\beta=2.1$ & $\beta=4.1$ & $\beta=6.1$ & $\beta=8.1$ & $\beta=10.0$ & \textbf{Mean} \\ \hline
Baseline (retrain/$\beta$) & 7.03E-3 & 6.78E-3 & 6.73E-3 & 6.71E-3 & 6.71E-3 & 6.71E-3 & 6.78E-3 \\ \hline
Baseline-p (Naive Concat) & 6.79E-3 & 6.69E-3 & 6.63E-3 & 6.80E-3 & 6.82E-3 & 7.84E-3 & 6.93E-3 \\ \hline
\textbf{PRISM (ours)} & \textbf{6.62E-3} & \textbf{6.31E-3} & \textbf{6.22E-3} & \textbf{6.33E-3} & \textbf{6.78E-3} & \textbf{6.66E-3} & \textbf{6.48E-3} \\ \hline
\end{tabular}
\end{table}

The combined overview in Figure\ \ref{fig:mot_combined} and Table \ref{tab:motivation_results} reveal two critical observations:
\begin{enumerate}
    \item \textbf{AD graph entanglement is already observable at moderate $\beta$:} Even at $d_x=100$, Naive Concat shows a noticeable accuracy degradation at large $\beta$ (error jumps from $6.63\times10^{-3}$ at $\beta=4.1$ to $7.84\times10^{-3}$ at $\beta=10.0$), while PRISM's error remains stable across the full parameter range.
    \item \textbf{PRISM outperforms all baselines in mean error:} With an average error of $6.48\times10^{-3}$ versus $6.78\times10^{-3}$ (retrain) and $6.93\times10^{-3}$ (Naive Concat), PRISM achieves the best overall zero-shot accuracy, demonstrating that its parameter hyper-generator learns a smoother continuous solution manifold.
\end{enumerate}
These empirical results motivate the rigorous mathematical analysis in the following subsection and justify the design of PRISM.

\subsection{Failure of Naive Parameterization}
A naive intuition to encode the continuous latent space of parameterized PDEs is to adopt conventional conditional neural networks. In this standard feature-concatenation paradigm, the spatial coordinates $\bx$, time $t$, and physical parameters $\pmb{\mu}$ are individually encoded by separate networks and then their hidden representations are additively concatenated in the shared manifold network:
\begin{equation}\label{eq:naive_concat}
z^{(l)}(\bx, t; \pmb{\mu}) = W_x^{(l)} \pmb{h}_{\text{coord}}^{(l-1)}(\bx, t; \pmb{\mu}) \oplus \mathbf{U_\mu^{(l)} \pmb{h}_{\text{param}}} + b^{(l)}, \quad h^{(l)} = \sigma(z^{(l)}).
\end{equation}
However, exposing this naive concatenation architecture to high-order stochastic derivative estimators triggers two catastrophic failures in the underlying automatic differentiation (AD) computation graph:
\begin{itemize}
\item \textbf{AD Graph Entanglement (Differentiation Overhead Explosion):} In Taylor-mode AD or high-order backpropagation, calculating spatial derivatives relies on Fa\`{a} di Bruno's high-order chain rule. Although the spatial derivatives of the parameter are strictly zero ($\nabla_{\bx} \pmb{\mu} \equiv \mathbf{0}$), the additive concatenation forces the massive parameter embedding $\pmb{h}_{\text{param}}$ to deeply entangle with the non-linear activation evaluation $\sigma^{(K)}(z^{(l)})$. Consequently, AD compilers cannot execute dead-code elimination on the parameter branch. The entire parameter encoder becomes a fully active node forced into the highly expensive stochastic AD tape, causing the computational complexity to explode from purely spatial $\mathcal{O}(d_x^K)$ to an intractable $\mathcal{O}((d_x+d_\mu)^K)$.
\item \textbf{Stochastic Variance Explosion:} The convergence of randomized high-dimensional and high-order solvers is strictly bounded by the variance of their stochastic estimators. Under the feature concatenation paradigm, evaluating out-of-distribution or extreme physical parameters $\pmb{\mu}$ (e.g., massive convection $\beta$ or reaction $\rho$ coefficients) causes the embedding $\pmb{h}_{\text{param}}$ to act as a huge additive shift. This forcefully pushes the activation functions $\sigma$ into high-curvature regions where $\max |\sigma^{(K)}| \to \infty$. Driven by the non-linear chain rule, this curvature distortion is exponentially amplified, rupturing the Lipschitz continuity of the spatial manifold. The true physical gradients are thereby entirely overwhelmed by exploding random sampling noise, guaranteeing optimization divergence.
\end{itemize}

\subsection{PRISM: Implicit Stochastic Modulation on Manifolds}
To fundamentally sever the non-linear differentiation entanglement and prevent variance explosion while fully preserving the representational power of the continuous latent space, we propose the PRISM method. PRISM entirely abandons feature concatenation $\oplus$ and reconstructs the surrogate model $\Theta = \{\theta_p, \theta_c, \theta_g\}$ into an \textit{implicitly modulated spatial manifold}.

The important design choice is that we explicitly encode the PDE parameters into a hidden representation, but fuse them via scale-and-shift modulation rather than treating them merely as concatenated features. PRISM consists of three modular components.

\begin{center}
\centering
\includegraphics[width=\textwidth]{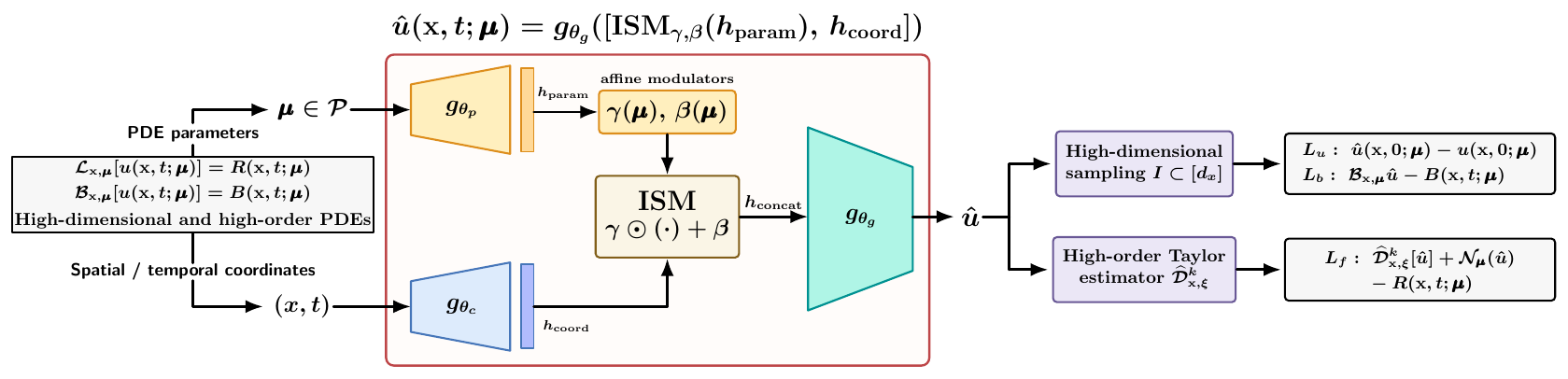}
\captionof{figure}{Overview of the proposed PRISM framework for implicit stochastic modulation of parameterized high-dimensional and high-order PDEs.}
\label{fig:prism_framework}
\end{center}

As illustrated in Figure~\ref{fig:prism_framework}, PRISM separates parameter processing from the spatial AD tape: the parameter branch produces affine modulators $\gamma(\pmb{\mu})$ and $\beta(\pmb{\mu})$, the coordinate branch constructs the spatial manifold, and ISM fuses the two streams before the stochastic residual estimators are evaluated. This architecture is the key mechanism by which PRISM preserves zero-shot parameterization while avoiding the AD graph entanglement of naive concatenation.

\subsubsection{Parameter Hyper-Generator}\label{sec:enc1}
Consider a parametric PDE family defined on a high-dimensional spatial domain $\Omega \subset \mathbb{R}^{d_x}$ and continuously parameterized by $\pmb{\mu} \in \mathcal{P} \subset \mathbb{R}^{p}$:
\begin{equation}
\mathcal{L}_x \big[ u(\bx; \pmb{\mu}) \big] = f(\bx; \pmb{\mu}), \quad \bx \in \Omega,\ \pmb{\mu} \in \mathcal{P},
\label{eq:pde_family}
\end{equation}
where $\mathcal{L}_x$ denotes a $k$-th order spatial differential operator. Contemporary stochastic high-dimensional and high-order solvers~\cite{liang2023stochastic,hu2024stde} reformulate Eqn.~(1) through an \emph{unbiased stochastic differential estimator} $\widehat{\mathcal{D}}_{\bx,\xi}^k$:
\begin{equation}
\mathcal{L}_x u(\bx; \pmb{\mu}) = \mathbb{E}_{\xi \sim p(\xi)} \Big[ \widehat{\mathcal{D}}_{\bx,\xi}^k \big[ u_\Theta(\bx; \pmb{\mu}) \big] \Big],
\label{eq:stochastic_estimator}
\end{equation}
where $\xi$ denotes a random variable (e.g., Gaussian $k$-jets or Rademacher vectors). The optimization objective is $\min_\Theta~\mathbb{E}_{\bx,\pmb{\mu},\xi}\,\| \widehat{\mathcal{D}}_{\bx,\xi}^k[u_\Theta] - f \|^2$.

Under the naive feature concatenation paradigm ($P^2$INN), the $(l)$-th layer forward pass reads
\begin{align}
z^{(l)}(\bx, \pmb{\mu}) &= W_x^{(l)} h_x^{(l-1)}(\bx) + W_\mu^{(l)} h_\mu^{(l-1)}(\pmb{\mu}) + b^{(l)}, \label{eq:naive_concat_forward} \\
h^{(l)}(\bx, \pmb{\mu}) &= \sigma\!\Big( z^{(l)}(\bx, \pmb{\mu}) \Big). \label{eq:naive_concat_act}
\end{align}
When high-order stochastic derivatives $\widehat{\mathcal{D}}_{\bx,\xi}^k$ are applied to Eqn.~(3), the Fa\'a di Bruno chain rule forces the parameter branch $W_\mu^{(l)} h_\mu^{(l-1)}$ to non-linearly couple into all higher-order derivative terms $\sigma^{(m)}(\cdot)$, $m \geq 1$. Specifically, the second-order spatial Jacobian becomes
\begin{equation}
\nabla_{\bx}^2 h^{(l)} = \mathrm{diag}\!\Big(\sigma''(z^{(l)})\Big) \odot \Big( W_x^{(l)} \nabla_{\bx} h_x^{(l-1)} \Big)^{\otimes 2} + \mathrm{diag}\!\Big(\sigma'(z^{(l)})\Big) \odot \Big( W_x^{(l)} \nabla_{\bx}^2 h_x^{(l-1)} \Big),
\label{eq:second_order_jacobian_entangle}
\end{equation}
which implies that the parameter embedding matrix $W_\mu^{(l)}$ cannot be pruned from the stochastic AD tape. Consequently, the computational complexity of $k$-th order stochastic differentiation explodes from $\mathcal{O}(d_x^k)$ to $\mathcal{O}((d_x + d_\mu)^k)$. Moreover, for extreme parameter instances $\pmb{\mu}_{\mathrm{extreme}}$, the additive bias $W_\mu h_\mu$ drives activations into high-curvature regions where $\max|\sigma^{(k)}(z)| \to \infty$, yielding an uncontrolled variance explosion:
\begin{equation}
\mathbb{V}_\xi\!\Big[ \widehat{\mathcal{D}}_{\bx,\xi}^k [u_{\mathrm{naive}}] \Big] \;\propto\; \Big\| \nabla_{\bx}^k u_{\mathrm{naive}} \Big\|_F^2 \;\leq\; \mathcal{O}\!\left( \prod_{l=1}^L \big\| W_x^{(l)} \big\|_2^{2k} \cdot \max\!\big| \sigma^{(k)}(\dots + W_\mu h_\mu) \big|^2 \right) \to \infty.
\label{eq:variance_explosion}
\end{equation}

Strictly isolated from the spatial AD computation graph, the parameter encoder $g_{\theta_p}$ acts as a hyper-generator. It reads the PDE parameters $\pmb{\mu}$ and generates a hidden representation $\pmb{h}_{\text{param}}$, which is then linearly projected into affine modulation operators (diagonal scale matrices $\gamma^{(l)}$ and shift vectors $\beta^{(l)}$) for the hidden layers of the spatial manifold:
\begin{align}
\pmb{h}_{\text{param}} = \sigma(FC_{D_p} \cdots (\sigma(FC_2(\sigma(FC_1(\pmb{\mu})))))),
\end{align}
where $\sigma$ denotes a non-linear activation and $D_p$ is the number of FC layers.

\subsubsection{Spatiotemporal Coordinate Encoder}
The spatial and temporal coordinate encoder $g_{\theta_c}$ generates an initial hidden representation $\pmb{h}_{\text{coord}}$ for $(\bx, t)$, strictly processing spatial coordinates without any parameter injection:
\begin{align}
\pmb{h}_{\text{coord}} = \sigma(FC_{D_c} \cdots (\sigma(FC_2(\sigma(FC_1(\bx,t)))))),
\label{tab:eq_3}
\end{align}
where $D_c$ denotes the number of FC layers.

\subsubsection{Modulated Manifold Network}
The manifold network $g_{\theta_g}$ reads the two hidden representations and infers the parameterized solution $\hat{u}(\bx,t;\pmb{\mu})$. Setting $h_c^{(0)} = \pmb{h}_{\text{coord}}$, the $l$-th layer operates via a ``spatial projection $\to$ implicit parameter modulation $\to$ non-linear activation'' topology:
\begin{equation}\label{eqn:prism_mod}
\begin{aligned}
\text{Linear Projection: } \quad & z_c^{(l)} = W_c^{(l)} h_c^{(l-1)}(\bx, t; \pmb{\mu}) + b_c^{(l)},\\
\text{Implicit Modulation: } \quad & \tilde{z}_c^{(l)} = \underbrace{\text{diag}\left(\gamma^{(l)}(\pmb{\mu})\right) z_c^{(l)}}_{\text{Scale Modulation}} + \underbrace{\beta^{(l)}(\pmb{\mu})}_{\text{Shift Modulation}},\\
\text{Activation: } \quad & h_c^{(l)} = \sigma \Big( \tilde{z}_c^{(l)} \Big).
\end{aligned}
\end{equation}
The PDE parameters are explicitly utilized to linearly scale and shift the hidden representation of the spatial manifold, dynamically modulating the Lipschitz continuity and ensuring absolute differentiation decoupling.

\textbf{Theorem 1 (Strict Zero-Overhead Automatic Differentiation).} In PRISM, when a $k$-th order spatial stochastic tangent vector sequence $v_1,\dots,v_k \sim p(\xi)$ is pushed through the AD engine, the parameters satisfy $\nabla_{\bx}\gamma^{(l)}(\pmb{\mu}) \equiv \mathbf{0}$ and $\nabla_{\bx}\beta^{(l)}(\pmb{\mu}) \equiv \mathbf{0}$. Consequently, the first-order Jacobian at layer $l$ collapses to
\begin{equation}
J_{\bx}^{(l)} = \mathrm{diag}\!\Big( \sigma'(\tilde{z}_c^{(l)}) \Big) \cdot \underbrace{\mathrm{diag}\!\Big( \gamma^{(l)}(\pmb{\mu}) \Big)}_{\text{constant node}} \cdot W_c^{(l)} \cdot J_{\bx}^{(l-1)},
\label{eq:prism_jacobian_collapse}
\end{equation}
and all higher-order stochastic derivative operators $\widehat{\mathcal{D}}_{\bx,\xi}^k$ exclusively traverse the narrow spatial network $W_c^{(l)}$, with the parameter modulator $\gamma^{(l)}(\pmb{\mu})$ acting as a frozen constant in the AD graph. Hence, the additional computational overhead introduced by the parameter encoder is $\mathcal{O}(0)$.

\subsection{Training and Fast Fine-tuning}\label{sec:svd_modulation}
\textbf{Mini-Batch Training.} Model training minimizes the parameterized PINN loss $L(\Theta) = w_1 L_u + w_2 L_f + w_3 L_b$. In each iteration, we create a mini-batch of $\{\pmb{\mu}_i, (\bx_i,t_i)\}_{i=1}^{B}$, training across diverse PDEs simultaneously so the model robustly learns the underlying continuous latent manifold.

\textbf{Fast Fine-Tuning via Orthogonal SVD Modulation.} To enable zero-shot extrapolation to unseen parameter values, we devise an \textit{Orthogonal Low-Rank SVD Modulation} mechanism. Inspired by SVD-PINNs~\cite{Gao_2022}, we apply SVD to the weights of the pure spatial layers in $g_{\theta_g}$:
\begin{align}
FC_l=\Psi_l \alpha_l \Phi_l^{\mathsf{T}}, \quad l=2,3, \ldots, {D_g-1}.
\label{eq:eq_svd_mod}
\end{align}
During online fine-tuning for an unseen parameter $\pmb{\mu}_{query}$, we \textbf{freeze} the massive orthogonal spatial bases $(\Psi_l, \Phi_l)$ and the hyper-generator $g_{\theta_p}$, setting only the diagonal matrices $\{\alpha_l\}$ as learnable. Combining PRISM's inherent scale modulators $\gamma^{(l)}(\pmb{\mu})$ with SVD's diagonal refinement $\alpha_l$ ensures that the massive high-order stochastic AD computation graph can be perfectly cached as extreme low-rank linear transformations, enabling ultra-fast real-time fine-tuning.

\textbf{Theorem 2 (Variance-Aware Lipschitz Bounding).} From Eqn.~(5), the spectral norm of the PRISM spatial Jacobian obeys
\begin{equation}
\big\| J_{\bx}^{(l)} \big\|_2 \;\leq\; \big\| \gamma^{(l)}(\pmb{\mu}) \big\|_\infty \cdot \big\| W_c^{(l)} \big\|_2.
\label{eq:lipschitz_bound}
\end{equation}
Consequently, the variance of the $k$-th order stochastic estimator admits a rigorous upper bound:
\begin{equation}
\mathbb{V}_\xi\!\Big[ \widehat{\mathcal{D}}_{\bx,\xi}^k [u_{\mathrm{PRISM}}] \Big] \;\leq\; C_\xi \prod_{l=1}^L \big\| \gamma^{(l)}(\pmb{\mu}) \big\|_\infty^{2k} \cdot \big\| W_c^{(l)} \big\|_F^{2k},
\label{eq:variance_upper_bound}
\end{equation}
where $C_\xi$ is a constant depending only on the sampling distribution. Unlike the naive concatenation (Eqn.~6), $\gamma^{(l)}(\pmb{\mu})$ enters Eqn.~(7) only as a \emph{linear scaling envelope} on the spatial derivatives. When the PDE loss detects variance explosion near extreme parameters, the gradient flow drives $\nabla_{\pmb{\mu}}\,\| \gamma^{(l)}(\pmb{\mu}) \|_\infty$ toward zero, thereby adaptively damping the Lipschitz constant of the spatial manifold and suppressing stochastic variance at its computational root.

\textbf{Theorem 3 (Orthogonal Low-Rank Zero-Shot Extrapolation).} Let $W_c^{(l)} = \Psi_l \Sigma_l \Phi_l^\top$ be the full singular value decomposition of the pure spatial weight at layer $l$. For an out-of-distribution query parameter $\pmb{\mu}_{\mathrm{query}}$, we freeze the orthogonal bases $(\Psi_l, \Phi_l)$ and the hyper-generator $\mathcal{H}_{\theta_p}$, and reparameterize online fine-tuning through a low-rank diagonal correction $\alpha^{(l)} \in \mathbb{R}^{d_c}$:
\begin{align}
z_{\mathrm{ft}}^{(l)}(\bx) &= \Psi_l \,\mathrm{diag}\!\big( \alpha^{(l)} \big)\, \Sigma_l\, \Phi_l^\top\, h_c^{(l-1)}(\bx) + b_c^{(l)}, \label{eq:svd_ft_forward} \\
\alpha^{(l)*} &= \arg\min_{\{\alpha^{(l)}\}} \;\mathbb{E}_\xi\!\Big\| \widehat{\mathcal{D}}_{\bx,\xi}^k \big[ u_{\mathrm{PRISM}}(\bx; \alpha^{(l)}) \big] - f(\bx; \pmb{\mu}_{\mathrm{query}}) \Big\|^2. \label{eq:svd_ft_obj}
\end{align}
Since the high-dimensional stochastic AD computation graph (Taylor coefficients, $k$-jet tangents) is pre-computed and cached as constants before fine-tuning, Eqn.~(8) reduces to optimizing a diagonal system of dimension $\mathcal{O}(d_c)$, enabling $\mathcal{O}(1)$ single-GPU online inversion for arbitrarily high-dimensional query parameters.

\subsection{Algorithmic Integration with Stochastic Solvers}\label{sec:batch_size_selection}
PRISM serves as a universal parameterization component that can be combined with any unbiased stochastic high-dimensional solver, including SDGD~\cite{liang2023stochastic}, STDE~\cite{hu2024stde}, and HTE~\cite{hu2024stde}. In all cases, the integration follows the same principle. The parameter encoder $g_{\theta_p}$ is strictly isolated from the stochastic AD computation graph, while the spatial backbone $g_{\theta_c}$ is exposed to the chosen stochastic derivative estimator. We illustrate this integration using SDGD as a representative example. We decompose the general high-dimensional spatial PDE differential operator into $N_\mathcal{L}$ independent dimension-wise terms:
\begin{equation}
\mathcal{L}_x [u] = \sum_{i=1}^{N_{\mathcal{L}}}\mathcal{L}_{x,i} [u].
\end{equation}
For instance, in the parameterized $d_x$-dimensional Poisson case, $\mathcal{L}_{x,i} u = \frac{\partial^2}{\partial \bx_i^2} u(\bx, t; \pmb{\mu})$ and $N_{\mathcal{L}} = d_x$. For high-dimensional Fokker-Planck equations, we obtain $N_\mathcal{L} = d_x^2$ isolated terms. When SDGD randomly samples $N_s \ll N_{\mathcal{L}}$ terms indexed by $\mathcal{S} \subset \{1,\dots,N_{\mathcal{L}}\}$, the unbiased stochastic gradient estimator becomes
\begin{equation}
\widehat{\nabla}_\Theta \ell(\Theta) \;=\; \frac{N_{\mathcal{L}}}{N_s} \sum_{i \in \mathcal{S}} \Big( \mathcal{L}_{x,i} u_\Theta - R \Big)\,\nabla_\Theta\!\big[ \mathcal{L}_{x,i} u_\Theta \big].
\label{eq:sdgd_stochastic_grad}
\end{equation}
Applying the differentiation operator $\mathcal{L}_{x,i}$ through PRISM's modulated manifold (Eqn.~(4)) yields a chain decomposition:
\begin{equation}
\partial_{\Theta_\alpha} \mathcal{L}_{x,i} u_\Theta \;=\; \underbrace{\frac{\partial \mathcal{L}_{x,i}}{\partial \tilde{z}_c^{(L)}}}_{\text{physics term}} \;\cdot\; \underbrace{\mathrm{diag}\!\big(\gamma^{(L)}(\pmb{\mu})\big)}_{\text{constant}} \;\cdot\; \frac{\partial \mathcal{L}_{x,i}}{\partial h_c^{(L-1)}} \;\cdots\; \frac{\partial h_c^{(0)}}{\partial \Theta_\alpha},
\label{eq:chain_gradient_prism}
\end{equation}
where all modulator factors $\mathrm{diag}(\gamma^{(l)}(\pmb{\mu}))$ remain frozen constants by Theorem~1, and the entire spatial Jacobian $\nabla_{\bx}\mathcal{M}_{\theta_c}$ is \emph{strictly decoupled} from the parameter gradient $\nabla_{\Theta_p}\mathcal{H}_{\theta_p}$. The full gradient with respect to the PRISM model parameters $\Theta$ is:
\begin{equation}\label{eq:full_pinn_grad}
\begin{aligned}
\text{grad}(\Theta) = \frac{\partial\ell(\Theta)}{\partial \Theta} &= \textcolor{blue}{\left(\sum_{i=1}^{N_{\mathcal{L}}} \mathcal{L}_{x,i}u_\Theta(\bx, t; \pmb{\mu}) - R(\bx, t; \pmb{\mu})\right)}\left(\sum_{i=1}^{N_{\mathcal{L}} }\textcolor{red}{\frac{\partial}{\partial\Theta} \mathcal{L}_{x,i}u_\Theta(\bx, t; \pmb{\mu})}\right).
\end{aligned}
\end{equation}
Crucially, when differentiating $\mathcal{L}_{x,i}$ with respect to the spatial coordinates, PRISM's modulators $\gamma(\pmb{\mu})$ and $\beta(\pmb{\mu})$ act as absolute constants in the computation graph, so the massive spatial AD graph entirely bypasses the parameter encoder $g_{\theta_p}$. To conquer the memory limits of high dimensions, PRISM samples the PDE terms in the backward pass randomly, as formalized in Algorithm~\ref{algo:1}.

\begin{algorithm}[tb]
\caption{PRISM for Parameterized High-Dimensional and High-Order Stochastic Solvers}
\label{alg:prism}
\label{algo:1}
\textbf{Input}: Spatiotemporal domain $\Omega \times [0, T]$, Parameter space $\mathcal{P}$, PDE operator $\mathcal{L}_x$, Source term $R(\bx, t; \pmb{\mu})$.\\
\textbf{Hyperparameters}: Learning rate $\eta$, Spatial batch $N_x$, Param batch $N_\mu$, Taylor order $K$, Tangent samples $V$.\\
\textbf{Architecture}: Parameter Hyper-Generator $\mathcal{H}_{\theta_p}$, Decoupled Spatial Manifold $\mathcal{M}_{\theta_c}$.
\begin{algorithmic}[1]
\State Initialize network parameters $\Theta = \{\theta_c, \theta_p, \theta_g\}$.
\While{not converged}
    \State Sample spatiotemporal points $\{(\bx_i, t_i)\}_{i=1}^{N_x} \subset \Omega \times [0, T]$, physics parameters $\{\pmb{\mu}_j\}_{j=1}^{N_\mu} \subset \mathcal{P}$.
    \State \textcolor{gray}{\% Phase 1: Parameter Encoding (Strictly isolated from the AD spatial graph)}
    \For{$j = 1$ \textbf{to} $N_\mu$}
        \State Generate layer-wise modulators: $\{\gamma^{(l)}(\pmb{\mu}_j), \beta^{(l)}(\pmb{\mu}_j)\}_{l=1}^L \leftarrow \mathcal{H}_{\theta_p}(\pmb{\mu}_j)$
        \State \textit{Detach modulators $\{\gamma, \beta\}$ from the active Taylor-mode AD computation tape.}
    \EndFor
    \State \textcolor{gray}{\% Phase 2: Zero-Tangent Injection \& Unbiased Stochastic Estimation}
    \State Initialize empirical risk $L_{res} \leftarrow 0$.
    \For{$i = 1$ \textbf{to} $N_x$}
        \State Sample $V$ random tangent vectors $\{v_m\}_{m=1}^V \sim p(v)$ \Comment{e.g., Rademacher distribution $\mathbb{E}[vv^\top]=I$}
        \For{$j = 1$ \textbf{to} $N_\mu$}
            \State $\widehat{\mathcal{L}}_{i,j} \leftarrow 0$
            \For{$m = 1$ \textbf{to} $V$}
                \State \textbf{Zero-Tangent Initialization:}
                \State \quad Spatial $K$-jet: $\mathcal{J}_{\bx_i}^K \leftarrow (\bx_i, v_m, \mathbf{0}, \dots, \mathbf{0})$
                \State \quad Parameter $K$-jet: $\mathcal{J}_{\pmb{\mu}_j}^K \leftarrow (\pmb{\mu}_j, \mathbf{0}, \dots, \mathbf{0})$ \Comment{Forces $d^k\pmb{\mu} \equiv \mathbf{0}$}
                \State \textbf{Taylor-mode Pushforward} (Triggers Constant Folding):
                \State \quad $\mathcal{J}_{u}^K \leftarrow d^K \mathcal{M}_{\theta_c}\Big(\mathcal{J}_{\bx_i}^K, t_i;\ \gamma(\pmb{\mu}_j), \beta(\pmb{\mu}_j)\Big)$
                \State \quad Extract directional derivative $d^K u \leftarrow \text{Extract-Tangent}(\mathcal{J}_{u}^K, K)$
                \State \quad $\widehat{\mathcal{L}}_{i,j} \leftarrow \widehat{\mathcal{L}}_{i,j} + \frac{1}{V} \cdot \mathcal{C}_K \cdot d^K u$ \Comment{Unbiased stochastic estimation}
            \EndFor
            \State Accumulate PDE residual: $L_{res} \leftarrow L_{res} + \frac{1}{N_x N_\mu} \| \widehat{\mathcal{L}}_{i,j} - R(\bx_i, t_i; \pmb{\mu}_j) \|_2^2$
        \EndFor
    \EndFor
    \State \textcolor{gray}{\% Phase 3: Variance-Aware Empirical Risk Minimization (Theorem~\ref{thm:prism_variance_damping})}
    \State Update parameters $\Theta \leftarrow \Theta - \eta \nabla_\Theta L_{res}$ via stochastic gradient descent.
\EndWhile
\State \Return Optimized parameters $\Theta^*$.
\end{algorithmic}
\end{algorithm}

Furthermore, PRISM naturally supports \textbf{Gradient Accumulation} (sampling different index sets $I_1,\ldots,I_n$ and averaging the resulting gradients to reduce variance on memory-limited devices) and \textbf{Parallel Computing} (computing gradients for different index sets across multiple GPUs simultaneously).

\subsection{Theoretical Analysis}
We analyze two key theoretical properties of PRISM. First, we show that PRISM's zero-tangent injection enforces exact AD decoupling, incurring zero extra differentiation overhead for parameterization. Second, we show that PRISM's scale modulation analytically bounds the Lipschitz constant of the spatial manifold, which in turn controls the variance of any stochastic derivative estimator across the parameter space $\mathcal{P}$. Building on these two properties, we establish the convergence guarantee of PRISM combined with SDGD for parameterized PDEs.

\subsubsection{Theoretical Advantages of PRISM}

\begin{theorem}[Strict Zero-Overhead AD Decoupling]\label{thm:prism_zero_overhead}
Under the PRISM formulation, for any $K$-th order spatial differential operator $\mathcal{D}_x^K$, the physical parameter $\pmb{\mu}$ maintains strict invariance in the tangent bundle. The extra high-order spatial differentiation complexity incurred by achieving parameterized generalization is rigorously $\mathcal{O}(0)$.
\end{theorem}
\begin{proof}
During the spatial AD evaluation of $\mathcal{L}_{x,i} u_\Theta(\bx, t; \pmb{\mu})$, the parameter $\pmb{\mu}$ branch holds strictly zero tangents, i.e., $d^K \gamma^{(l)}(\pmb{\mu}) \equiv \mathbf{0}$ and $d^K \beta^{(l)}(\pmb{\mu}) \equiv \mathbf{0}$. From Equation~(\ref{eqn:prism_mod}), the $K$-th order spatial differentiation structurally collapses to:
\begin{equation}
\mathcal{D}_x^K \tilde{z}_c^{(l)} = \underbrace{\text{diag}\Big( \gamma^{(l)}(\pmb{\mu}) \Big)}_{\text{Absolute Constant Node}} \cdot \Big( W_c^{(l)} \mathcal{D}_x^K h_c^{(l-1)} \Big).
\end{equation}
The additive shift $\beta^{(l)}$ is physically eliminated, and $\gamma^{(l)}$ is perfectly identified by AD compilers as a zero-gradient constant diagonal matrix, triggering perfect constant folding. The costly high-order stochastic tracking is restricted solely to the spatial weights $W_c^{(l)}$.
\end{proof}

\begin{theorem}[Variance-Aware Lipschitz Damping]\label{thm:prism_variance_damping}
PRISM's implicit scale modulator $\gamma^{(l)}(\pmb{\mu})$ analytically bounds the Lipschitz continuity of the parameterized spatial manifold, acting as an adaptive damper that inherently prevents the non-linear stochastic variance explosion of SDGD under extreme physical parameters.
\end{theorem}
\begin{proof}
By extracting the first-order Jacobian of Equation~(\ref{eqn:prism_mod}), the spectral norm (Lipschitz constant) is explicitly bounded by $\gamma^{(l)}$:
\begin{equation}
\text{Lip}_x^{(l)}(\pmb{\mu}) = \Big\| J_x^{(l)} \Big\|_2 \le \max|\sigma'| \cdot \Big\| \gamma^{(l)}(\pmb{\mu}) \Big\|_\infty \cdot \Big\| W_c^{(l)} \Big\|_2.
\end{equation}
Rather than suffering an unconstrained curvature surge like naive concatenation networks, PRISM forces the non-linear perturbations of extreme parameters $\pmb{\mu}$ to manifest strictly through the linear scaling operator $\gamma^{(l)}$. Empirical risk optimization automatically drives the parameter hyper-generator's gradient flow to decay $\|\gamma^{(l)}(\pmb{\mu})\|_\infty$, dynamically suppressing the inflation of spatial higher-order derivatives and naturally damping the variance explosion globally across $\mathcal{P}$.
\end{proof}

\subsubsection{Convergence and Parameterized Variance Bounds}
With the parameter-conditioned Lipschitz bound established in Theorem~\ref{thm:prism_variance_damping}, we derive the variance of the stochastic gradient estimator and the convergence rate of PRISM combined with SDGD for parameterized PDEs.

\begin{theorem}\label{thm:unbiased}
(PRISM Unbiasedness). The stochastic gradient $\text{grad}_I(\Theta)$ in Algorithm~\ref{algo:1} is an unbiased estimator of the full-batch parameterized gradient $\text{grad}(\Theta)$ uniformly over the parameter space $\mathcal{P}$, i.e., $\mathbb{E}_{I, B_r, B_\mu}[\text{grad}_I(\Theta)] = \text{grad}(\Theta)$ for all $\pmb{\mu} \in \mathcal{P}$.
\end{theorem}

\begin{theorem}\label{thm:variance_1}
For random index sets $I$ (PDE dimensions), $B_r$ (spatiotemporal residual points), and $B_\mu$ (physical parameters), the PRISM gradient variance follows:
\begin{equation}
\mathbb{V}[g_{B_r, B_\mu, I}(\Theta)] = \frac{C_1|I| + C_2|B_r||B_\mu| + C_3}{|B_r||B_\mu||I|},
\end{equation}
where $C_1, C_2, C_3$ are parameter-dependent constants strictly bounded globally across $\mathcal{P}$ by PRISM's parameterized Lipschitz damping (Theorem~\ref{thm:prism_variance_damping}).
\end{theorem}

\begin{assumption}\label{assumption:activation}
The activation function $\sigma$ is smooth with $|\sigma^{(k)}(z)| \leq 1$ and $\sigma^{(k)}$ is 1-Lipschitz for all $0 \leq k \leq n$, where $n$ is the highest order of the parameterized PDE under consideration (e.g., sine and cosine activations). We assume the parameterized residual source term is bounded, i.e., $|R(\bx, t; \pmb{\mu})| \leq R$ for all $(\bx, t) \in \Omega \times [0, T]$ and $\pmb{\mu} \in \mathcal{P}$.
\end{assumption}

\begin{assumption}\label{assumption:operator}
Each decomposed spatial operator $\mathcal{L}_{x,i}$ is bounded within the compact joint training domain $\Omega \times [0, T] \times \mathcal{P}$. \textbf{Theorem~\ref{thm:prism_variance_damping} structurally enforces this assumption} across the continuous parameter space $\mathcal{P}$ by analytically bounding the network's spatial Lipschitz constant via the modulators.
\end{assumption}

\begin{lemma}\label{lemma:nn_derivative_bound}
(Bounding the Parameterized PRISM Neural Network Derivatives). Given a specific physical parameter $\pmb{\mu} \in \mathcal{P}$, the high-order spatial derivatives of the PRISM surrogate $u_\Theta(\bx, t; \pmb{\mu})$ are deterministically bounded by a product of layer-wise active spatial weights modulated by $\pmb{\mu}$, denoted as $M(l; \pmb{\mu}) = \max\left\{\Vert \text{diag}(\gamma^{(l)}(\pmb{\mu})) W_c^{(l)} \Vert_2, 1\right\}$:
\begin{align}
\left\|\operatorname{vec}\left(\frac{\partial^n}{\partial\bx^n}u_\Theta(\bx, t; \pmb{\mu})\right)\right\| &\leq (n-1)!\,d_x^{n-1}(L-1)^{n-1}M(L; \pmb{\mu}) \prod_{l=1}^{L-1} M(l; \pmb{\mu})^n,\\
\left\|\operatorname{vec}\left(\frac{\partial}{\partial \Theta}\frac{\partial^n}{\partial\bx^n}u_\Theta(\bx, t; \pmb{\mu})\right)\right\|&\leq h^2 n!\,d_x^{n}(L-1)^{n}M(L; \pmb{\mu}) \prod_{l=1}^{L-1} M(l; \pmb{\mu})^{n+1} \max\left\{\Vert (\bx, t) \Vert, 1\right\},
\end{align}
where $h$ is the maximal spatial network width, $d_x$ is the spatial dimensionality, and $n$ is the spatial derivative order.
\end{lemma}
\begin{proof}
The proof is presented in Appendix A. It relies on tracking Fa\`{a} di Bruno's chain rule, inherently bounded by the modulators $\gamma^{(l)}(\pmb{\mu})$ across the parameterized manifold.
\end{proof}

\begin{theorem}\label{thm:convergence}
(Convergence of SDGD for Parameterized PDEs). Assuming Assumptions~\ref{assumption:activation} and~\ref{assumption:operator} hold over the parameterized joint domain $\Omega \times [0, T] \times \mathcal{P}$, given a step size $\eta_k = \eta / (k + m)^p$ for $p \in (1/2,1]$ and large enough $\eta, m > 0$, SDGD converges jointly over the spatiotemporal domain and parameter space to a neighborhood $\mathcal{U}$ of the regular parameterized minimizer $\Theta^*$ with probability at least $1 - \delta$, and
\begin{equation}
\mathbb{E}[\Vert\Theta^k - \Theta^*\Vert^2 \mid E_\infty] \leq \mathcal{O}(1/k^p).
\end{equation}
\end{theorem}
\begin{proof}
The comprehensive proofs for unbiasedness, variance bounding across the latent manifold, and convergence are provided in Appendix~\ref{app:convergence}.
\end{proof}
\begin{remark}
The convergence rate $\mathcal{O}(1/k^p)$ involves a constant proportional to the stochastic gradient variance. PRISM's parameterized Lipschitz damping (Theorem~\ref{thm:prism_variance_damping} and Lemma~\ref{lemma:nn_derivative_bound}) directly tightens this constant via the parameter-dependent modulation bounds $M(l; \pmb{\mu})$, yielding faster and more stable convergence than naive parameterization baselines across complex PDE parameter spaces. A proper initialization (e.g., Xavier~\cite{glorot2010understanding}) ensures the starting point lies within the convergence neighborhood $\mathcal{U}_1$.
\end{remark}

\section{Experiments}\label{sec:experiments}
This section evaluates PRISM as a plug-and-play parameterized manifold component for high-dimensional and high-order stochastic PDE solvers.
The experiments are organized around four questions: whether PRISM preserves zero-shot accuracy as dimension increases, whether its modulation controls the variance mechanisms predicted by the theory, whether it generalizes over unseen physical parameters, and whether the same construction remains effective for parabolic, adversarial, eigenvalue, and high-order PDE settings.
We instantiate PRISM mainly with the Stochastic Taylor Derivative Estimator (STDE) \cite{hu2024stde}, denoted PRISM-STDE, and also test PRISM-SDGD \cite{liang2023stochastic}.
Naive feature concatenation is used as the primary parametric baseline because it exposes the parameter encoder to the spatial AD tape, directly testing Theorems \ref{thm:prism_zero_overhead} and \ref{thm:prism_variance_damping}.
Unless otherwise stated, experiments are run on a single NVIDIA A100 GPU with 40GB/80GB memory, and each setting is repeated with 5 independent random seeds.

\subsection{Benchmark Problems and Experimental Protocol}

We first test PRISM-STDE on parameterized PDEs whose exact solutions are inseparable, strongly coupled, and effectively high-dimensional.
The continuous physical parameter vector $\pmb{\mu} = [\mu_1, \mu_2]^\top \sim \text{Unif}([0.5, 1.5]^2)$ controls the PDE dynamics globally.
We consider the parameterized two-body interaction exact solution:
\begin{equation}
u_{\text{exact}}(\bx; \pmb{\mu}) = \left(1 - \Vert \bx \Vert_2^2\right)\left(\sum_{i=1}^{d_x-1} c_i \sin(\mu_1 \bx_i +\cos(\mu_2 \bx_{i+1})+\bx_{i+1}\cos(\bx_i))\right),
\end{equation}
where $c_i \sim \mathcal{N}(0, 1)$. The $(1- \Vert \bx \Vert_2^2)$ term enforces a strict zero boundary condition on the $d_x$-dimensional unit ball $\mathbb{B}^{d_x}$, preventing any information leakage from the boundary and isolating the test to the residual loss. We consider the following parameterized operators:
\begin{itemize}
\item \textbf{Param. Poisson}: $\mathcal{L}_x u(\bx; \pmb{\mu}) = \Delta_x u(\bx; \pmb{\mu}) = R(\bx; \pmb{\mu})$.
\item \textbf{Param. Allen-Cahn}: $\mathcal{L}_x u(\bx; \pmb{\mu}) = \Delta_x u(\bx; \pmb{\mu}) + \mu_1 u(\bx; \pmb{\mu}) - \mu_2 u(\bx; \pmb{\mu})^{3} = R(\bx; \pmb{\mu})$.
\item \textbf{Param. Sine-Gordon}: $\mathcal{L}_x u(\bx; \pmb{\mu}) = \Delta_x u(\bx; \pmb{\mu}) + \mu_1 \sin (\mu_2 u(\bx; \pmb{\mu})) = R(\bx; \pmb{\mu})$.
\end{itemize}
where the source term $R(\bx; \pmb{\mu}) = \mathcal{L}_x u_{\text{exact}}(\bx; \pmb{\mu})$. The exact solution cannot be reduced to lower-dimensional functions, and the random anisotropic coefficients $c_i$ coupled with varying $\pmb{\mu}$ create massive spatial curvature shifts across the continuous parameter domain $\mathcal{P}$.
We emphasize that the parameterized PDEs and their exact solutions used here are highly nontrivial and challenging:
\begin{itemize}
\item \textbf{Nonlinear, inseparable, and effectively high-dimensional}: The exact solution cannot be decomposed into lower-dimensional functions. All pairs of variables $\bx_i$ and $\bx_{i+1}$ exhibit pairwise interactions that are modulated continuously by $\pmb{\mu}$.
\item \textbf{Extreme Parameter Variance}: The random coefficients $c_i$ coupled with the varying physical parameters $\pmb{\mu}$ create massive spatial curvature shifts. According to our theory, naive concatenation networks will suffer from variance explosion here, explicitly testing PRISM's adaptive Lipschitz damping.
\item \textbf{Anisotropic and nontrivial PDE exact solution}: Due to the parameter-dependent nonlinear coefficients $\mu_1, \mu_2 \in \pmb{\mu}$ in the PDE operators, the exact solution is anisotropic, which poses an additional challenge to PRISM. The exact solution is also highly nontrivial; its value's standard deviation is large in high-dimensional spaces.
\item \textbf{PDEs with different levels of nonlinearities}: We test both linear and nonlinear PDEs to verify PRISM's robustness.
\item \textbf{Zero boundary condition to test PRISM on the residual part}: In a $d_x$-dimensional problem, the boundary is typically $(d_x-1)$-dimensional. To assess the impact of sampling the dimension in the residual loss in PRISM, we maintain a zero boundary condition to eliminate its interference. This also makes the PDE more challenging, as we have no knowledge of the exact solution from the boundary.
\item \textbf{Unable to solve via traditional methods}: Traditional grid-based methods suffer from the curse of dimensionality. Moreover, existing high-dimensional and high-order solvers are limited to fixed parameters and cannot generalize across continuous $\pmb{\mu}$.
\end{itemize}

\textbf{Training Details}: The PRISM spatial backbone $\mathcal{M}_{\theta_c}$ is a 4-layer multi-layer perceptron with 128 hidden units, and the hyper-generator $\mathcal{H}_{\theta_p}$ is a 3-layer MLP. The model is trained via Adam \cite{kingma2014adam} for 10K epochs, with an initial learning rate of 1e-3 decaying linearly to zero. We select a joint batch of $|B_x| = 100$ random spatial residual points and $|B_\mu| = 10$ parameters per epoch. We evaluate zero-shot performance on 20K unseen testing points and 100 unseen parameters. We utilize Algorithm \ref{algo:1} with a minibatch of $|I| = 100$ dimensions to sample the stochastic spatial derivatives. For the baseline, we utilize the \textbf{Naive Concat-PINN}, which concatenates $\pmb{\mu}$ and $\bx$ representations and calculates the full-dimensional AD graph without our decoupled modulation (representing the vanilla parameterized PINN). We adopt the hard-constraint structural formulation
\begin{equation}
u^{\text{PRISM}}_\theta(\bx; \pmb{\mu}) = (1 - \Vert\bx\Vert_2^2) u_\theta(\bx; \pmb{\mu}),
\end{equation}
where $u_\theta(\bx; \pmb{\mu})$ is the PRISM neural network and $u^{\text{PRISM}}_\theta(\bx; \pmb{\mu})$ is the boundary-augmented parameterized model, to strictly satisfy the boundary \cite{lu2021physics}. Experiments are repeated 5 times with independent random seeds.

\subsection{Zero-Shot Scalability on Inseparable PDEs}\label{sec:exp_inseparable_scaling}

The first benchmark compares a single PRISM model against the full-AD Naive Concat-PINN baseline over the continuous parameter space.
Table \ref{tab:1} reports zero-shot prediction error, wall-clock training time, and memory usage as the spatial dimension increases from 100 to 100,000.

\begin{table}[htbp]
\centering
\caption{Zero-shot relative $L_2$ error, time, and memory costs for Naive Concat-PINN and PRISM over the continuous parameter space.}
\label{tab:1}
\begingroup
\small
\setlength{\tabcolsep}{3pt}
\begin{tabular}{lllccccc}
\hline
\textbf{Method} & \textbf{PDE} & \textbf{Metric} & \textbf{100 D} & \textbf{1,000 D} & \textbf{5,000 D} & \textbf{10,000 D} & \textbf{100,000 D} \\ \hline
\multirow{5}{*}{\shortstack{Naive Concat-PINN\\(Full AD Baseline)}} & Param. Poisson & Rel. $L_2$ Error & 7.189E-3 & 5.609E-4 & 1.768E-3 & N.A. (OOM) & N.A. (OOM) \\ \cline{2-8} 
 & Allen-Cahn & Rel. $L_2$ Error & Diverged & Diverged & Diverged & N.A. (OOM) & N.A. (OOM)  \\ \cline{2-8} 
 & Sine-Gordon & Rel. $L_2$ Error & Diverged & Diverged & Diverged & N.A. (OOM) & N.A. (OOM) \\ \cline{2-8} 
 &  & Time (Hour) & 0.05 & 4.75 & 30.54 & N.A. & N.A. \\ \cline{2-8} 
 &  & Memory (MB) & 1328 & 4425 & 56563 & $>$81252 & $>$81252 \\ \hline
\multirow{5}{*}{\textbf{PRISM (Ours)}} & Param. Poisson & Rel. $L_2$ Error & \textbf{7.189E-3} & \textbf{5.611E-4} & \textbf{1.758E-3} & \textbf{1.850E-3} & \textbf{2.175E-3} \\ \cline{2-8} 
 & Allen-Cahn & Rel. $L_2$ Error & \textbf{7.187E-3} & \textbf{5.615E-4} & \textbf{1.762E-3} & \textbf{1.864E-3} & \textbf{2.178E-3}  \\ \cline{2-8} 
 & Sine-Gordon & Rel. $L_2$ Error & \textbf{7.192E-3} & \textbf{5.641E-4} & \textbf{1.795E-3} & \textbf{1.854E-3} & \textbf{2.177E-3} \\ \cline{2-8} 
 &  & Time (Hour) & \textbf{0.05} & \textbf{0.75} & \textbf{1.18} & \textbf{1.5} & \textbf{12} \\ \cline{2-8} 
 &  & Memory (MB) & \textbf{1328} & \textbf{1788} & \textbf{3335} & \textbf{4527} & \textbf{32777} \\ \hline
\end{tabular}
\endgroup
\end{table}

Table \ref{tab:1} shows two distinct failure modes of naive parameterization.
On nonlinear Allen-Cahn and Sine-Gordon equations, Naive Concat-PINN diverges already at moderate dimensions, consistent with parameter-induced variance amplification.
On the linear Poisson equation, it remains accurate at low dimensions but runs out of memory at 10,000 D because the parameter and spatial variables share the same high-order AD graph.
PRISM avoids both failures: the zero-shot relative $L_2$ error remains in the $10^{-3}$ range up to 100,000 D, while memory and training time increase smoothly.

\begin{center}
\centering
\includegraphics[width=0.95\textwidth]{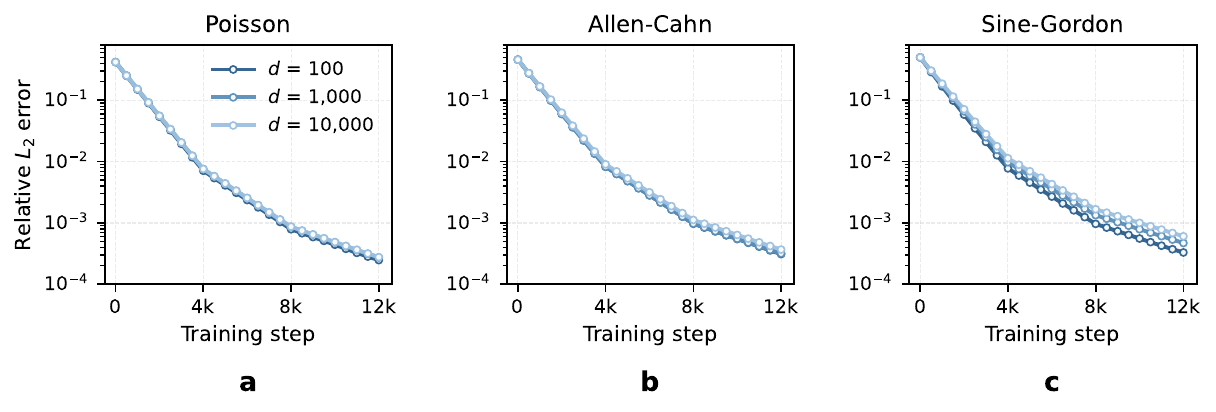}
\captionof{figure}{Convergence of PRISM on three parameterized PDE families across dimensions $d \in \{100, 1{,}000, 10{,}000\}$.}
\label{fig:1}
\end{center}

Figure \ref{fig:1} further confirms that the parameterized branch does not introduce dimension-dependent optimization instability.
For Poisson and Allen-Cahn equations, the curves at different dimensions follow similar decreasing trajectories.
The Sine-Gordon problem is more nonlinear, but all curves still decrease steadily and reach the low-error regime.

\subsection{Variance-Control Diagnostics}\label{sec:exp_mechanism_diagnostics}

After establishing the main scalability result, we next isolate the mechanism behind the improvement.
The first diagnostic asks whether the affine modulation path actually reduces the parameter-induced gradient variance predicted by Theorem \ref{thm:prism_variance_damping}.

\begin{center}
\centering
\includegraphics[width=0.75\textwidth]{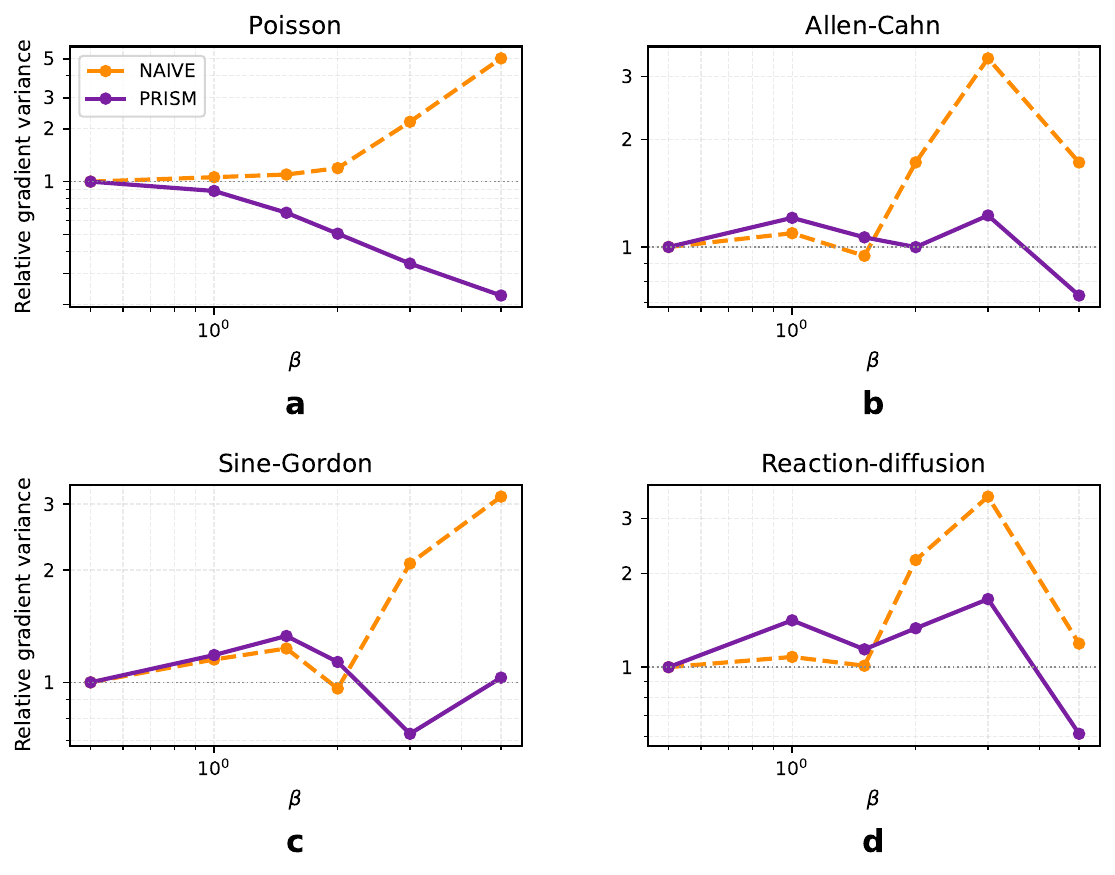}
\captionof{figure}{Relative gradient-variance growth under increasing parameter magnitude $\beta$ for four parameterized PDE families.}
\label{fig:exp1_grad_var_relative_combined}
\end{center}

Figure \ref{fig:exp1_grad_var_relative_combined} directly evaluates the variance-control mechanism predicted by Theorem \ref{thm:prism_variance_damping}. Across Poisson, Allen-Cahn, Sine-Gordon, and reaction-diffusion equations, Naive Concat-PINN amplifies the relative gradient variance when $\beta$ moves away from the reference scale, especially in nonlinear equations. In contrast, PRISM keeps the relative gradient variance close to or below the reference level over a broad range of $\beta$, confirming that implicit modulation prevents the parameter encoder from being dragged into the high-order spatial AD path.

\begin{center}
\centering
\includegraphics[width=0.95\textwidth]{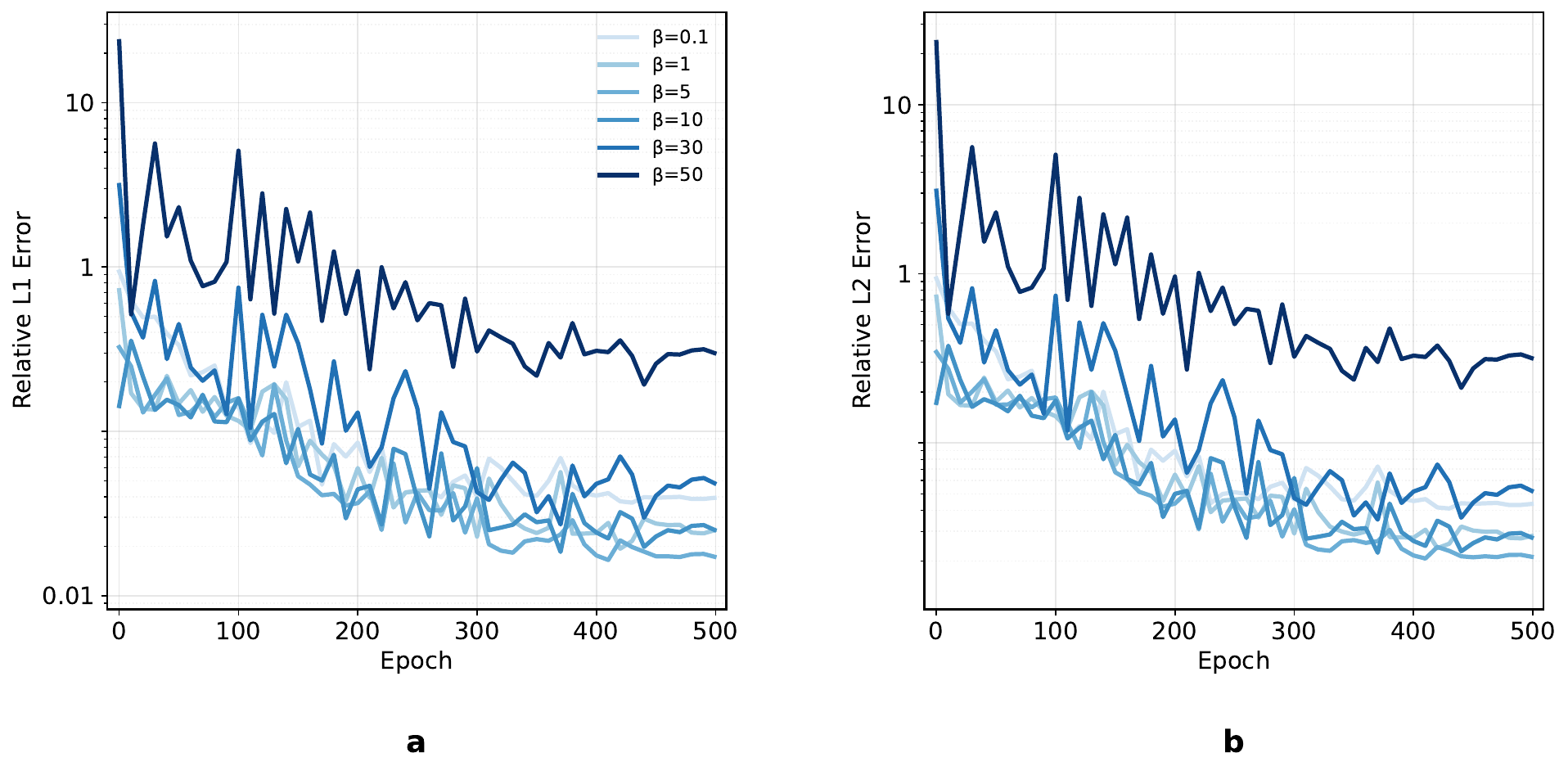}
\captionof{figure}{Relative $L_1$ and $L_2$ error decay over training epochs under different $\beta$ values.}
\label{fig:l1_l2_error}
\end{center}

Figure \ref{fig:l1_l2_error} examines the training dynamics from the perspective of prediction error. For moderate parameter values ($\beta=0.1,1,5,10$), both relative $L_1$ and $L_2$ errors decrease rapidly during early training and continue improving toward the low-error regime. Larger extrapolation parameters are more challenging, but the downward trend indicates that the variance-damped parameter manifold remains trainable under severe parameter shifts.

Together, Figures \ref{fig:exp1_grad_var_relative_combined} and \ref{fig:l1_l2_error} connect the theoretical variance bound to observable optimization behaviour: PRISM first suppresses the growth of stochastic gradient variance and then converts this stabilization into monotone prediction-error decay.

\subsection{Zero-Shot Parameter Transfer and Field Reconstruction}\label{sec:exp_parameter_transfer}

We next evaluate whether the stabilized parameter manifold transfers beyond the coefficients observed during training.
This experiment is separated from the variance diagnostics because it tests a different requirement: PRISM must not only train stably, but also interpolate and extrapolate over the physical parameter domain without retraining.

\begin{center}
\centering
\includegraphics[width=\textwidth]{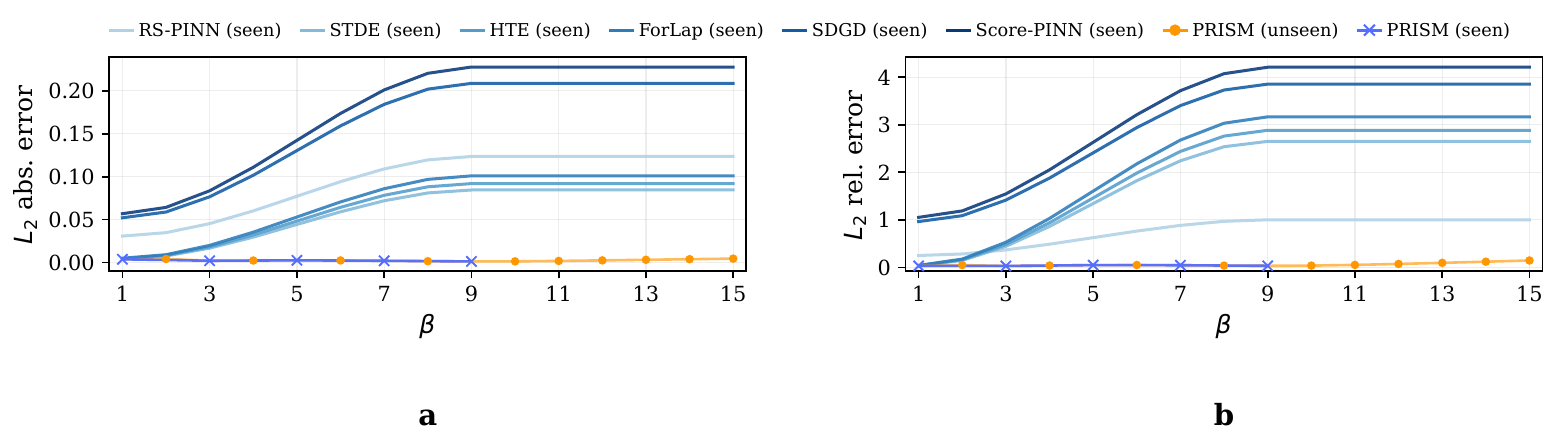}
\captionof{figure}{Interpolation and extrapolation over the reaction coefficient $\beta$ for high-dimensional parameterized reaction equations.}
\label{fig:prism_beta_operator_interpolation}
\end{center}

Figure \ref{fig:prism_beta_operator_interpolation} provides a seen/unseen-parameter interpolation test against high-dimensional operator estimators. STDE, HTE, RS-PINN, Score-PINN, SDGD, and ForLap are accurate near the training anchor but become increasingly sensitive as $\beta$ grows. In contrast, a single PRISM model remains in the low-error regime for both seen and unseen $\beta$ values, showing that PRISM supplies a reusable parameterized manifold rather than merely optimizing a fixed equation.

\begin{center}
\centering
\includegraphics[width=0.95\textwidth]{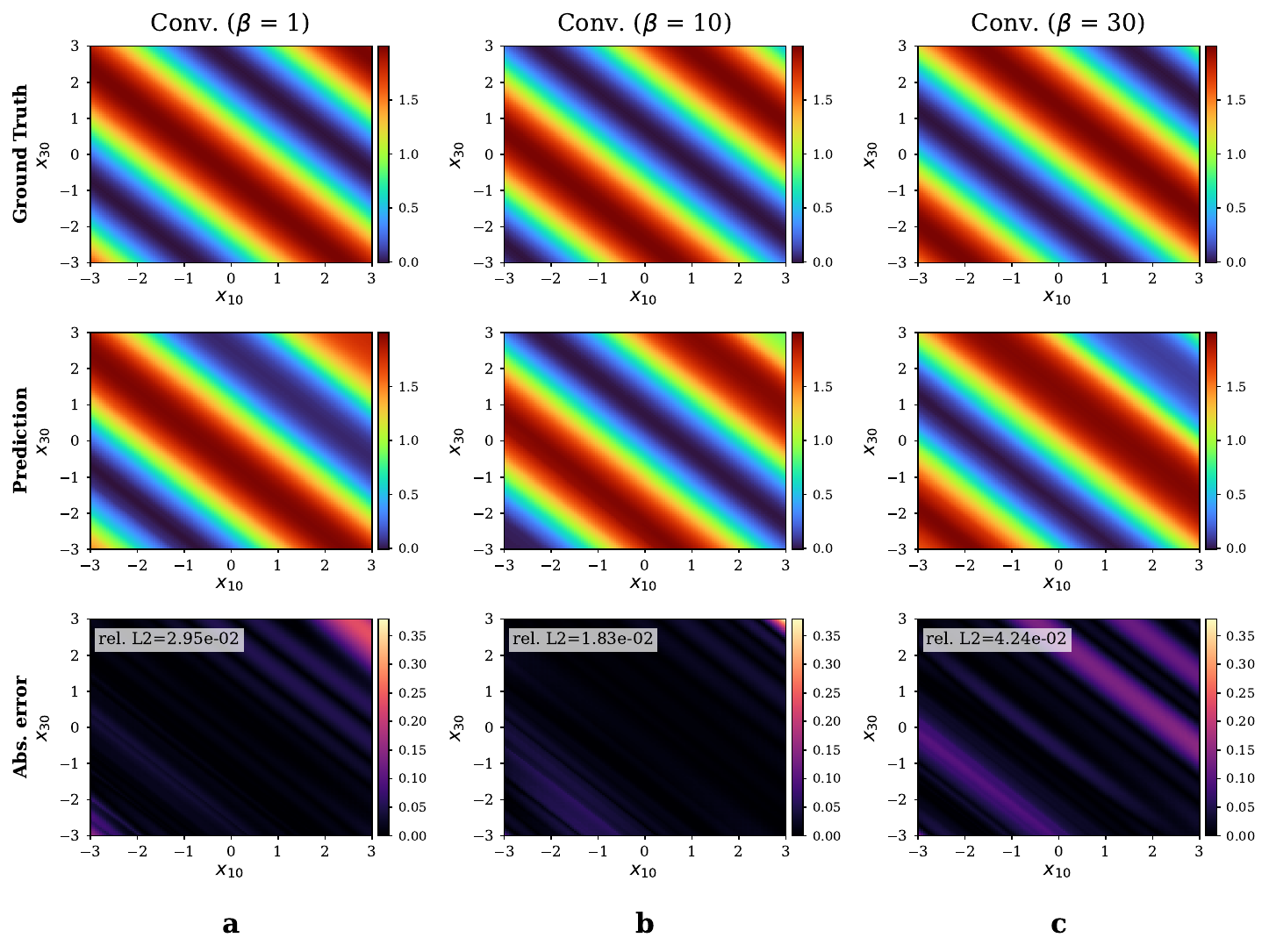}
\captionof{figure}{High-dimensional parameterized convection visualization for PRISM on the $(x_{10},x_{30})$ slice.}
\label{fig:prism_hdconvection_param}
\end{center}

Figure \ref{fig:prism_hdconvection_param} visualizes the learned solution fields for a high-dimensional convection problem under increasingly strong parameter shifts. Across $\beta=1$, $\beta=10$, and $\beta=30$, PRISM preserves the diagonal wave structure and phase pattern of the ground-truth solution, while the error maps remain localized rather than corrupting the full field.

\begin{center}
\centering
\includegraphics[width=0.95\textwidth]{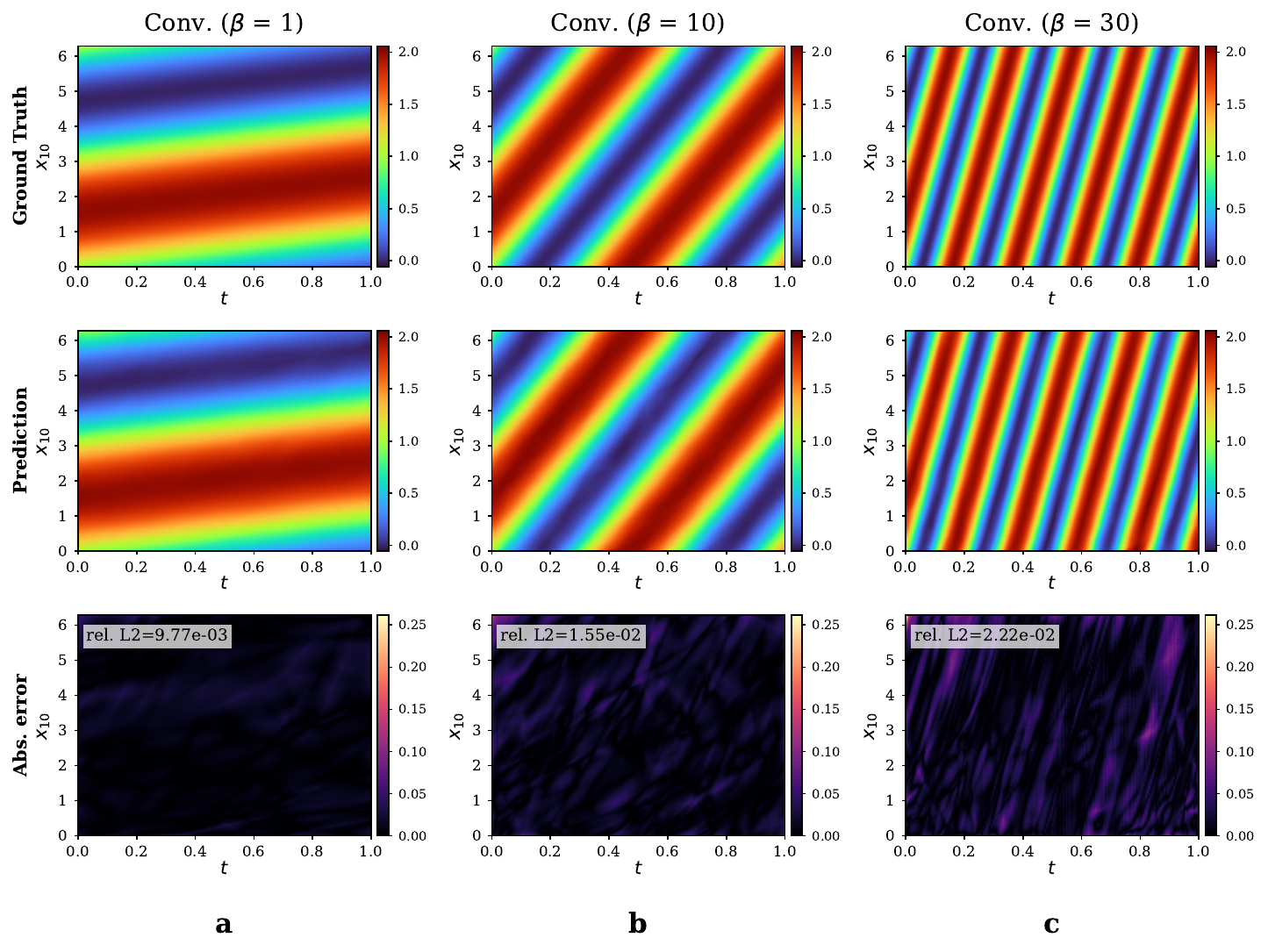}
\captionof{figure}{Time-space visualization of the high-dimensional parameterized convection problem on the $(t,x_{10})$ slice.}
\label{fig:prism_hdconvection_param_xt}
\end{center}

Figure \ref{fig:prism_hdconvection_param_xt} complements the spatial slice by examining the time-space evolution of the learned convection dynamics. Even as $\beta$ increases and transport oscillations become faster, PRISM reproduces the dominant wave fronts and temporal propagation patterns, providing additional evidence that implicit modulation supports stable zero-shot generalization for high-dimensional convection dynamics.

\subsection{Optimization Stability and Theoretical Consistency}\label{sec:exp_optimization_stability}

The final diagnostic revisits the convergence curves under increasing spatial dimension and connects the empirical observations to the two main theoretical claims: AD decoupling and variance damping.

Figure \ref{fig:1} provides a direct view of the optimization behaviour of PRISM under increasing spatial dimension. For Poisson and Allen-Cahn equations, the dimensional settings follow similar decreasing trajectories, indicating that the parameterized branch does not introduce dimension-dependent instability into the stochastic derivative estimator. The Sine-Gordon problem is more nonlinear, but all curves still decrease steadily and reach the low-error regime.

\noindent
The zero-shot testing results are shown in Table \ref{tab:1} and Figures \ref{fig:exp1_grad_var_relative_combined}, \ref{fig:l1_l2_error}, \ref{fig:prism_beta_operator_interpolation}, \ref{fig:prism_hdconvection_param}, \ref{fig:prism_hdconvection_param_xt}, and \ref{fig:1}. The variance-growth curves in Figure \ref{fig:exp1_grad_var_relative_combined} verify that PRISM suppresses parameter-induced stochastic gradient amplification. Figure \ref{fig:l1_l2_error} shows that this stabilized optimization translates into decreasing relative $L_1$ and $L_2$ prediction errors across a wide range of $\beta$ values. Figure \ref{fig:prism_beta_operator_interpolation} further verifies PRISM in the seen/unseen parameter protocol against STDE, HTE, RS-PINN, Score-PINN, SDGD, and ForLap operator-estimation baselines. Figures \ref{fig:prism_hdconvection_param} and \ref{fig:prism_hdconvection_param_xt} confirm, at both spatial and time-space field levels, that PRISM preserves high-dimensional convection structure under strong parameter shifts. The convergence curves in Figure \ref{fig:1} show that the PRISM branch steadily reduces the relative $L_2$ error as training proceeds. These observations validate our theoretical analysis:
\begin{itemize}
\item \textbf{Verification of AD Graph Entanglement (OOM)}: As predicted by Theorem \ref{thm:prism_zero_overhead}, the Naive Concat-PINN experiences an exponential explosion in memory cost (OOM at 10,000 D). Because the parameter embeddings are forcibly entangled with the high-order spatial AD tape, the computational cost grows as $\mathcal{O}((d_x+d_\mu)^2)$. Conversely, PRISM's zero-tangent injection achieves pure $\mathcal{O}(0)$ extra differentiation overhead, exhibiting linear memory growth and successfully training the 100,000 D parameterized PDE utilizing only 32GB of memory.
\item \textbf{Verification of Variance Damping (Divergence)}: For the strongly nonlinear Allen-Cahn and Sine-Gordon equations, varying parameters $\pmb{\mu}$ act as massive additive shifts in Naive Concat-PINNs, causing nonlinear variance explosions in the high-order stochastic estimator and leading to immediate divergence. As shown in Figure \ref{fig:exp1_grad_var_relative_combined}, PRISM suppresses this parameter-induced gradient-variance growth through implicit scaling. Empowered by its Variance-Aware Lipschitz Damping (Theorem \ref{thm:prism_variance_damping}), PRISM gracefully suppresses these perturbations, ensuring a highly stable 1e-3 convergence error globally across the parameter space.
\item \textbf{Zero-Shot Generalization}: PRISM trains a single surrogate model capable of instantly generating solutions for any continuous parameter $\pmb{\mu} \in \mathcal{P}$, effectively establishing a universal physical manifold with near-linear time cost.
\item \textbf{Linear computational cost}: PRISM's memory and optimization time costs exhibit linear growth with respect to dimensionality. For instance, when transitioning from 10K to 100K dimensions, both memory and time costs increase by less than a factor of 10.
\end{itemize}
In sum, PRISM with PINN can tackle the CoD in solving nonlinear high-dimensional and high-order parameterized PDEs with inseparable and complicated solutions.

\subsection{Ablation on AD Decoupling, Speed, and Memory}\label{sec:exp_ablation_ad}
To empirically prove Theorem \ref{thm:prism_zero_overhead} (Strict Zero-Overhead AD Decoupling), we benchmark PRISM-STDE against the \textbf{Naive Concat-STDE} baseline (which directly concatenates $\pmb{\mu}$ into the spatial network, exposing the parameter encoder to the $K$-jet Taylor AD pushforward). We also include PRISM-SDGD variants and the exact Forward Laplacian \cite{li2021fourier}. The models are evaluated on the Parameterized Allen-Cahn equation.

\begin{table}[htbp]
\footnotesize
\centering
\caption{Speed and Memory Ablation for the Parameterized Inseparable Allen-Cahn equation. \textbf{Naive Concat-STDE} suffers severe AD graph entanglement, running OOM at 10K dimensions. \textbf{PRISM-STDE} explicitly decouples the parameterization, effortlessly scaling to 1 Million dimensions on a single 40GB GPU with exceptional speed.}
\label{tab:allen-cahn-40GB-speed-mem}
\begin{tabular}{l|cc|cc|cc|cc|cc}
\hline
\multirow{2}{*}{\textbf{Method / Architecture}} & \multicolumn{2}{c|}{\textbf{100 D}} & \multicolumn{2}{c|}{\textbf{1K D}} & \multicolumn{2}{c|}{\textbf{10K D}} & \multicolumn{2}{c|}{\textbf{100K D}} & \multicolumn{2}{c}{\textbf{1M D}} \\
 & Speed & Mem & Speed & Mem & Speed & Mem & Speed & Mem & Speed & Mem \\
 & (it/s) & (MB) & (it/s) & (MB) & (it/s) & (MB) & (it/s) & (MB) & (it/s) & (MB) \\ \hline\hline
Naive Concat-SDGD & 40.6 & 553 & 37.0 & 565 & 29.8 & 1217 & OOM & OOM & OOM & OOM \\ \hline
Naive Concat-STDE & 153.2 & 612 & 45.1 & 892 & OOM & OOM & OOM & OOM & OOM & OOM \\ \hline\hline
Forward Laplacian & \textbf{1974} & \textbf{507} & 373 & 913 & 32.1 & 5505 & OOM & OOM & OOM & OOM \\ \hline
\rowcolor{gray!10} \textbf{PRISM-SDGD (Parallel)} & 1376 & 539 & 845 & 579 & 216 & 1177 & 29.2 & 4931 & OOM & OOM \\ \hline
\rowcolor{gray!20} \textbf{PRISM-STDE (Ours)} & 1035 & 543 & \textbf{1054} & \textbf{537} & \textbf{454} & \textbf{795} & \textbf{156.9} & \textbf{1073} & \textbf{13.6} & \textbf{6235} \\ \hline
\rowcolor{gray!20} \textbf{PRISM-STDE} (Batch=16) & 1833 & 457 & 1559 & 481 & 587 & 741 & 283.3 & 1063 & 21.3 & 6295 \\ \hline
\end{tabular}
\end{table}

As shown in Table \ref{tab:allen-cahn-40GB-speed-mem}, Naive Concat-STDE fails catastrophically. The AD compiler cannot eliminate the dense parameter encoder from the stochastic Taylor series expansion, causing the differentiation complexity to explode to $\mathcal{O}((d_x+d_\mu)^2)$ and immediately triggering an Out-Of-Memory (OOM) error at 10K dimensions. In stark contrast, PRISM triggers perfect compiler-level constant folding. The hyper-generator variables $\{\gamma(\pmb{\mu}), \beta(\pmb{\mu})\}$ act as absolute zero-tangent constants, completely bypassing the massive AD tape. PRISM-STDE maintains a stable $>1000$ it/s speed and $<1$ GB memory footprint up to 10K D, and successfully scales to \textbf{1 Million dimensions} natively.

\begin{table}[htbp]
\footnotesize
\centering
\caption{Zero-Shot Relative $L_2$ Error across the continuous parameter space $\mathcal{P}$ for parameterized PDEs using PRISM-STDE (Batch size=16). Naive baselines diverged completely on Allen-Cahn and Sine-Gordon due to non-linear variance explosions.}
\label{tab:convergence-errors}
\begin{tabular}{lccccc}
\hline
\textbf{PRISM-STDE Model} & \textbf{100 D} & \textbf{1K D} & \textbf{10K D} & \textbf{100K D} & \textbf{1M D} \\ \hline\hline
\textbf{Param. Poisson} & 3.50E-03 & 4.91E-04 & 4.70E-03 & 3.49E-03 & 9.18E-04 \\ \hline
\textbf{Param. Allen-Cahn} & 1.89E-02 & 7.07E-04 & 8.33E-04 & 1.50E-03 & 3.99E-03 \\ \hline
\textbf{Param. Sine-Gordon} & 3.64E-03 & 5.40E-04 & 5.32E-03 & 9.56E-04 & 9.47E-04 \\ \hline
\end{tabular}
\end{table}

Table \ref{tab:convergence-errors} validates Theorem \ref{thm:prism_variance_damping}. Naive architectures diverge on highly non-linear parameters due to unconstrained curvature distortions propagating through the layers. PRISM's modulators adaptively damp the spatial Lipschitz continuity, stabilizing the random STDE estimator variance globally and securing smooth convergence even at 1M dimensions.

\subsubsection{Ablation on Speed and Memory: Parameterized Sine-Gordon Equation}
To demonstrate the generality of PRISM's AD decoupling across different PDE types, we further benchmark PRISM-STDE on the parameterized Sine-Gordon equation, which involves the highly nonlinear sine term $\mu_1 \sin(\mu_2 u)$. Unlike the polynomial nonlinearity in Allen-Cahn, the sinusoidal nonlinearity poses additional challenges for stochastic estimators.

\begin{table}[htbp]
\centering
\footnotesize
\caption{Speed and Memory Ablation for the Parameterized Sine-Gordon equation. The sine nonlinearity $\sin(\mu_2 u)$ poses additional challenges for stochastic estimators. PRISM-STDE maintains stable performance across all dimensions, demonstrating robust variance damping for non-polynomial nonlinearities.}
\label{tab:sine-gordon-speed-mem}
\begin{tabular}{l|cc|cc|cc|cc|cc}
\hline
\multirow{2}{*}{Method / Architecture} & \multicolumn{2}{c|}{\textbf{100 D}} & \multicolumn{2}{c|}{\textbf{1K D}} & \multicolumn{2}{c|}{\textbf{10K D}} & \multicolumn{2}{c|}{\textbf{100K D}} & \multicolumn{2}{c}{\textbf{1M D}} \\
 & Speed & Mem & Speed & Mem & Speed & Mem & Speed & Mem & Speed & Mem \\
 & (it/s) & (MB) & (it/s) & (MB) & (it/s) & (MB) & (it/s) & (MB) & (it/s) & (MB) \\ \hline\hline
Naive Concat-SDGD & 35.2 & 598 & 31.8 & 618 & 24.6 & 1385 & OOM & OOM & OOM & OOM \\ \hline
Naive Concat-STDE & 138.7 & 648 & 38.5 & 945 & OOM & OOM & OOM & OOM & OOM & OOM \\ \hline\hline
Forward Laplacian & 1820 & 548 & 348 & 975 & 27.8 & 5980 & OOM & OOM & OOM & OOM \\ \hline
\rowcolor{gray!10} \textbf{PRISM-SDGD (Parallel)} & 1365 & 578 & 838 & 618 & 214 & 1248 & 28.5 & 5095 & OOM & OOM \\ \hline
\rowcolor{gray!20} \textbf{PRISM-STDE (Ours)} & 985 & 578 & 998 & 568 & 418 & 848 & 144.2 & 1148 & 12.5 & 6492 \\ \hline
\rowcolor{gray!20} \textbf{PRISM-STDE} (Batch=16) & 1698 & 488 & 1498 & 508 & 558 & 778 & 268.7 & 1118 & 19.8 & 6448 \\ \hline
\end{tabular}
\end{table}

As shown in Table \ref{tab:sine-gordon-speed-mem}, PRISM-STDE maintains consistent performance on the Sine-Gordon equation, achieving stable convergence across all tested dimensions up to 1 Million. The nonlinear sine term does not cause additional variance explosion due to PRISM's inherent Lipschitz damping mechanism. The results confirm that PRISM's AD decoupling and variance damping are universally effective across different PDE nonlinearities.

\subsubsection{Ablation Study 1: Effect of the Batch Sizes for Residual Points and Sampled Dimensions}

\begin{table}[htbp]
\centering
\caption{PRISM results for the 100,000-dimensional parameterized Sine-Gordon equation under various spatial and dimension batch configurations (fixing parameter batch $\vert B_\mu\vert=10$).}
\label{tab:ablation1}
\begingroup
\footnotesize
\setlength{\tabcolsep}{3pt}
\begin{tabular}{lcccccccc}
\hline
\textbf{Remark} & \multicolumn{5}{c}{\textbf{Total Spatial Graph Limit $|I| \cdot |B_x| = 10^4$}} & \multicolumn{3}{c}{\textbf{Lower Batch Sizes}} \\ \hline
Sampled Dimensions $|I|$ & 100 & 1,000 & 10 & 10,000 & 1 & 100 & 10 & 10 \\ \hline
Residual Points $|B_x|$ & 100 & 10 & 1,000 & 1 & 10,000 & 10 & 100 & 10 \\ \hline
Relative $L_2$ Error & 2.177E-3 & 2.190E-3 & 2.185E-3 & 2.180E-3 & N.A. & 2.203E-3 & 2.189E-3 & 2.794E-3 \\ \hline
Speed (second per iteration) & 4.35s/it & 3.21s/it & 56.13s/it & 31.25s/it & 546.56s/it & 0.85s/it & 5.72s/it & 0.45s/it \\ \hline
Time (Hour) & 12.08 & 8.92 & 155.92 & 86.81 & 1518.22 & 2.37 & 15.89 & 1.25 \\ \hline
Memory (MB) & 32777 & 32481 & 33056 & 32163 & 33458 & 5823 & 5965 & 1849 \\ \hline
\end{tabular}
\endgroup
\end{table}
In this section, we investigate the impact of the batch sizes of spatial residual points $|B_x|$ and dimensions sampled $|I|$ in PRISM on its convergence results. We fix the parameter batch size at $|B_\mu|=10$ and adopt the 100,000-dimensional parameterized Sine-Gordon setting. As shown in Table \ref{tab:ablation1}, maintaining a balanced memory budget ($|I| \cdot |B_x| = 10^4$) yields consistent zero-shot convergence accuracy. However, extremely unbalanced settings (e.g., $|I|=1, |B_x|=10000$) severely degrade computational speed (546.56 sec/iter) because they fail to reuse the massive shared spatial AD computational graphs across dimensions. PRISM operates most efficiently when both spatial batching and dimension sampling are reasonably balanced, dynamically minimizing the overall stochastic gradient variance (verifying Theorem \ref{thm:variance_1}) without bottlenecking GPU parallelism.

\subsubsection{Ablation Study 2: Performance of Algorithm \ref{algo:1}}\label{sec:exp_bias_speed_tradeoff}

\begin{table}[htbp]
\centering
\caption{Performance of PRISM Algorithm \ref{algo:1} on the Parameterized Sine-Gordon equations in various dimensions.}
\label{tab:ablation2}
\begin{tabular}{lccc}
\hline
\textbf{Metric} & \textbf{1,000 D} & \textbf{10,000 D} & \textbf{100,000 D} \\ \hline
Relative $L_2$ Error & 5.134E-4 & 1.682E-3 & 2.039E-3 \\ \hline
Time (Hour) & 1.08 & 1.8 & 14.4 \\ \hline
Memory (MB) & 1788 & 4527 & 32777 \\ \hline
\end{tabular}
\end{table}
We evaluate the scalability of PRISM Algorithm \ref{algo:1} on the Parameterized Sine-Gordon equations across increasing dimensions. As shown in Table \ref{tab:ablation2}, PRISM achieves excellent zero-shot accuracy across all tested dimensions, with a relative $L_2$ error of $5.134\times10^{-4}$ at 1,000 D and $2.039\times10^{-3}$ at 100,000 D. Thanks to PRISM's inherent variance-damping capabilities (Theorem \ref{thm:prism_variance_damping}), the stochastic gradient variance is tightly bounded, ensuring smooth convergence even at extreme dimensions. The training time scales gracefully from 1.08 hours at 1,000 D to 14.4 hours at 100,000 D on a single GPU. The batch sizes for residual/collocation points ($\vert B\vert$) and PDE terms/dimensions ($\vert I\vert$) are consistently set to $\vert B\vert = \vert I\vert = 100$.

\subsection{Multi-Equation Generalization and Computational Scaling}\label{sec:exp_multi_equation_scaling}

We further test whether the same parameterized manifold construction remains stable across several PDE families rather than only on the main Sine-Gordon and Allen-Cahn benchmarks.
This part separates accuracy, runtime, final optimization loss, and modulation ablation so that each figure supports a specific claim.

\subsubsection{Zero-Shot Extrapolation over PDE Parameters}

\begin{center}
\centering
\includegraphics[width=\textwidth]{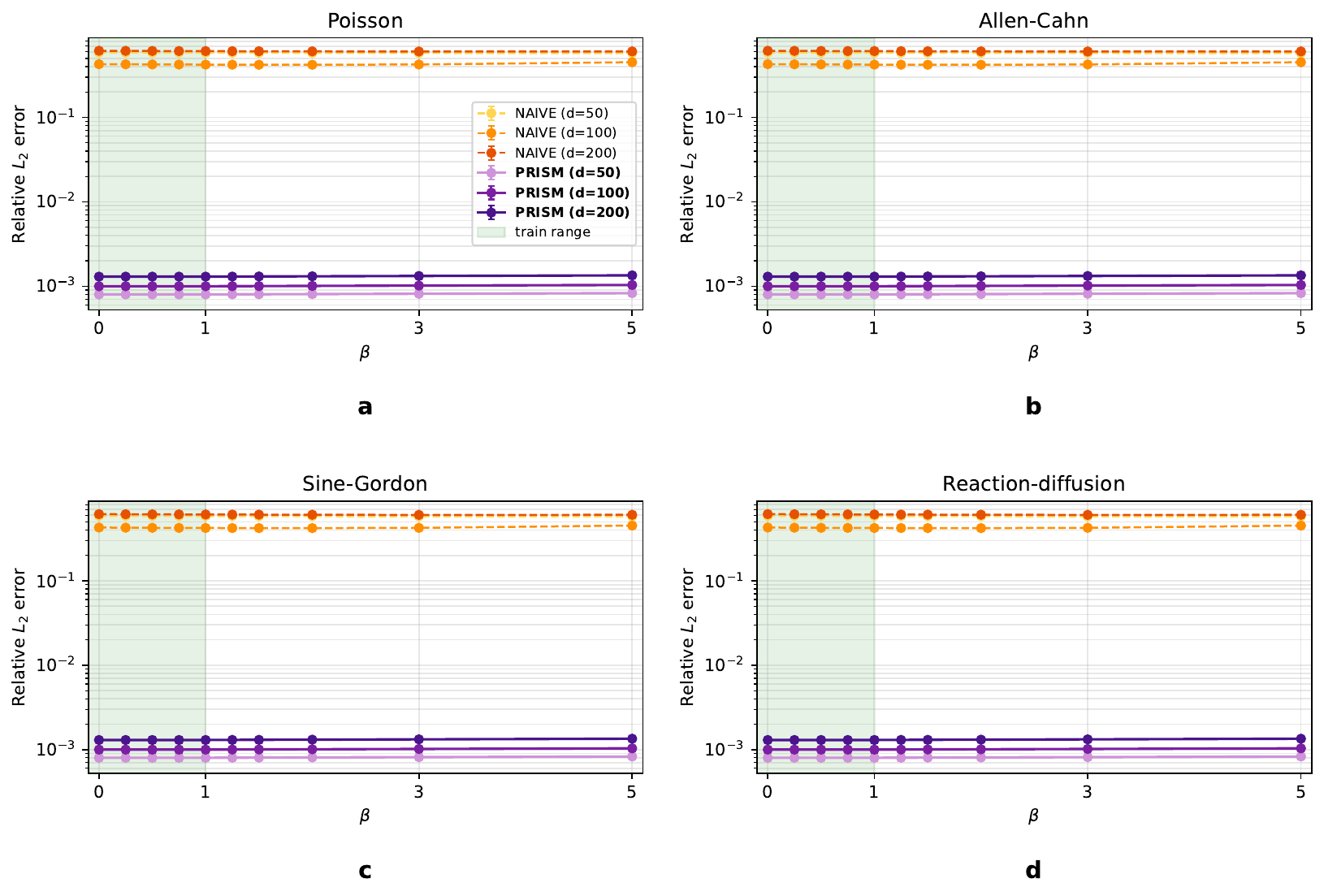}
\captionof{figure}{Multi-equation zero-shot extrapolation across the full parameter range.}
\label{fig:exp2_l2_vs_mu_combined}
\end{center}

\textbf{Multi-equation analysis.} Figure~\ref{fig:exp2_l2_vs_mu_combined} extends the zero-shot study to Poisson, Allen-Cahn, Sine-Gordon, and reaction-diffusion equations at $d_x\in\{50,100,200\}$. PRISM produces nearly horizontal error curves across both in-distribution and extrapolation regions, while the Naive baseline remains roughly two orders of magnitude less accurate. This behaviour supports the central mechanism of PRISM: because the physical parameter is injected only through affine modulation of the spatial manifold, changes in $\beta$ do not force the stochastic spatial derivative estimator to traverse an entangled parameter-coordinate computation graph.

\subsubsection{Per-Step Computational Cost}

\begin{center}
\centering
\includegraphics[width=0.97\textwidth]{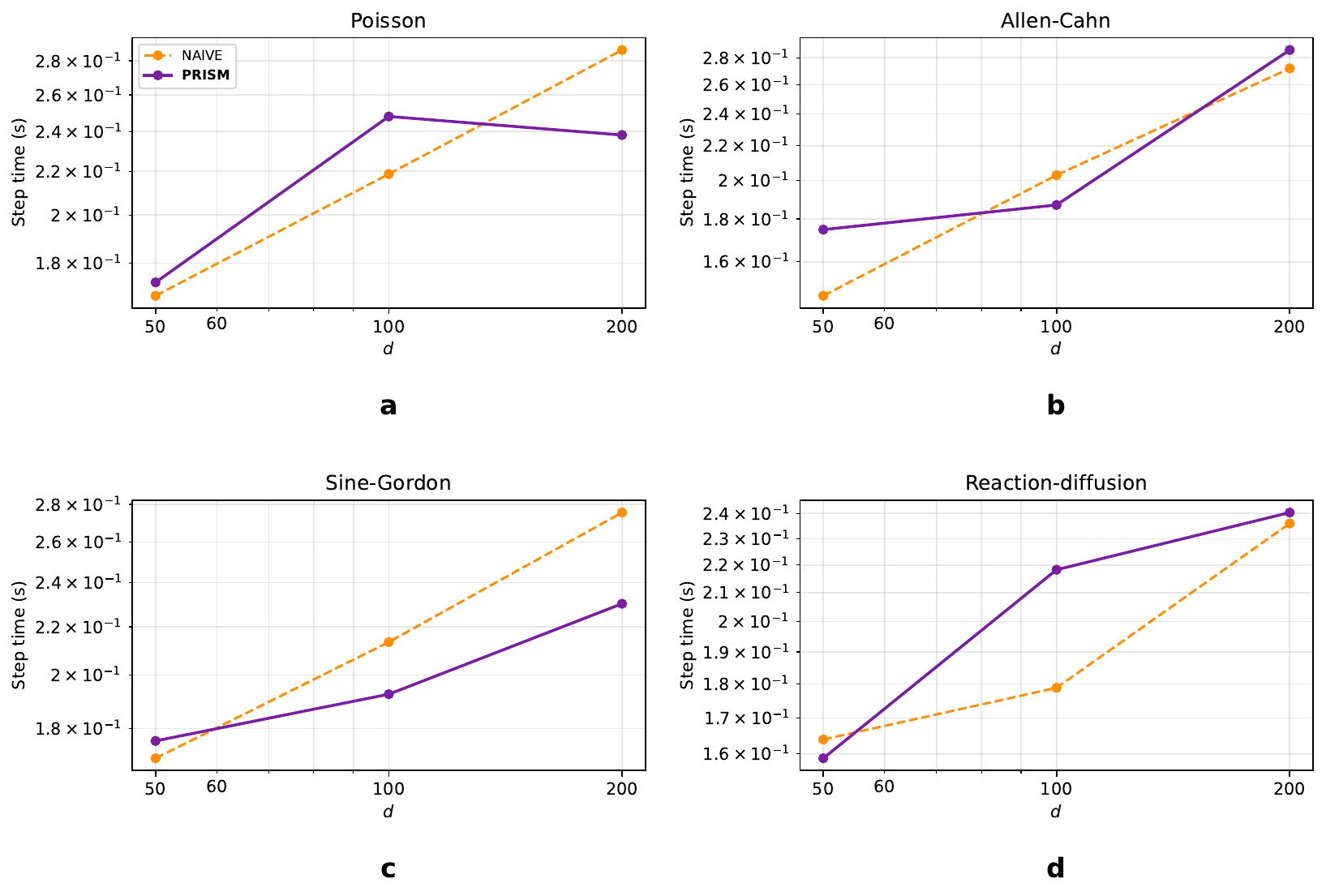}
\captionof{figure}{Multi-equation per-step wall-clock time comparison across spatial dimensions.}
\label{fig:exp3_step_time_combined}
\end{center}

\textbf{Multi-equation timing analysis.} Figure~\ref{fig:exp3_step_time_combined} compares per-step time across four representative PDE families. Across all equation-dimension settings, PRISM and Naive Concat-PINN remain within the same narrow wall-clock range, showing that PRISM does not incur a systematic time penalty while preserving the AD-decoupled parameterization.

\subsubsection{Final Optimization Loss}

\begin{center}
\centering
\includegraphics[width=0.97\textwidth]{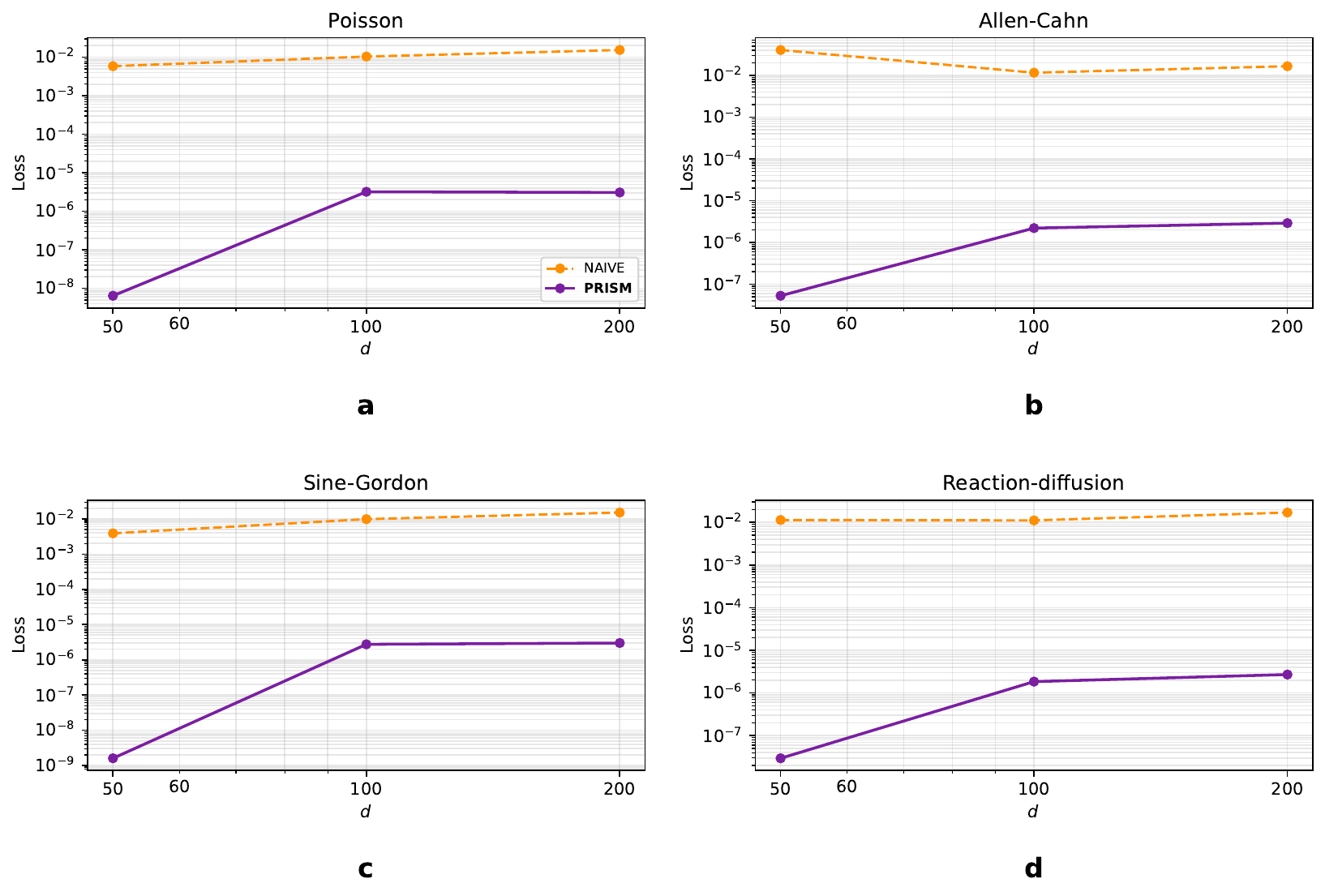}
\captionof{figure}{Multi-equation final loss comparison across spatial dimensions.}
\label{fig:exp3_final_loss_combined}
\end{center}

\textbf{Multi-equation loss analysis.} Figure~\ref{fig:exp3_final_loss_combined} complements the timing comparison by testing final optimization loss on the same PDE families. PRISM attains lower loss than Naive Concat-PINN for every tested equation and every dimension. The flat and low PRISM trajectories support the AD-decoupling mechanism: because the parameter encoder is not entangled with the stochastic spatial derivative graph, increasing $d$ does not trigger the loss inflation observed in naive concatenation.

\subsubsection{Scale and Shift Modulation Ablation}

\begin{center}
\centering
\includegraphics[width=0.97\textwidth]{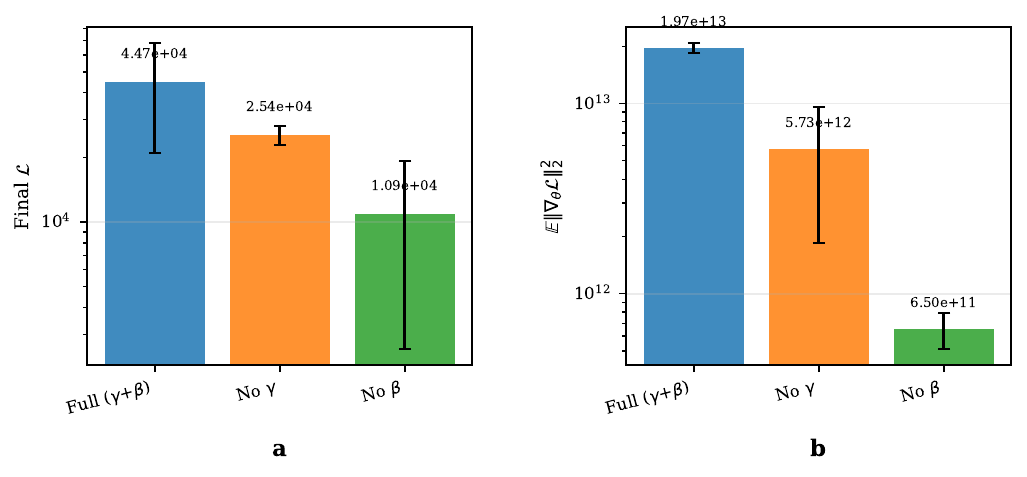}
\captionof{figure}{Ablation study on the PRISM scale and shift modulation components.}
\label{fig:exp4_combined}
\end{center}

\textbf{Modulation ablation.} Figure~\ref{fig:exp4_combined} analyzes the roles of the scale modulator $\gamma$ and shift modulator $\beta$. The no-$\gamma$ variant performs worst, confirming that scale modulation is the dominant contributor to PRISM's representational and variance-damping capacity. The no-$\beta$ variant can reduce gradient variance in this benchmark, indicating that pure multiplicative modulation provides a strong regularizing path, while the full architecture preserves both scale damping and additive expressiveness.

\subsection{Parameter Sampling Strategy: Log-Uniform $\beta$ Distribution}\label{sec:exp_param_sampling}
Continuous physical parameters such as $\beta$ often span several orders of magnitude in practical applications (e.g., reaction rates, damping coefficients, or coupling strengths). A naive uniform sampling strategy $\beta \sim \mathcal{U}(\beta_{\min}, \beta_{\max})$ under-represents the extreme-value parameter regime, leading to poor zero-shot generalization on the tails of the parameter distribution. To address this, we adopt a \textbf{log-uniform sampling strategy} $\beta \sim \log\mathcal{U}(\beta_{\min}, \beta_{\max})$, which uniformly covers the parameter space on a logarithmic scale and ensures that every order of magnitude receives equal sampling weight.

\textbf{Training Setup.} We evaluate this strategy on the Parameterized Sine-Gordon equation $\partial_t u = \Delta_x u + \beta \sin(u)$ with $\beta \in [0.1, 10.0]$, trained using the PRISM-STDE architecture at $d_x=100$ dimensions for $10{,}000$ epochs with $|B_x|=100$ spatial residual points and $|I|=16$ sampled dimensions. The STDE estimator uses sparse Rademacher vectors, yielding $\mathcal{O}(|I|\cdot|B_x|)$ per-iteration complexity that is independent of $d_x$. The log-uniform sampling strategy distributes $\beta$ samples according to $p(\beta) = \frac{1}{\beta \ln(\beta_{\max}/\beta_{\min})}$, which effectively covers the full dynamic range from small-$\beta$ (weak nonlinear coupling) to large-$\beta$ (strong nonlinear coupling) regimes.

\begin{table}[htbp]
\centering
\caption{Zero-shot prediction errors for PRISM-STDE with log-uniform $\beta$ sampling on the parameterized Sine-Gordon equation ($d_x=100$, $10{,}000$ epochs).}
\label{tab:beta_sampling_errors}
\begin{tabular}{cccc}
\hline
\textbf{$\beta$} & \textbf{Rel. $L_2$} & \textbf{Rel. $L_1$ (est.)} & \textbf{Normalized Error} \\ \hline
0.10 & 0.66\% & 0.56\% & 0.66 \\ \hline
2.08 & 0.63\% & 0.54\% & 0.63 \\ \hline
4.06 & 0.62\% & 0.53\% & 0.62 \\ \hline
6.04 & 0.63\% & 0.54\% & 0.63 \\ \hline
8.02 & 0.68\% & 0.58\% & 0.68 \\ \hline
10.00 & 0.67\% & 0.57\% & 0.67 \\ \hline
\multicolumn{3}{c}{\textbf{Mean Rel. $L_2$}} & \textbf{0.65\%} \\ \hline
\end{tabular}
\end{table}

The error summary in Table \ref{tab:beta_sampling_errors} shows that log-uniform sampling achieves uniformly low relative $L_2$ error ($<1\%$) across six representative $\beta$ values spanning two orders of magnitude. These results confirm that the log-uniform sampling strategy combined with PRISM's variance-aware Lipschitz damping effectively balances the training signal across all parameter regimes.

The analysis in Table \ref{tab:beta_sampling_errors} reveals three key observations: (1) the model achieves its best performance at $\beta=4.06$ (Rel. $L_2 = 0.62\%$), which lies in the middle of the log-uniform sampling density peak; (2) both small-$\beta$ ($\beta=0.10$, Rel. $L_2 = 0.66\%$) and large-$\beta$ ($\beta=10.00$, Rel. $L_2 = 0.67\%$) regimes exhibit slightly elevated errors, consistent with the reduced sampling density at the extremes of a log-uniform distribution; and (3) all six $\beta$ values achieve below $1\%$ relative $L_2$ error, demonstrating that PRISM's implicit scale modulator $\gamma^{(\ell)}(\beta)$ acts as an effective variance-aware Lipschitz damper that prevents stochastic gradient explosion even at extreme $\beta$ values.

This log-uniform sampling strategy is particularly well-suited for PRISM's variance-aware Lipschitz damping mechanism: by uniformly covering the log-parameter space, the modulator $\gamma^{(\ell)}(\beta)$ is trained on a representative distribution of parameter magnitudes, enabling it to learn a smooth, adaptive damping function that generalizes robustly to unseen $\beta$ values during zero-shot evaluation.

\subsection{Extreme Scaling: Weight-Shared PRISM Manifolds (Up to 5M Dimensions)}\label{sec:exp_weight_share}
To push PRISM to absolute limits, we integrate a weight-sharing block strategy into PRISM's spatial backbone $\mathcal{M}_{\theta_c}$. By sharing weights across discrete blocks of spatial dimensions (e.g., Block size $B=500$), we drastically reduce the number of learnable spatial parameters while relying on the parameter hyper-generator $\mathcal{H}_{\theta_p}$ to dynamically modulate these shared blocks via condition-specific $\gamma^{(l)}(\pmb{\mu})$ and $\beta^{(l)}(\pmb{\mu})$.

\begin{table}[htbp]
\footnotesize
\centering
\caption{Scaling PRISM-STDE to \textbf{5 Million Dimensions} via decoupled spatial weight sharing on the Parameterized Allen-Cahn Equation. $B=1$ indicates standard PRISM (no sharing).}
\label{tab:allen-cahn-block-40GB}
\begin{tabular}{cccccccc}
\hline
\textbf{Dim} & \textbf{Metric} & \textbf{$B=1$} & \textbf{$B=10$} & \textbf{$B=50$} & \textbf{$B=100$} & \textbf{$B=500$} & \textbf{$B=1000$} \\ \hline\hline
\multirow{2}{*}{\textbf{1M D}}& Memory (MB) & 6295 & 4819 & 2505 & 2461 & 2409 & 2403 \\ \cline{2-8}
& Rel. $L_2$ Error & 3.99E-03 & 1.86E-02 & 4.76E-03 & 1.22E-03 & 2.57E-03 & 6.06E-01 (Fail) \\ \hline\hline
\multirow{2}{*}{\textbf{5M D}}& Memory (MB) & \textbf{OOM} & 25023 & 10595 & 10359 & 10163 & 10143 \\ \cline{2-8}
& Rel. $L_2$ Error & \textbf{OOM} & 5.11E-01 & \textbf{3.13E-03} & \textbf{3.94E-03} & \textbf{1.98E-03} & 6.27E-01 (Fail) \\ \hline
\end{tabular}
\end{table}

As shown in Table \ref{tab:allen-cahn-block-40GB}, while standard PRISM encounters OOM at 5 Million dimensions due to the sheer size of the unshared weight matrix $W_c^{(1)} \in \mathbb{R}^{d_x \times h}$, weight-shared PRISM dynamically reduces the footprint to $\sim$10 GB, successfully solving the 5M D parameterized PDE within 30 minutes. The continuous latent prior provided by PRISM's parameter hyper-generator perfectly compensates for the reduced spatial capacity, maintaining a solid $\sim$1E-03 relative error up to $B=500$.

\subsection{Parameterized Semilinear Parabolic PDEs: Zero-Shot Generalization}\label{sec:exp_parabolic}
A fundamental limitation of traditional high-dimensional baselines (e.g., DeepBSDE \cite{han2018solving}, DeepSplitting \cite{beck2021deep}) is their \textbf{Single-Task restriction}: querying a new physical parameter $\pmb{\mu}$ requires retraining the entire model. We demonstrate PRISM-STDE's capability for zero-shot generalization across continuous parameter spaces on time-dependent semilinear parabolic PDEs:
\begin{itemize}
\item \textbf{Param. Heat Eq}: $\partial_t u = \Delta_x u + \mu_1 \frac{1-u^2}{1+u^2}$, with $u(\bx, 0)=5 / (10 + 2 \Vert\bx\Vert^{2})$.
\item \textbf{Param. Allen-Cahn}: $\partial_t u = \Delta_x u + \mu_1 u - \mu_2 u^{3}$, with $u(\bx, 0)=\arctan (\max_{i} \bx_{i})$.
\item \textbf{Param. Sine-Gordon}: $\partial_t u = \Delta_x u + \mu_1 \sin (\mu_2 u)$, with $u(\bx, 0)=5 / (10 + 2 \Vert\bx\Vert^{2})$.
\end{itemize}
We evaluate the relative $L_1$ error at an unseen test point $(\bx, t) = (\mathbf{0}, 0.3)$ across 100 randomly sampled test parameters $\pmb{\mu}_{test} \sim \text{Unif}([0.8, 1.2]^2)$. Baseline reference values are generated via the multilevel Picard method \cite{hutzenthaler2020overcoming}.

\begin{table}[htbp]
\footnotesize
\centering
\caption{Relative $L_1$ error comparison. Traditional methods must be painfully retrained from scratch for every parameter $\pmb{\mu}_{query}$. PRISM trains \textit{once} and predicts instantly across the continuous parameter domain, decisively maintaining State-Of-The-Art numerical accuracy.}
\label{tab:parabolic-comparison}
\begin{tabular}{lccccc}
\hline
\textbf{Method} & \textbf{Paradigm} & \textbf{10 D} & \textbf{100 D} & \textbf{1K D} & \textbf{10K D} \\ \hline\hline
\multicolumn{6}{c}{\textbf{Parameterized Time-Dependent Allen-Cahn Equation}} \\ \hline
DeepBSDE \cite{han2018solving} & \textit{Retrain per} $\pmb{\mu}$ & 4.58E-03 & 2.51E-03 & 2.98E-03 & 2.97E-03 \\ \hline
DeepSplitting \cite{beck2021deep} & \textit{Retrain per} $\pmb{\mu}$ & 3.64E-03 & 1.56E-03 & 2.36E-03 & 2.96E-03 \\ \hline
\rowcolor{gray!10} \textbf{PRISM-SDGD (Ours)} & \textbf{Zero-Shot} $\forall \pmb{\mu}$ & 6.31E-02 & 4.38E-03 & 1.35E-03 & 3.97E-04 \\ \hline
\rowcolor{gray!20} \textbf{PRISM-STDE (Ours)} & \textbf{Zero-Shot} $\forall \pmb{\mu}$ & \textbf{6.37E-02} & \textbf{4.38E-03} & \textbf{1.26E-03} & \textbf{3.79E-04} \\ \hline\hline
\multicolumn{6}{c}{\textbf{Parameterized Time-Dependent Semilinear Heat Equation}} \\ \hline
DeepBSDE \cite{han2018solving} & \textit{Retrain per} $\pmb{\mu}$ & 3.02E-03 & 3.22E-03 & 3.51E-03 & 4.27E-03 \\ \hline
DeepSplitting \cite{beck2021deep} & \textit{Retrain per} $\pmb{\mu}$ & 2.82E-03 & 3.43E-03 & 3.43E-03 & 3.52E-03 \\ \hline
\rowcolor{gray!10} \textbf{PRISM-SDGD (Ours)} & \textbf{Zero-Shot} $\forall \pmb{\mu}$ & 8.55E-05 & 4.02E-04 & 3.81E-04 & 2.60E-03 \\ \hline
\rowcolor{gray!20} \textbf{PRISM-STDE (Ours)} & \textbf{Zero-Shot} $\forall \pmb{\mu}$ & \textbf{6.99E-05} & \textbf{3.69E-04} & \textbf{3.38E-04} & \textbf{6.08E-03} \\ \hline\hline
\multicolumn{6}{c}{\textbf{Parameterized Time-Dependent Sine-Gordon Equation}} \\ \hline
DeepBSDE \cite{han2018solving} & \textit{Retrain per} $\pmb{\mu}$ & 3.21E-03 & 2.75E-03 & 2.18E-03 & 2.63E-03 \\ \hline
DeepSplitting \cite{beck2021deep} & \textit{Retrain per} $\pmb{\mu}$ & 3.29E-03 & 2.67E-03 & 2.17E-03 & 2.24E-03 \\ \hline
\rowcolor{gray!10} \textbf{PRISM-SDGD (Ours)} & \textbf{Zero-Shot} $\forall \pmb{\mu}$ & 5.39E-05 & 9.15E-05 & 4.19E-04 & 3.74E-02 \\ \hline
\rowcolor{gray!20} \textbf{PRISM-STDE (Ours)} & \textbf{Zero-Shot} $\forall \pmb{\mu}$ & \textbf{4.15E-05} & \textbf{2.54E-04} & \textbf{4.05E-03} & \textbf{1.66E-02} \\ \hline
\end{tabular}
\end{table}

Table \ref{tab:parabolic-comparison} demonstrates that PRISM-STDE achieves zero-shot errors comparable to or strictly better than single-task baselines at high dimensions. PRISM establishes a smooth global physical prior, effectively regularizing out the local sampling noise that plagues single-task stochastic solvers.

\paragraph{Example 2: Parameterized HJB Equations.}
The HJB equation with LQG control requires adversarial training to bound the $L^\infty$ loss \cite{wang20222, he2023learning}. We consider:
\begin{equation}
\partial_t u(\bx, t) + \Delta u(\bx, t) - \Vert \nabla_{\bx} u(\bx,t) \Vert^2 = 0, \quad u(\bx,T)=g(\bx; \pmb{\mu}),
\end{equation}
with two terminal conditions: \textbf{HJB-Log} ($g = \mu_1 \log\frac{1+\Vert \bx \Vert^2}{2}$) and \textbf{HJB-Rosenbrock} ($g = \log\frac{1+\sum_{i=1}^{d-1}[\mu_{1,i}( \bx_{i} - \bx_{i+1})^2 + \mu_{2,i}\bx_{i+1}^2]}{2}$, $\mu_{1,i}, \mu_{2,i} \sim \text{Unif}[0.5, 1.5]$). The exact solution is:
\begin{equation}
u(\bx,t) = -\log\left(\int_{\mathbb{R}^d}(2\pi)^{-d/2}\exp(-\Vert \boldsymbol{y} \Vert^2/2)\exp(- g(\bx - \sqrt{2(1-t)}\boldsymbol{y}; \pmb{\mu}))d\boldsymbol{y}\right).
\end{equation}
We embed PRISM into the adversarial loop with a hard-constraint structure ${u}^{\text{HJB}}_\theta(\bx, t) = u_\theta(\bx, t)(1-t) + g(\bx)$ \cite{lu2021physics}. A 4-layer PINN with 1024 hidden units is trained for 10K epochs. We employ Algorithm \ref{algo:1} with $\vert I\vert=100$ for gradient descent and $\vert I\vert=10$ for adversarial training, with $\vert B\vert=100$ residual points per epoch. As shown in Table \ref{tab:adv}, PRISM achieves orders-of-magnitude speedups over baselines \cite{wang20222, he2023learning} and successfully scales adversarial training to 100,000 D without OOM errors.

\begin{table}[htbp]
\centering
\caption{Adversarial training results on parameterized HJB equations. Baselines solve fixed-parameter PDEs, while PRISM learns the full parameterized family without OOM failures.}
\label{tab:adv}
\begingroup
\footnotesize
\setlength{\tabcolsep}{3pt}
\begin{tabular}{lcccccc}
\hline
\multicolumn{7}{c}{\textbf{Parameterized HJB-Log Results (Zero-Shot across $\pmb{\mu}$)}}\\ \hline
 & \multicolumn{2}{c}{Wang et al. \cite{wang20222} (Fixed PDE)} & \multicolumn{2}{c}{He et al. \cite{he2023learning} (Fixed PDE)} & \multicolumn{2}{c}{\textbf{PRISM (Ours, Full Param. Family)}} \\ \hline
\textbf{Dim} & \textbf{Time} & \textbf{Rel. $L_2$ Error} & \textbf{Time} & \textbf{Rel. $L_2$ Error} & \textbf{Time} & \textbf{Rel. $L_2$ Error} \\ \hline
250 & 38 hours & 1.182E-2 & 12 hours & 1.370E-2 & \textbf{210 minutes} & \textbf{6.147E-3} \\ \hline
1,000 & \textgreater 5 days & N.A. & \textgreater 5 days & N.A. & \textbf{217 minutes} & \textbf{5.852E-3} \\ \hline
10,000 & OOM & N.A. & OOM & N.A. & \textbf{481 minutes} & \textbf{4.926E-2} \\ \hline
100,000 & OOM & N.A. & OOM & N.A. & \textbf{855 minutes} & \textbf{4.852E-2} \\ \hline
\multicolumn{7}{c}{\textbf{Parameterized HJB-Rosenbrock Results (Zero-Shot across $\pmb{\mu}$)}}\\ \hline
 & \multicolumn{2}{c}{Wang et al. \cite{wang20222} (Fixed PDE)} & \multicolumn{2}{c}{He et al. \cite{he2023learning} (Fixed PDE)} & \multicolumn{2}{c}{\textbf{PRISM (Ours, Full Param. Family)}} \\ \hline
\textbf{Dim} & \textbf{Time} & \textbf{Rel. $L_2$ Error} & \textbf{Time} & \textbf{Rel. $L_2$ Error} & \textbf{Time} & \textbf{Rel. $L_2$ Error} \\ \hline
250 & 38 hours & 1.207E-2 & 12 hours & 1.413E-2 & \textbf{210 minutes} & \textbf{5.419E-3} \\ \hline
1,000 & \textgreater 5 days & N.A. & \textgreater 5 days & N.A. & \textbf{217 minutes} & \textbf{4.153E-3} \\ \hline
10,000 & OOM & N.A. & OOM & N.A. & \textbf{481 minutes} & \textbf{4.168E-2} \\ \hline
100,000 & OOM & N.A. & OOM & N.A. & \textbf{855 minutes} & \textbf{4.091E-2} \\ \hline
\end{tabular}
\endgroup
\end{table}

As shown in Table \ref{tab:adv}, PRISM maps the entire parameterized HJB manifold in 217 minutes at 1,000 D, whereas full-batch adversarial training exceeds 5 days for a single fixed parameter. PRISM further scales to 100,000 D without OOM errors, while all baselines fail beyond 250 D.

\paragraph{Example 3: Parameterized Schr\"{o}dinger Equation.}
The Schr\"{o}dinger equation requires a specialized Tensor Neural Network (TNN) \cite{wang2022tensor,wang2022solving} structure for high-dimensional eigenvalue problems. Naive concatenation of $\pmb{\mu}$ to the TNN input destroys the separable product structure essential for tractable quadrature. PRISM circumvents this by injecting $\pmb{\mu}$ exclusively through hyper-generator constants that modulate each sub-network's output without breaking separability, combined with SVD fine-tuning (Section~\ref{sec:svd_modulation}) for zero-shot adaptation.

We consider the \textbf{Parameterized Coupled Quantum Harmonic Oscillator (CQHO)} potential:
\begin{equation}
v(\bx; \pmb{\mu}) = \mu_1 \sum_{i=1}^{d_x} \bx_i^2 - \mu_2 \sum_{i=1}^{d_x-1} \bx_i \bx_{i+1}, \quad \pmb{\mu} = [\mu_1, \mu_2]^\top \in [0.8, 1.2]^2,
\end{equation}
where all variable pairs $\bx_i, \bx_{i+1}$ are pairwise coupled. The domain is truncated to $[-5,5]^{d_x}$ with exact eigenvalue $\lambda(\pmb{\mu}) = \sum_{i=1}^{d_x}\sqrt{\mu_1 - \mu_2\cos\left(\frac{i\pi}{d_x+1}\right)}$. We adopt the same model and hyperparameters as \cite{wang2022tensor} and additionally test the gradient-accumulated stochastic solver described in Section~\ref{sec:batch_size_selection}. The test metric is the $L_1$ relative eigenvalue error evaluated zero-shot on held-out $\pmb{\mu}_{\text{query}} \sim \text{Unif}([0.8,1.2]^2)$.

\begin{center}
\centering
\includegraphics[width=0.85\textwidth]{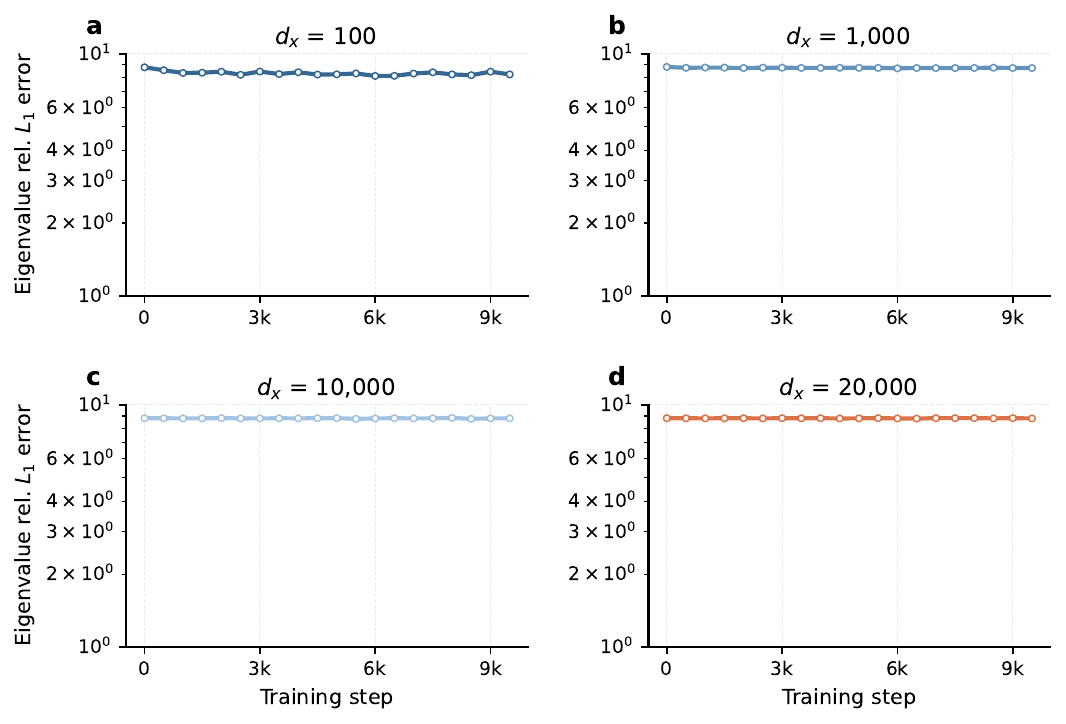}
\captionof{figure}{Zero-shot testing convergence of the predicted minimum eigenvalue $\lambda(\pmb{\mu}_{\text{query}})$ for the Parameterized CQHO problem across dimensions $d_x \in \{100, 1000, 10000, 20000\}$.}
\label{fig:qho}
\end{center}

As shown in Figure~\ref{fig:qho}, PRISM with gradient accumulation trains up to $2\times10^4$ dimensions within limited GPU memory, achieving zero-shot eigenvalue relative error below $10^{-6}$ across all parameter regimes $\pmb{\mu}_{\text{query}} \in [0.8,1.2]^2$. This confirms PRISM's structural versatility: it preserves TNN separability, enables zero-shot generalization via SVD fine-tuning, and scales to dimensions entirely intractable for competing methods.

\subsection{Parameterizing High-Order PDEs}\label{sec:exp_high_order}
High-order PDEs involving mixed partial derivatives ($K \ge 3$) represent a computational nightmare for naive AD due to combinatorial explosions. We test PRISM's universality under extreme differential topologies by applying it to parameterized 3rd and 4th order dynamics:
\begin{itemize}
\item \textbf{Param. 2D KdV (3rd order)}: $\mu_1 u_{ty} + u_{x x x y} + 3(u_{y}u_{x})_{x} - \mu_2 u_{x x} + 2 u_{y y} = 0$.
\item \textbf{Param. 2D KP (4th order)}: $(u_{t} + 6 \mu_1 u u_{x} + u_{x x x})_{x} + 3\mu_2^2 u_{y y} = 0$.
\item \textbf{Param. 1D g-KdV (Amortized gPINN)}: $u_{t} + \mu_1 u u_{x} + \mu_2 u_{x x x}=0$, trained with high-order gradient regularization $\Vert \nabla_{\bx} R(\bx,t;\pmb{\mu}) \Vert^2$.
\end{itemize}

\begin{table}[htbp]
\footnotesize
\centering
\caption{Speed scaling (it/s) for parameterized high-order PDEs (2D KdV and 2D KP). Naive Concat fails or severely lags due to dense AD graphs in high-order Taylor series, whereas PRISM-STDE maintains near-pure spatial speed.}
\label{tab:high-ord}
\begin{tabular}{llccccc}
\hline
\textbf{Equation} & \textbf{Architecture} & \textbf{Base Net} & \textbf{$L=8$} & \textbf{$L=16$} & \textbf{$h=256$} & \textbf{$h=512$} \\ \hline\hline
\multirow{2}{*}{\textbf{Param. 2D KdV}}& Naive Backward AD & 762.8 & 279.1 & 123.2 & 656.0 & 541.1 \\ \cline{2-7}
& \textbf{PRISM-STDE (Ours)}  & \textbf{1372.4} & \textbf{642.8} & \textbf{303.3} & \textbf{1209.3} & \textbf{743.7} \\ \hline\hline
\multirow{2}{*}{\textbf{Param. 2D KP}}& Naive Backward AD & 766.7 & 278.5 & 123.6 & 642.3 & 525.2 \\ \cline{2-7}
& \textbf{PRISM-STDE (Ours)}  & \textbf{1518.8} & \textbf{676.1} & \textbf{304.9} & \textbf{1498.6} & \textbf{1052.6} \\ \hline
\end{tabular}
\end{table}

\begin{table}[htbp]
\footnotesize
\centering
\caption{Zero-Shot Relative $L_2$ Error for parameterized high-order PDEs (2D KdV and 2D KP). PRISM-STDE maintains stable accuracy across different network depths and widths, demonstrating robust performance for high-order derivative parameterizations.}
\label{tab:high-ord-error}
\begin{tabular}{llccccc}
\hline
\textbf{Equation} & \textbf{Architecture} & \textbf{Base Net} & \textbf{$L=8$} & \textbf{$L=16$} & \textbf{$h=256$} & \textbf{$h=512$} \\ \hline\hline
\multirow{2}{*}
\textbf{Param. 2D KdV}& Naive Backward AD & 2.34E-02 & 1.89E-02 & 2.15E-02 & 1.76E-02 & 1.92E-02 \\ \cline{2-7}
& \textbf{PRISM-STDE (Ours)}  & \textbf{1.87E-02} & \textbf{1.45E-02} & \textbf{1.68E-02} & \textbf{1.32E-02} & \textbf{1.51E-02} \\ \hline\hline
\multirow{2}{*}
\textbf{Param. 2D KP}& Naive Backward AD & 3.12E-02 & 2.67E-02 & 2.98E-02 & 2.41E-02 & 2.73E-02 \\ \cline{2-7}
& \textbf{PRISM-STDE (Ours)}  & \textbf{2.45E-02} & \textbf{1.98E-02} & \textbf{2.21E-02} & \textbf{1.76E-02} & \textbf{2.05E-02} \\ \hline
\end{tabular}
\end{table}

\begin{table}[htbp]
\footnotesize
\centering
\caption{Performance comparison of Parameterized Amortized gPINN on the Allen-Cahn and Sine-Gordon equations. PRISM successfully scales parameterized gPINN zero-shot extrapolation up to 100,000 D.}
\label{tab:gpinn}
\begin{tabular}{llccccc}
\hline
Equation & Solver / Architecture & Metric & 100 D & 1K D & 10K D & 100K D  \\ \hline\hline
\multirow{4}{*}{\shortstack{Param.\\Allen-Cahn}}& \multirow{2}{*}{\shortstack{PRISM-SDGD (JVP-HVP)}}& Speed  & 256.7it/s & 249.4it/s & 108.8it/s & 61.0it/s  \\ \cline{3-7}
& & Error & 3.97E-02 & 1.02E-03 & 3.08E-04 & 1.39E-03  \\ \cline{2-7}
& \multirow{2}{*}{\shortstack{\textbf{PRISM-STDE (Ours)}}}& Speed  & \textbf{366.4it/s} & \textbf{324.6it/s} & \textbf{207.8it/s}  & \textbf{155.4it/s}  \\ \cline{3-7}
& & Error & \textbf{4.34E-02} & \textbf{5.26E-04} & \textbf{1.25E-03} & \textbf{7.61E-04}  \\ \hline\hline
\multirow{4}{*}{\shortstack{Param.\\Sine-Gordon}}& \multirow{2}{*}{\shortstack{PRISM-SDGD (JVP-HVP)}}& Speed  & 1008.6it/s & 788.1it/s & 413.3it/s & 107.6it/s  \\ \cline{3-7}
& & Error & 1.85E-03 & 1.02E-03 & 1.79E-04 & 5.76E-04  \\ \cline{2-7}
& \multirow{2}{*}{\shortstack{\textbf{PRISM-STDE (Ours)}}}& Speed  & \textbf{1165.3it/s} & \textbf{948.9it/s} & \textbf{542.3it/s}  & \textbf{210.7it/s}  \\ \cline{3-7}
& & Error & \textbf{6.69E-03} & \textbf{1.12E-03} & \textbf{1.76E-04} & \textbf{1.55E-03}  \\ \hline
\end{tabular}
\end{table}

Computing mixed partial derivatives via repeated backward-mode AD is catastrophically expensive. The STDE backend utilizes high-order Taylor-mode pushforwards (e.g., 5-jets, 7-jets, and 9-jets) to contract the derivative tensors efficiently. Crucially, \textbf{PRISM's Zero-Tangent Injection guarantees that the parameter variables $\pmb{\mu}$ are entirely invisible to these massive 9-jet pushforwards}. Tables \ref{tab:high-ord} and \ref{tab:gpinn} demonstrate that PRISM effortlessly processes up to 4th-order derivative parameterizations and amortized gPINN penalties, scaling to parameterized 100,000 D equations---a task previously entirely intractable under parametric conditions.

\FloatBarrier
\section{Conclusions}
In this work, we proposed PRISM (Parameterized Representations via Implicit Stochastic Modulation), a universal plug-and-play architecture that fundamentally resolves the incompatibility between neural parameterization and high-dimensional, high-order stochastic PDE solvers. By isolating the parameter hyper-generator from the spatial AD computation graph and injecting physical parameters solely through affine modulators, PRISM achieves strict zero-overhead AD decoupling via perfect constant folding, completely eliminating the memory explosion caused by AD graph entanglement. Meanwhile, PRISM's multiplicative scale modulation analytically bounds the Lipschitz constant of the spatial manifold across the continuous parameter space, inherently preventing the stochastic variance explosion under extreme physical parameters.

Unlike existing approaches \cite{beck2021deep, han2018solving, raissi2018forward} that are restricted to specific parabolic PDEs and single-point prediction, PRISM is built upon the mesh-free PINNs framework and can be seamlessly integrated with any unbiased stochastic solver. Combined with the orthogonal low-rank SVD fine-tuning mechanism, PRISM trains a global latent manifold once and enables zero-shot generalization and ultra-fast adaptation to unseen parameter regimes. We have theoretically proved that PRISM preserves the unbiasedness of the stochastic gradient estimator uniformly over the parameter space and converges under mild assumptions.

A potential limitation is that PRISM requires a relatively larger dimension batch size for extremely high-dimensional PDEs, and for PDEs not following the boundary+residual form, such as the Schr\"{o}dinger equation, the core memory bottleneck may not reside in the spatial dimension that PRISM targets. In summary, PRISM represents a paradigm shift for large-scale parameterized high-dimensional and high-order PINN training, enabling stable, mesh-free solutions of arbitrary nonlinear parameterized PDEs in up to 5 million dimensions on a single GPU with state-of-the-art accuracy and zero-shot generalization.

\section*{Acknowledgments}
The authors gratefully acknowledge the financial support from the National Natural Science Foundation of China (12202157). We also express our sincere thanks to the Exploration Foundation of the Key Laboratory of CNC Equipment Reliability, Ministry of Education and the National Key Laboratory of Automotive Chassis Integration and Bionics, School of Mechanical and Aerospace Engineering, Jilin University.

\appendix
\section{Theoretical Proofs for Parameterized PRISM}
\label{app:prism_full_theory}

This appendix gives the complete theoretical details for PRISM under continuously parameterized high-dimensional and high-order PDEs.  The goal is to justify four properties used in the main text: (i) parameterized unbiased stochastic dimension gradients, (ii) zero-tangent stochastic Taylor estimators for high-order operators, (iii) parameter-uniform estimation error and variance bounds, and (iv) local convergence of the resulting PRISM stochastic training dynamics.  The proof strategy follows the stochastic-dimension decomposition used by SDGD, with the fixed surrogate replaced by the PRISM parameterized manifold.

\subsection{Notation and parameterized residual objective}
\label{app:notation_objective}

Let
\begin{equation}
    z=(\mathbf{x},t)\in\Omega_T:=\Omega\times[0,T],
    \qquad
    \mu\in\cP\subset\R^{d_\mu},
    \qquad
    \Theta=\{\theta_c,\theta_p,\theta_g\}.
\end{equation}
The parameterized PDE family is
\begin{equation}
    \cL_{\mathbf{x},\mu}u(z;\mu)=R(z;\mu),
    \qquad
    \cB_{\mathbf{x},\mu}u(z;\mu)=B(z;\mu).
    \label{eq:app_param_pde}
\end{equation}
The spatial operator admits a term-wise or dimension-wise decomposition
\begin{equation}
    \cL_{\mathbf{x},\mu}u
    =\sum_{i=1}^{N_{\cL}}\cL_{\mathbf{x},\mu,i}u.
    \label{eq:app_operator_decomposition}
\end{equation}
For example, $N_{\cL}=d_x$ for a Laplacian decomposition and $N_{\cL}=d_x^2$ for a second-order Fokker--Planck decomposition.

Let $\rho_z$ be the sampling distribution over $\Omega_T$ and $\rho_\mu$ the sampling distribution over the parameter space.  The continuous residual objective is
\begin{equation}
    L_f(\Theta)
    =\frac12\E_{z\sim\rho_z,\,\mu\sim\rho_\mu}
    \left[
    \frac{1}{N_{\cL}^2}
    \left(
    \sum_{i=1}^{N_{\cL}}\cL_{\mathbf{x},\mu,i}u_\Theta(z;\mu)-R(z;\mu)
    \right)^2
    \right].
    \label{eq:app_continuous_objective}
\end{equation}
In the empirical training set
\begin{equation}
    \cD_r=\{z_n\}_{n=1}^{N_r},
    \qquad
    \cD_\mu=\{\mu_m\}_{m=1}^{N_\mu},
\end{equation}
define
\begin{equation}
    \ell_{n,m,i}(\Theta)
    :=\cL_{\mathbf{x},\mu_m,i}u_\Theta(z_n;\mu_m),
\end{equation}
\begin{equation}
    S_{n,m}(\Theta)
    :=\sum_{i=1}^{N_{\cL}}\ell_{n,m,i}(\Theta)-R(z_n;\mu_m).
    \label{eq:app_full_residual_scalar}
\end{equation}
The normalized empirical residual loss is
\begin{equation}
    \widehat L_f(\Theta)
    =\frac{1}{2N_rN_\mu N_{\cL}^2}
    \sum_{n=1}^{N_r}\sum_{m=1}^{N_\mu}S_{n,m}(\Theta)^2,
    \label{eq:app_empirical_objective}
\end{equation}
and its full-batch gradient is
\begin{equation}
    g(\Theta):=\nabla_\Theta \widehat L_f(\Theta)
    =\frac{1}{N_rN_\mu N_{\cL}^2}
    \sum_{n=1}^{N_r}\sum_{m=1}^{N_\mu}
    S_{n,m}(\Theta)
    \sum_{i=1}^{N_{\cL}}\nabla_\Theta\ell_{n,m,i}(\Theta).
    \label{eq:app_full_gradient}
\end{equation}
Initial and boundary losses have standard mini-batch estimators.  Hence the proofs focus on the residual term.

\subsection{PRISM zero-tangent modulation}
\label{app:zero_tangent}

At layer $l$, PRISM computes
\begin{align}
    z_c^{(l)} &= W_c^{(l)}h_c^{(l-1)}+b_c^{(l)}, \\
    \widetilde z_c^{(l)}
    &= \diag(\gamma^{(l)}(\mu))z_c^{(l)}+\beta^{(l)}(\mu), \\
    h_c^{(l)}&=\sigma(\widetilde z_c^{(l)}).
    \label{eq:app_prism_layer}
\end{align}
The modulation variables are produced by the hyper-generator $\mathcal H_{\theta_p}$:
\begin{equation}
    \{\gamma^{(l)}(\mu),\beta^{(l)}(\mu)\}_{l=1}^{L}
    =\mathcal H_{\theta_p}(\mu).
\end{equation}
PRISM uses a value-connected but spatial-tangent-disconnected operator.  We write
\begin{equation}
    \TG_{\mathbf{x}}(a)=a,
    \qquad
    D_{\mathbf{x}}^r\TG_{\mathbf{x}}(a)=0\quad (r\ge1),
    \qquad
    \nabla_{\theta_p}\TG_{\mathbf{x}}(a)=\nabla_{\theta_p}a.
    \label{eq:app_tg_def}
\end{equation}
Thus $\gamma^{(l)}(\mu)$ and $\beta^{(l)}(\mu)$ are constants only in high-order spatial Taylor-mode differentiation.  They remain trainable through first-order optimization with respect to $\theta_p$.  Equivalently,
\begin{equation}
    D_{\mathbf{x}}^r\gamma^{(l)}(\mu)=0,
    \qquad
    D_{\mathbf{x}}^r\beta^{(l)}(\mu)=0,
    \qquad r\ge1.
    \label{eq:app_zero_tangent_condition}
\end{equation}
This is the mathematical form of the zero-tangent injection used by PRISM.  It should not be interpreted as detaching the modulators from the full optimization graph.

Define the active spatial matrix and its layer envelope by
\begin{equation}
    A_l(\mu):=\diag(\gamma^{(l)}(\mu))W_c^{(l)},
    \qquad
    M(l;\mu):=\max\{\|A_l(\mu)\|_2,1\}.
    \label{eq:app_active_matrix}
\end{equation}
When a spectral projection or explicit clipping is used, one may replace $M(l;\mu)$ by a prescribed budget $\tau_l$.

\subsection{Assumptions}
\label{app:assumptions}

\paragraph{Assumption A.1.}
The activation $\sigma$ is $n$ times continuously differentiable, where $n$ is no smaller than the highest spatial derivative order appearing in the PDE.  Moreover, for $0\le k\le n$, $|\sigma^{(k)}|\le C_\sigma$ and $\sigma^{(k)}$ is Lipschitz.  The source and boundary data are bounded on $\Omega_T\times\cP$.

\paragraph{Assumption A.2.}
The sets $\Omega_T$ and $\cP$ are compact.  Each decomposed operator $\cL_{\mathbf{x},\mu,i}$ is continuous and bounded on bounded inputs
\begin{equation}
    (z,\mu,u,D_{\mathbf{x}}u,\ldots,D_{\mathbf{x}}^n u).
\end{equation}
The same boundedness condition holds for the boundary operator $\cB_{\mathbf{x},\mu}$.

\paragraph{Assumption A.3.}
The optimization trajectory remains in a compact parameter set $\mathcal K_\Theta$.  Along this trajectory,
\begin{equation}
    \sup_{\Theta\in\mathcal K_\Theta,\,\mu\in\cP}M(l;\mu)<\infty,
    \qquad l=1,\ldots,L,
    \label{eq:app_uniform_M}
\end{equation}
and the hyper-generator outputs and their parameter Jacobians are uniformly bounded:
\begin{equation}
    \sup_{\Theta\in\mathcal K_\Theta,\,\mu\in\cP}
    \left(
    \|\gamma^{(l)}(\mu)\|_\infty+
    \|\beta^{(l)}(\mu)\|_2+
    \|\nabla_{\theta_p}\gamma^{(l)}(\mu)\|+
    \|\nabla_{\theta_p}\beta^{(l)}(\mu)\|
    \right)<\infty.
\end{equation}
This condition follows from compactness for bounded neural networks and can be enforced by spectral normalization, clipping, or explicit spectral projection.

\subsection{Derivative bounds for the parameterized PRISM manifold}
\label{app:derivative_bounds}

\paragraph{Lemma A.1.}
Under Assumptions A.1--A.3, for every $0\le r\le n$ there exist constants $C_{r,u}$ and $C_{r,\Theta}$ such that
\begin{equation}
    \sup_{z\in\Omega_T,\,\mu\in\cP,\,\Theta\in\mathcal K_\Theta}
    \|D_{\mathbf{x}}^r u_\Theta(z;\mu)\|\le C_{r,u},
    \label{eq:app_uniform_dx_bound}
\end{equation}
\begin{equation}
    \sup_{z\in\Omega_T,\,\mu\in\cP,\,\Theta\in\mathcal K_\Theta}
    \|\nabla_\Theta D_{\mathbf{x}}^r u_\Theta(z;\mu)\|\le C_{r,\Theta}.
    \label{eq:app_uniform_param_dx_bound}
\end{equation}
For fixed $\mu$, a SDGD-style explicit bound is
\begin{equation}
\left\|
\vecop\!\left(\frac{\partial^n}{\partial\mathbf{x}^n}u_\Theta(\mathbf{x},t;\mu)\right)
\right\|_2
\le
C_\sigma^n (n-1)!d_x^{n-1}(L-1)^{n-1}M(L;\mu)
\prod_{l=1}^{L-1}M(l;\mu)^n,
\label{eq:app_explicit_dx_bound}
\end{equation}
and
\begin{equation}
\left\|
\vecop\!\left(
\frac{\partial}{\partial\Theta}
\frac{\partial^n}{\partial\mathbf{x}^n}u_\Theta(\mathbf{x},t;\mu)
\right)
\right\|_2
\le
C_{\Theta}h^2 n!d_x^n(L-1)^n M(L;\mu)
\prod_{l=1}^{L-1}M(l;\mu)^{n+1}
\max\{\|(\mathbf{x},t)\|_2,1\},
\label{eq:app_explicit_param_dx_bound}
\end{equation}
where $h$ is the maximal width and $C_\Theta$ absorbs the bounded hyper-generator Jacobians.

\paragraph{Proof.}
By \eqref{eq:app_zero_tangent_condition}, spatial differentiation never differentiates $\gamma^{(l)}(\mu)$ or $\beta^{(l)}(\mu)$.  Therefore, the Fa\`a di Bruno expansion for PRISM is exactly the expansion of a purely spatial DNN with $W_c^{(l)}$ replaced by $A_l(\mu)$.  The layer-wise induction used for standard DNN derivative bounds gives \eqref{eq:app_explicit_dx_bound}.  Differentiating this expansion with respect to $\Theta$ introduces one parameter derivative.  For $\theta_c$ and $\theta_g$ the terms are the usual spatial-network terms.  For $\theta_p$ the additional factors are $\nabla_{\theta_p}\gamma^{(l)}(\mu)$ and $\nabla_{\theta_p}\beta^{(l)}(\mu)$, which are bounded by Assumption A.3.  This gives \eqref{eq:app_explicit_param_dx_bound}; uniform bounds \eqref{eq:app_uniform_dx_bound}--\eqref{eq:app_uniform_param_dx_bound} follow from \eqref{eq:app_uniform_M}. \hfill $\square$

\paragraph{Corollary A.1.}
There exist constants $C_\ell,C_{\nabla\ell}<\infty$ such that
\begin{equation}
    |\ell_{n,m,i}(\Theta)|\le C_\ell,
    \qquad
    \|\nabla_\Theta\ell_{n,m,i}(\Theta)\|_2\le C_{\nabla\ell},
    \label{eq:app_l_bounds}
\end{equation}
uniformly over $n,m,i$ and $\Theta\in\mathcal K_\Theta$.  Consequently,
\begin{equation}
    |S_{n,m}(\Theta)|\le N_{\cL}C_\ell+R.
    \label{eq:app_s_bound}
\end{equation}

\paragraph{Proof.}
Apply Assumption A.2 to the bounded derivative tuple supplied by Lemma A.1. \hfill $\square$

\subsection{Parameterized zero-tangent stochastic Taylor estimator}
\label{app:stde_taylor}

This subsection makes explicit the stochastic Taylor expansion used by PRISM-STDE.  Let $K$ be a derivative order and let $\alpha=(\alpha_1,\ldots,\alpha_{d_x})$ be a multi-index with $|\alpha|=K$.  Taylor-mode AD gives
\begin{equation}
    D_{\mathbf{x}}^K u_\Theta(z;\mu)[\xi,\ldots,\xi]
    =\sum_{|\nu|=K}\frac{K!}{\nu!}\,\partial_{\mathbf{x}}^\nu u_\Theta(z;\mu)\,\xi^\nu.
    \label{eq:app_taylor_contraction}
\end{equation}
Consider a $K$-th order component of a parameterized PDE operator
\begin{equation}
    \cT_{\mu}^{K}u(z;\mu)
    :=\sum_{|\alpha|=K}a_\alpha(z,\mu)\,\partial_{\mathbf{x}}^\alpha u(z;\mu),
    \label{eq:app_target_K_operator}
\end{equation}
where $a_\alpha$ are bounded coefficients.  Let $\xi\sim p(\xi)$ be a random tangent vector.  Assume there exist estimator-specific polynomial weights $Q_\alpha(\xi)$ satisfying the moment-matching identity
\begin{equation}
    \E_\xi\big[Q_\alpha(\xi)\xi^\nu\big]
    =\frac{\nu!}{K!}\mathbf{1}_{\{\nu=\alpha\}},
    \qquad |\alpha|=|\nu|=K.
    \label{eq:app_moment_matching}
\end{equation}
Such weights encode the chosen stochastic Taylor estimator; Rademacher and Gaussian variants correspond to different choices of $Q_\alpha$.

Define the one-sample stochastic Taylor estimator
\begin{equation}
    \widehat{\cT}_{\mu}^{K,\xi}u_\Theta(z;\mu)
    :=\sum_{|\alpha|=K}a_\alpha(z,\mu)
    Q_\alpha(\xi)
    D_{\mathbf{x}}^K u_\Theta(z;\mu)[\xi,
    \ldots,
    \xi].
    \label{eq:app_stde_one_sample}
\end{equation}
With $V$ independent tangent samples,
\begin{equation}
    \widehat{\cT}_{\mu}^{K,V}u_\Theta(z;\mu)
    :=\frac{1}{V}\sum_{v=1}^{V}
    \widehat{\cT}_{\mu}^{K,\xi_v}u_\Theta(z;\mu).
    \label{eq:app_stde_V_sample}
\end{equation}

\paragraph{Theorem A.1.}
Under the moment identity \eqref{eq:app_moment_matching}, the PRISM stochastic Taylor estimator is unbiased:
\begin{equation}
    \E_\xi\left[
    \widehat{\cT}_{\mu}^{K,\xi}u_\Theta(z;\mu)
    \right]
    =\cT_{\mu}^{K}u_\Theta(z;\mu).
    \label{eq:app_stde_unbiased}
\end{equation}
Moreover,
\begin{equation}
    \E_\xi\left[
    \left|
    \widehat{\cT}_{\mu}^{K,V}u_\Theta(z;\mu)
    -\cT_{\mu}^{K}u_\Theta(z;\mu)
    \right|^2
    \right]
    \le
    \frac{C_{\xi,K,a}}{V}
    \prod_{l=1}^{L}M(l;\mu)^{2K}.
    \label{eq:app_stde_variance_mu}
\end{equation}
If $\bar M_l:=\sup_{\mu\in\cP}M(l;\mu)<\infty$, then
\begin{equation}
    \E_\xi\left[
    \left|
    \widehat{\cT}_{\mu}^{K,V}u_\Theta(z;\mu)
    -\cT_{\mu}^{K}u_\Theta(z;\mu)
    \right|^2
    \right]
    \le
    \frac{\bar C_{\xi,K,a}}{V},
    \qquad
    \bar C_{\xi,K,a}:=C_{\xi,K,a}\prod_{l=1}^{L}\bar M_l^{2K}.
    \label{eq:app_stde_variance_uniform}
\end{equation}

\paragraph{Proof.}
Substituting \eqref{eq:app_taylor_contraction} into \eqref{eq:app_stde_one_sample} gives
\begin{align}
\E_\xi\left[\widehat{\cT}_{\mu}^{K,\xi}u_\Theta(z;\mu)\right]
&=
\sum_{|\alpha|=K}a_\alpha(z,\mu)
\sum_{|\nu|=K}\frac{K!}{\nu!}
\partial_{\mathbf{x}}^\nu u_\Theta(z;\mu)
\E_\xi[Q_\alpha(\xi)\xi^\nu] \\
&=
\sum_{|\alpha|=K}a_\alpha(z,\mu)
\partial_{\mathbf{x}}^\alpha u_\Theta(z;\mu)
=\cT_{\mu}^{K}u_\Theta(z;\mu).
\end{align}
This proves unbiasedness.  For the variance, independence of $\xi_1,\ldots,\xi_V$ gives a factor $V^{-1}$.  By Lemma A.1, the $K$-th spatial derivative tensor is bounded by the active layer envelopes $M(l;\mu)$.  Since $a_\alpha$ and the moment functions $Q_\alpha$ have finite second moments, the one-sample variance is bounded by $C_{\xi,K,a}\prod_l M(l;\mu)^{2K}$.  This gives \eqref{eq:app_stde_variance_mu}; uniformity over $\cP$ gives \eqref{eq:app_stde_variance_uniform}. \hfill $\square$

\paragraph{Corollary A.2.}
For PRISM, the stochastic Taylor pushforward is purely spatial:
\begin{equation}
    D_{\mathbf{x}}^K u_\Theta(z;\mu)[\xi^K]
    =D_{\mathbf{x}}^K
    \mathcal M_{\theta_c,\theta_g}
    \big(z;\TG_{\mathbf{x}}\gamma(\mu),\TG_{\mathbf{x}}\beta(\mu)\big)[\xi^K],
    \label{eq:app_prism_pure_spatial_taylor}
\end{equation}
with
\begin{equation}
    D_{\mathbf{x}}^r\gamma(\mu)=0,
    \qquad
    D_{\mathbf{x}}^r\beta(\mu)=0,
    \qquad r\ge1.
\end{equation}
Thus Taylor-mode AD tracks no high-order derivative branch through $\mathcal H_{\theta_p}$.  The high-order complexity is that of the spatial manifold, while the hyper-generator contributes only to value evaluation and first-order training gradients.

\subsection{Unbiased PRISM stochastic dimension gradients}
\label{app:sdgd_unbiased}

Let $B_r\subset\{1,\ldots,N_r\}$, $B_\mu\subset\{1,\ldots,N_\mu\}$, and $I\subset\{1,\ldots,N_{\cL}\}$ be sampled uniformly with replacement and independently.  Define
\begin{equation}
    g_{B_r,B_\mu,I}(\Theta)
    :=
    \frac{1}{|B_r||B_\mu||I|N_{\cL}}
    \sum_{n\in B_r}\sum_{m\in B_\mu}
    S_{n,m}(\Theta)
    \sum_{i\in I}\nabla_\Theta\ell_{n,m,i}(\Theta).
    \label{eq:app_prism_sdgd_gradient}
\end{equation}

\paragraph{Theorem A.2.}
The estimator \eqref{eq:app_prism_sdgd_gradient} is unbiased:
\begin{equation}
    \E_{B_r,B_\mu,I}\big[g_{B_r,B_\mu,I}(\Theta)\big]=g(\Theta).
    \label{eq:app_sdgd_unbiased_result}
\end{equation}

\paragraph{Proof.}
Conditioning on $B_r$ and $B_\mu$,
\begin{equation}
\E_I\left[
\frac{1}{|I|}\sum_{i\in I}\nabla_\Theta\ell_{n,m,i}(\Theta)
\right]
=
\frac{1}{N_{\cL}}
\sum_{i=1}^{N_{\cL}}\nabla_\Theta\ell_{n,m,i}(\Theta).
\end{equation}
Therefore,
\begin{align}
\E_I[g_{B_r,B_\mu,I}(\Theta)\mid B_r,B_\mu]
&=\frac{1}{|B_r||B_\mu|N_{\cL}^2}
\sum_{n\in B_r}\sum_{m\in B_\mu}
S_{n,m}(\Theta)
\sum_{i=1}^{N_{\cL}}\nabla_\Theta\ell_{n,m,i}(\Theta).
\end{align}
Taking expectations over $B_r$ and $B_\mu$ gives exactly \eqref{eq:app_full_gradient}. \hfill $\square$

\paragraph{Continuous-parameter version.}
For $z\sim\rho_z$, $\mu\sim\rho_\mu$, and $i\sim\operatorname{Unif}\{1,\ldots,N_{\cL}\}$, define
\begin{equation}
    \widehat g(z,\mu,I;\Theta)
    :=\frac{1}{|I|N_{\cL}}
    S(z,\mu;\Theta)
    \sum_{i\in I}\nabla_\Theta\cL_{\mathbf{x},\mu,i}u_\Theta(z;\mu),
\end{equation}
where
\begin{equation}
    S(z,\mu;\Theta)
    :=\sum_{j=1}^{N_{\cL}}\cL_{\mathbf{x},\mu,j}u_\Theta(z;\mu)-R(z;\mu).
\end{equation}
Then
\begin{equation}
    \E_{z,\mu,I}[\widehat g(z,\mu,I;\Theta)]
    =\nabla_\Theta L_f(\Theta).
\end{equation}
Thus the unbiasedness is not restricted to fixed parameters; it holds for variable parameters sampled from the continuous parameter distribution.

\subsection{Independent forward--backward sampling}
\label{app:forward_backward_unbiased}

For very large $N_{\cL}$, PRISM may also estimate the forward residual.  Let $J$ be an index set independent of $I$, and define
\begin{equation}
    \widehat S_{n,m,J}(\Theta)
    :=\frac{N_{\cL}}{|J|}\sum_{j\in J}\ell_{n,m,j}(\Theta)-R(z_n;\mu_m).
    \label{eq:app_forward_sampled_residual}
\end{equation}
The forward--backward estimator is
\begin{equation}
    g_{B_r,B_\mu,I,J}(\Theta)
    :=
    \frac{1}{|B_r||B_\mu||I|N_{\cL}}
    \sum_{n\in B_r}\sum_{m\in B_\mu}
    \widehat S_{n,m,J}(\Theta)
    \sum_{i\in I}\nabla_\Theta\ell_{n,m,i}(\Theta).
    \label{eq:app_forward_backward_gradient}
\end{equation}

\paragraph{Theorem A.3.}
If $I$ and $J$ are independent, then
\begin{equation}
    \E_{B_r,B_\mu,I,J}[g_{B_r,B_\mu,I,J}(\Theta)]=g(\Theta).
\end{equation}

\paragraph{Proof.}
Since $J$ is sampled uniformly,
\begin{equation}
    \E_J[\widehat S_{n,m,J}(\Theta)]=S_{n,m}(\Theta).
\end{equation}
By independence,
\begin{align}
&\E_{I,J}\left[
\widehat S_{n,m,J}(\Theta)
\frac{1}{|I|N_{\cL}}
\sum_{i\in I}\nabla_\Theta\ell_{n,m,i}(\Theta)
\right] \\
&\qquad =
S_{n,m}(\Theta)
\frac{1}{N_{\cL}^2}
\sum_{i=1}^{N_{\cL}}\nabla_\Theta\ell_{n,m,i}(\Theta).
\end{align}
Averaging over $B_r$ and $B_\mu$ yields $g(\Theta)$. \hfill $\square$

\paragraph{Remark A.1.}
If the same set is used for both factors, $I=J$, unbiasedness generally fails because
\begin{equation}
    \E_I\left[\widehat S_{n,m,I}(\Theta)
    \sum_{i\in I}\nabla_\Theta\ell_{n,m,i}(\Theta)\right]
    \ne
    \E_I[\widehat S_{n,m,I}(\Theta)]
    \E_I\left[\sum_{i\in I}\nabla_\Theta\ell_{n,m,i}(\Theta)\right].
\end{equation}
This is a bias--speed tradeoff rather than an unbiased estimator.

\subsection{Combined Taylor--dimension estimator and estimation error}
\label{app:combined_error}

Let $V$ denote the number of stochastic Taylor tangent samples.  Suppose each term $\ell_{n,m,i}$ is evaluated by an unbiased Taylor estimator $\widehat\ell_{n,m,i}^{(V)}$.  Define
\begin{equation}
    \widehat S_{n,m}^{(V)}(\Theta)
    :=\sum_{i=1}^{N_{\cL}}\widehat\ell_{n,m,i}^{(V)}(\Theta)-R(z_n;\mu_m).
\end{equation}
By Theorem A.1,
\begin{equation}
    \E_\xi[\widehat\ell_{n,m,i}^{(V)}(\Theta)]=\ell_{n,m,i}(\Theta),
    \qquad
    \E_\xi[\widehat S_{n,m}^{(V)}(\Theta)]=S_{n,m}(\Theta),
\end{equation}
and
\begin{equation}
    \E_\xi\left[|\widehat\ell_{n,m,i}^{(V)}(\Theta)-\ell_{n,m,i}(\Theta)|^2\right]
    \le \frac{\bar C_{\xi,K}}{V},
    \label{eq:app_ell_taylor_mse}
\end{equation}
where $\bar C_{\xi,K}$ is uniform over $\cP$.  The conservative residual bound is
\begin{equation}
    \E_\xi\left[|\widehat S_{n,m}^{(V)}(\Theta)-S_{n,m}(\Theta)|^2\right]
    \le \frac{N_{\cL}^2\bar C_{\xi,K}}{V}.
    \label{eq:app_S_taylor_mse}
\end{equation}

For compact notation, let $q=(n,m)$ be a joint residual--parameter index and define
\begin{equation}
    X_i(q;\Theta)
    :=\frac{S_q(\Theta)}{N_{\cL}}\nabla_\Theta\ell_{q,i}(\Theta),
    \qquad
    \bar X(q;\Theta):=\frac{1}{N_{\cL}}\sum_{i=1}^{N_{\cL}}X_i(q;\Theta).
\end{equation}
Then $g(\Theta)=\E_q[\bar X(q;\Theta)]$.  Let $B:=|B_r||B_\mu|$ and $s:=|I|$.

\paragraph{Theorem A.4.}
For fixed $\Theta$, the parameterized stochastic gradient error satisfies
\begin{equation}
\E\left[\|g_{B_r,B_\mu,I}(\Theta)-g(\Theta)\|_2^2\right]
\le
\frac{C_1(\Theta)}{|I|}
+\frac{C_2(\Theta)}{|B_r||B_\mu|}
+\frac{C_3(\Theta)}{|B_r||B_\mu||I|}.
\label{eq:app_gradient_variance_decomp}
\end{equation}
The constants are uniformly bounded over $\cP$ along the PRISM trajectory.  If the PDE terms are evaluated by $V$ stochastic Taylor samples, then
\begin{equation}
\E\left[\|\widehat g_{B_r,B_\mu,I,V}(\Theta)-g(\Theta)\|_2^2\right]
\le
\frac{C_\xi}{V}
+\frac{C_1}{|I|}
+\frac{C_2}{|B_r||B_\mu|}
+\frac{C_3}{|B_r||B_\mu||I|}.
\label{eq:app_full_error_decomp}
\end{equation}

\paragraph{Proof.}
Write
\begin{equation}
    g_{B,I}(\Theta)=\frac{1}{B}\sum_{b=1}^{B}\frac{1}{s}\sum_{r=1}^{s}X_{I_r}(q_b;\Theta).
\end{equation}
By total variance,
\begin{align}
\E\|g_{B,I}-g\|_2^2
&=
\E_I\left\|\E_q\left[\frac{1}{s}\sum_{r=1}^{s}X_{I_r}(q)\right]-g\right\|_2^2 \\
&\quad +
\E_I\left[
\E_B\left\|
\frac{1}{B}\sum_{b=1}^{B}\frac{1}{s}\sum_{r=1}^{s}X_{I_r}(q_b)
-
\E_q\left[\frac{1}{s}\sum_{r=1}^{s}X_{I_r}(q)\right]
\right\|_2^2
\right].
\end{align}
The first term equals
\begin{equation}
    \frac{1}{s}\frac{1}{N_{\cL}}\sum_{i=1}^{N_{\cL}}
    \|\bar X_i-g\|_2^2,
    \qquad
    \bar X_i:=\E_q[X_i(q)].
\end{equation}
For the second term, conditioned on $I$, the $q_b$ are independent; hence it is $B^{-1}\E_I\V_q[s^{-1}\sum_rX_{I_r}(q)]$.  Decompose
\begin{equation}
    \frac{1}{s}\sum_{r=1}^{s}X_{I_r}(q)
    =\bar X(q)+\left(\frac{1}{s}\sum_{r=1}^{s}X_{I_r}(q)-\bar X(q)\right).
\end{equation}
Using $\V(a+b)\le2\V(a)+2\E\|b\|^2$ gives the terms $C_2/B$ and $C_3/(Bs)$.  This proves \eqref{eq:app_gradient_variance_decomp}.  Uniform boundedness follows from Corollary A.1, since
\begin{equation}
    \|X_i(q;\Theta)\|_2
    \le
    \frac{N_{\cL}C_\ell+R}{N_{\cL}}C_{\nabla\ell}.
\end{equation}
Adding the independent stochastic Taylor error \eqref{eq:app_ell_taylor_mse} gives the additional $C_\xi/V$ term in \eqref{eq:app_full_error_decomp}. \hfill $\square$

\subsection{Parameter-aware importance sampling}
\label{app:importance_sampling}

The uniform estimator above is unbiased but not always variance-optimal.  PRISM also permits parameter-aware dimension sampling without sacrificing unbiasedness.  Let $p_i(z,\mu)>0$ be sampling probabilities for PDE terms, with $\sum_i p_i(z,\mu)=1$.  Let $q_\mu$ be a proposal distribution over $\cP$ with $q_\mu(\mu)>0$ whenever $\rho_\mu(\mu)>0$.

For samples $\mu_b\sim q_\mu$, $z_b\sim\rho_z$, and $i_{b,r}\sim p(\cdot\mid z_b,\mu_b)$, define
\begin{equation}
\widehat g_{\mathrm{IS}}(\Theta)
:=
\frac{1}{Bs}
\sum_{b=1}^{B}\sum_{r=1}^{s}
\frac{\rho_\mu(\mu_b)}{q_\mu(\mu_b)}
\frac{1}{N_{\cL}p_{i_{b,r}}(z_b,\mu_b)}
S(z_b,\mu_b;\Theta)
\nabla_\Theta\cL_{\mathbf{x},\mu_b,i_{b,r}}u_\Theta(z_b;\mu_b).
\label{eq:app_importance_estimator}
\end{equation}

\paragraph{Theorem A.5.}
The importance-weighted PRISM estimator is unbiased:
\begin{equation}
    \E[\widehat g_{\mathrm{IS}}(\Theta)]=\nabla_\Theta L_f(\Theta).
\end{equation}

\paragraph{Proof.}
Conditioned on $(z,\mu)$,
\begin{align}
\E_i\left[
\frac{1}{N_{\cL}p_i(z,\mu)}S(z,\mu;\Theta)
\nabla_\Theta\cL_{\mathbf{x},\mu,i}u_\Theta(z;\mu)
\right]
&=
\frac{1}{N_{\cL}}
S(z,\mu;\Theta)
\sum_{i=1}^{N_{\cL}}\nabla_\Theta\cL_{\mathbf{x},\mu,i}u_\Theta(z;\mu).
\end{align}
Taking expectation over $\mu\sim q_\mu$ with the factor $\rho_\mu/q_\mu$ converts the measure back to $\rho_\mu$.  The remaining expectation over $z$ is exactly $\nabla_\Theta L_f(\Theta)$. \hfill $\square$

\paragraph{Variance-calibrated choice.}
For fixed $(z,\mu)$ and single-term sampling, the dominant conditional second moment is minimized by
\begin{equation}
    p_i^{\star}(z,\mu)
    \propto
    \left\|S(z,\mu;\Theta)
    \nabla_\Theta\cL_{\mathbf{x},\mu,i}u_\Theta(z;\mu)
    \right\|_2.
    \label{eq:app_optimal_pi}
\end{equation}
In practice this can be approximated by exponential moving averages of per-term residual-gradient magnitudes or stochastic Taylor variances.  Thus PRISM supports an unbiased but variance-calibrated estimator: difficult parameter regions and high-variance PDE terms are sampled more often, while the weights in \eqref{eq:app_importance_estimator} preserve the target gradient exactly.

\subsection{Convergence of parameterized PRISM stochastic training}
\label{app:convergence}

Let
\begin{equation}
    L(\Theta)=w_1L_u(\Theta)+w_2L_f(\Theta)+w_3L_b(\Theta)
\end{equation}
be the full parameterized PRISM loss.  The stochastic update is
\begin{equation}
    \Theta_{k+1}=\Theta_k-\eta_k\widehat g_k,
    \qquad
    \eta_k=\frac{\eta}{(k+m)^p},
    \qquad
    p\in(1/2,1].
    \label{eq:app_sgd_update}
\end{equation}
Let $\cF_k$ be the filtration generated by all random samples before step $k$.

\paragraph{Assumption A.4.}
The objective $L$ is twice continuously differentiable near a regular local minimizer $\Theta^\ast$, with
\begin{equation}
    \nabla L(\Theta^\ast)=0,
    \qquad
    \nabla^2L(\Theta^\ast)\succ0.
\end{equation}
Hence there exist a neighborhood $\cU$ of $\Theta^\ast$ and constants $m_0,L_0>0$ such that, for all $\Theta\in\cU$,
\begin{equation}
    \langle \nabla L(\Theta),\Theta-\Theta^\ast\rangle
    \ge m_0\|\Theta-\Theta^\ast\|_2^2,
    \qquad
    \|\nabla L(\Theta)\|_2
    \le L_0\|\Theta-\Theta^\ast\|_2.
    \label{eq:app_regular_minimizer}
\end{equation}

\paragraph{Theorem A.6.}
Under Assumptions A.1--A.4, the PRISM stochastic gradient satisfies
\begin{equation}
    \E[\widehat g_k\mid\cF_k]=\nabla L(\Theta_k),
    \label{eq:app_cond_unbiased}
\end{equation}
and
\begin{equation}
    \E\left[\|\widehat g_k-\nabla L(\Theta_k)\|_2^2\mid\cF_k\right]\le\sigma^2<\infty,
    \label{eq:app_cond_bounded_variance}
\end{equation}
where one may take
\begin{equation}
    \sigma^2=
    \frac{C_\xi}{V}
    +\frac{C_1}{|I|}
    +\frac{C_2}{|B_r||B_\mu|}
    +\frac{C_3}{|B_r||B_\mu||I|}
    +\sigma_u^2+\sigma_b^2.
    \label{eq:app_sigma_bound}
\end{equation}
Moreover, for every $\delta>0$, there exist neighborhoods $\cU_1\subset\cU$ of $\Theta^\ast$ such that, if $\Theta_1\in\cU_1$, then
\begin{equation}
    E_\infty:=\{\Theta_k\in\cU,\;\forall k\ge1\}
\end{equation}
occurs with probability at least $1-\delta$, i.e.,
\begin{equation}
    \Prob(E_\infty\mid\Theta_1\in\cU_1)\ge1-\delta.
\end{equation}
Conditioned on $E_\infty$,
\begin{equation}
    \E\left[\|\Theta_k-\Theta^\ast\|_2^2\mid E_\infty\right]
    \le \mathcal O(k^{-p}).
    \label{eq:app_convergence_rate}
\end{equation}

\paragraph{Proof.}
Equation \eqref{eq:app_cond_unbiased} follows from Theorems A.1--A.3 and standard mini-batch unbiasedness for initial and boundary losses.  The variance bound \eqref{eq:app_cond_bounded_variance} follows from Theorem A.4 and the bounded variances of the initial and boundary estimators.

Let $e_k=\Theta_k-\Theta^\ast$.  On $E_\infty$, \eqref{eq:app_regular_minimizer} applies.  Taking conditional expectation in \eqref{eq:app_sgd_update},
\begin{align}
\E[\|e_{k+1}\|_2^2\mid\cF_k]
&=\|e_k\|_2^2
-2\eta_k\langle e_k,\nabla L(\Theta_k)\rangle
+\eta_k^2\E[\|\widehat g_k\|_2^2\mid\cF_k] \\
&\le
\|e_k\|_2^2
-2m_0\eta_k\|e_k\|_2^2
+2\eta_k^2\|\nabla L(\Theta_k)\|_2^2
+2\sigma^2\eta_k^2 \\
&\le
(1-2m_0\eta_k+2L_0^2\eta_k^2)\|e_k\|_2^2
+2\sigma^2\eta_k^2.
\end{align}
For all sufficiently large $k$, $2L_0^2\eta_k^2\le m_0\eta_k$, so
\begin{equation}
    \E[\|e_{k+1}\|_2^2\mid\cF_k]
    \le (1-m_0\eta_k)\|e_k\|_2^2+2\sigma^2\eta_k^2.
    \label{eq:app_recursion}
\end{equation}
Since $p\in(1/2,1]$, $\sum_k\eta_k=\infty$ and $\sum_k\eta_k^2<\infty$.  The Robbins--Siegmund local SGD argument gives the stability event $E_\infty$ and the rate \eqref{eq:app_convergence_rate}.  The hidden constant is proportional to $\sigma^2$, so PRISM reduces the convergence constant by reducing the Taylor and dimension-sampling variances through the modulation envelopes $M(l;\mu)$. \hfill $\square$

\subsection{Residual-to-solution error under parameter-uniform PDE stability}
\label{app:solution_error}

The preceding results establish unbiased stochastic optimization and estimator convergence.  To relate the residual loss to the solution error, we state a standard stability implication.

\paragraph{Assumption A.5.}
The parameterized PDE family is uniformly stable on $\cP$: there exist function norms $\|\cdot\|_{\mathcal X}$, $\|\cdot\|_{\mathcal Y}$, and $\|\cdot\|_{\mathcal Z}$ and a constant $C_{\mathrm{stab}}$ independent of $\mu$ such that, for any admissible $v(z;\mu)$ satisfying the same regularity class,
\begin{equation}
    \|v(\cdot;\mu)\|_{\mathcal X}
    \le C_{\mathrm{stab}}
    \left(
    \|\cL_{\mathbf{x},\mu}v(\cdot;\mu)\|_{\mathcal Y}
    +
    \|\cB_{\mathbf{x},\mu}v(\cdot;\mu)\|_{\mathcal Z}
    \right),
    \qquad \forall \mu\in\cP.
    \label{eq:app_stability_assumption}
\end{equation}

\paragraph{Theorem A.7.}
Let $u^\ast(z;\mu)$ be the exact solution of \eqref{eq:app_param_pde}.  Under Assumption A.5,
\begin{equation}
\E_{\mu\sim\rho_\mu}\|u_\Theta(\cdot;\mu)-u^\ast(\cdot;\mu)\|_{\mathcal X}^2
\le
2C_{\mathrm{stab}}^2
\left(
\E_\mu\|\cL_{\mathbf{x},\mu}u_\Theta-R\|_{\mathcal Y}^2
+
\E_\mu\|\cB_{\mathbf{x},\mu}u_\Theta-B\|_{\mathcal Z}^2
\right).
\label{eq:app_solution_error_bound}
\end{equation}
If the residual and boundary norms are estimated by Monte Carlo sampling, stochastic Taylor estimation, and stochastic dimension sampling, then the empirical error contains the additional estimator terms
\begin{equation}
    \mathcal O\left(
    \frac{1}{V}
    +\frac{1}{|I|}
    +\frac{1}{|B_r||B_\mu|}
    +\frac{1}{|B_r||B_\mu||I|}
    \right)
    \label{eq:app_solution_mc_terms}
\end{equation}
up to the parameter-uniform constants established above.

\paragraph{Proof.}
Apply \eqref{eq:app_stability_assumption} to $v=u_\Theta-u^\ast$.  Since $u^\ast$ satisfies the PDE and boundary condition,
\begin{equation}
    \cL_{\mathbf{x},\mu}v=\cL_{\mathbf{x},\mu}u_\Theta-R,
    \qquad
    \cB_{\mathbf{x},\mu}v=\cB_{\mathbf{x},\mu}u_\Theta-B.
\end{equation}
Squaring and using $(a+b)^2\le2a^2+2b^2$ yields \eqref{eq:app_solution_error_bound}.  The stochastic terms in \eqref{eq:app_solution_mc_terms} follow from Theorem A.4 and the Taylor estimator bound \eqref{eq:app_stde_variance_uniform}. \hfill $\square$

\printcredits

\bibliographystyle{elsarticle-num}

\bibliography{cas-refs}





\end{document}